\documentclass{article}
\usepackage{preprint}

\usepackage[utf8]{inputenc}
\usepackage[T1]{fontenc}
\usepackage{amsmath,amssymb,amsthm,mathtools}
\usepackage{booktabs}
\usepackage{graphicx}
\usepackage{xcolor}
\usepackage{microtype}
\usepackage[colorlinks=true,linkcolor=blue!60!black,citecolor=blue!60!black,urlcolor=blue!60!black]{hyperref}
\usepackage{enumitem}

\graphicspath{{figures/}}

\definecolor{tbdviolet}{RGB}{112,48,160}

\newcommand{\Sm}[1]{\mathcal{S}^{#1}}
\newcommand{\clo}{\mathcal{C}}
\DeclareMathOperator{\clr}{clr}
\DeclareMathOperator{\ilr}{ilr}
\DeclareMathOperator{\cen}{cen}
\newcommand{\dA}{d_{A}}
\newcommand{\dAc}{d_{A}^{\perp}}
\newcommand{\CAc}{C_{A}^{\perp}}
\newcommand{\rhoA}{\rho_{A}}
\newcommand{\JS}{\mathrm{JS}}
\newcommand{\KL}{\mathrm{KL}}
\newcommand{\bos}{\langle\mathrm{bos}\rangle}
\newcommand{\cmark}{\checkmark}
\newcommand{\xmark}{$\times$}

\newtheorem{theorem}{Theorem}
\newtheorem{proposition}{Proposition}
\newtheorem{lemma}{Lemma}
\newtheorem{corollary}{Corollary}
\theoremstyle{definition}
\newtheorem{definition}{Definition}
\theoremstyle{remark}
\newtheorem{remark}{Remark}

\title{Which Question Is Your Attention Metric Answering?\\ Attention Rows as Compositional Data}

\author{Marios Papamichalis\thanks{Human Nature Lab, Yale University, New Haven, CT 06511, \texttt{marios.papamichalis@yale.edu}}\hspace{.2cm}\\
    Human Nature Lab, Yale University\\
    and\\
    Regina Ruane\thanks{Department of Statistics and Data Science, The Wharton School, University of Pennsylvania, 3733 Spruce Street, Philadelphia, PA 19104-6340, \texttt{ruanej@wharton.upenn.edu}}\\
    Department of Statistics and Data Science, The Wharton School,\\ University of Pennsylvania}

\begin{document}
\maketitle

\begin{abstract}
Each row of a transformer's attention matrix is a probability distribution over tokens, and in trained models most of that probability lands on a single \emph{sink} token, usually the first. Standard tools for comparing attention rows (cosine similarity, Jensen--Shannon divergence, Shannon entropy) therefore hinge on a choice papers rarely report: keep the sink, or drop it and renormalize. This choice can reverse conclusions. On ten pretrained models from five families, 17--47\% of verdicts about which of two heads is more similar flip with the convention, and the most prominent structure in a standard BERT head-clustering pipeline is an artifact of it. The reason is that one-number summaries mix two questions: how much attention the sink takes, and how the rest is divided among the content tokens. Treating rows as compositional data separates them exactly: the Aitchison distance splits orthogonally into a sink term and a content term, entropy splits by an exact identity, and the content distance is characterized by invariances the transformer itself possesses. The separation matters in practice: most measured entropy collapse during training is the sink growing, not attention sharpening (30\% of the drop at 70M parameters, 95\% at 1B, 79\% at 1.4B), and pruning heads with the wrong channel can inflate perplexity more than a hundredfold. We map where each convention is safe, test a frozen out-of-sample predictor (one confirmation, one abstention, one failure), and release code regenerating every number.
\end{abstract}

\section{Introduction}

\begin{figure}[b]
\centering
\includegraphics[width=0.74\textwidth]{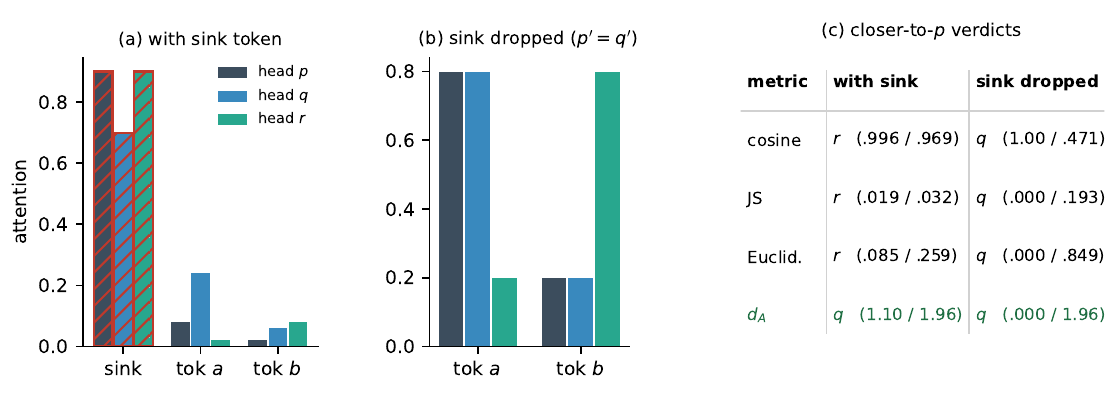}
\caption{\textbf{The reversal in one construction (Theorem~\ref{thm:reversal}).} Three attention rows over $\{\text{sink},a,b\}$: $p=(0.90,0.08,0.02)$, $q=(0.70,0.24,0.06)$, $r=(0.90,0.02,0.08)$. (a) With the sink column, $p$ and $r$ share the same sink mass while $q$ differs. (b) Dropping the sink and renormalizing, a routine, usually undocumented preprocessing choice, makes $p'$ and $q'$ \emph{identical} (both allocate content $4{:}1$ to $a{:}b$) while $r'$ mirrors them. (c) Cosine, JS, and Euclidean distance all say $r$ is the closer head with the sink and $q$ without it: the scientific conclusion is an artifact of the convention. The Aitchison distance returns the same verdict in both views, and its content part $\dAc$ is numerically identical.}
\label{fig:punchline}
\end{figure}

Softmax attention produces, for every query position, a probability distribution over keys, a point on the simplex. This is precisely a \emph{composition} in the sense of \citet{aitchison1982}: a vector of positive parts carrying only relative information. The interpretability and training-stability literatures routinely compare and summarize these rows. \citet{clark2019bert} cluster BERT heads by the JS divergence between attention rows; cosine similarity between attention maps is pervasive; \citet{zhai2023stabilizing} track the Shannon entropy of attention rows to diagnose \emph{entropy collapse}; head-role taxonomies and pruning scores \citep{voita2019analyzing,michel2019sixteen,olsson2022induction} summarize populations of attention rows with Euclidean averages. These statistics are valid for their intended targets, but full-row comparisons need not preserve conclusions after a coordinate is removed and the row is renormalized, and Shannon entropy combines marginal sink mass with conditional content entropy. Compositional data analysis (CoDA) identifies the properties at stake: perturbation invariance and subcompositional coherence \citep{aitchison1982,aitchison1992,egozcue2003ilr}.

CoDA critiques of Euclidean statistics date back four decades. What makes them consequential here is a second, recent empirical fact: attention mass concentrates overwhelmingly on a few \emph{sink} coordinates, chiefly the first or $\bos$ token \citep{xiao2024streaming,sun2024massive,gu2025sink}. In a typical Llama-3.1-405B prompt roughly $80\%$ of the attention mass sits on the $\bos$ token \citep{barbero2025first}, and the fraction of sink heads grows with model scale \citep{gu2025sink}. A sink is exactly the kind of dominant shared part whose inclusion or exclusion CoDA warns changes Euclidean and divergence-based comparisons. Whether an analysis keeps the sink column, or masks it and renormalizes, is today an undocumented analyst choice, and this choice can determine the answer: two heads that divide their non-sink attention identically can be scored as distant, and a head with the opposite content preference ranked their nearest neighbor, purely because all three share a large sink (Figure~\ref{fig:punchline}).

\textbf{Two estimands.} Full-row statistics answer questions about total allocation, sink included; sink-dropped statistics answer questions about allocation conditional on ignoring the sink. Both are legitimate; the failure mode is mixing them under an unstated convention, which is what one-number summaries do when a shared sink holds most of the mass, and these summaries feed head taxonomies, pruning decisions, and training-stability monitors. Our claim is not that the classical metrics are wrong but that they answer a convention-dependent question; we separate the channels exactly, characterize the content channel's canonical metric from the transformer's own symmetries, and measure what the separation changes. Concurrent 2026 work arrives at the same geometry for different questions (\S\ref{sec:related}); the measurement problem, the sink/content separation, and the convention audit are, to our knowledge, new.

\textbf{Contributions.}
\begin{itemize}[noitemsep,topsep=1pt,partopsep=0pt,parsep=0pt,leftmargin=*]
\item \textbf{Diagnosis.} We identify the keep-versus-drop sink convention as an unreported analyst choice and prove it reverses cosine, JS, and Euclidean closer-head verdicts on a positive-measure set (Theorem~\ref{thm:reversal}).
\item \textbf{Toolkit.} An exact orthogonal split of the Aitchison distance into sink and content terms, an exact entropy identity $H=H_b(s)+(1-s)H(\pi)$, invariance statements for the transformer's own nuisance transformations, and a five-axiom characterization with a quotient version for the content distance (Lemmas~\ref{lem:pythagoras},~\ref{lem:quotient}; Prop.~\ref{prop:invariances}; Thm.~\ref{thm:characterization}).
\item \textbf{Failure geometry.} In the sink-dominated regime all classical pairwise dissimilarities collapse into vanishing bands while the content diagnostics are unaffected (Theorem~\ref{thm:range}; Prop.~\ref{prop:taxonomy}).
\item \textbf{Measurement.} $17$--$47\%$ of closer-head verdicts flip on ten pretrained models, tracking curves computed before any model was run; prospectively specified re-analyses on taxonomies (including Clark et al.'s pipeline on its full support), scale, and training checkpoints are scored against their predictions (\S\ref{sec:realmodels}--\ref{sec:downstream}).
\item \textbf{Function.} As pruning criteria every metric fails somewhere; we map the regimes across ten models, show each criterion is stable along a different axis, and test a frozen regime predictor out of sample, reporting one confirmation, one abstention, and one failure (\S\ref{sec:realmodels}).
\end{itemize}

\section{Related work}
\label{sec:related}

\textbf{Attention-analysis metrics.} \citet{clark2019bert} compute head--head distances as summed JS divergences between attention rows and cluster the result; \citet{kovaleva2019dark} taxonomize attention patterns and study redundancy; cosine similarity between flattened attention maps is a default in many analyses. \textbf{Entropy collapse and stability.} \citet{zhai2023stabilizing} define attention entropy row-wise and tie pathologically low entropy to training instability; \citet{dong2021attention} study rank collapse of pure attention. \textbf{Attention sinks and registers.} StreamingLLM identifies the first-token sink and its role in windowed inference \citep{xiao2024streaming}; \citet{darcet2024registers} find high-norm register tokens in ViTs; \citet{sun2024massive} trace sinks to massive activations acting as implicit attention biases; \citet{gu2025sink} measure how the fraction of sink heads scales; \citet{barbero2025first} explain why models attend to the first token and quantify the mass involved. \textbf{Head roles and pruning.} \citet{voita2019analyzing} taxonomize heads and prune with $L_0$ gates; \citet{michel2019sixteen} prune most heads at test time; \citet{olsson2022induction} characterize induction heads. \textbf{Compositional data analysis.} The log-ratio approach originates with \citet{aitchison1982,aitchison1992}; ILR coordinates and balances come from \citet{egozcue2003ilr}; \citet{pawlowsky2015} is the standard treatment; zero handling follows \citet{martinfernandez2003zeros}; \citet{greenacre2011incoherence,greenacre2022reappraisal} quantify and reappraise subcompositional incoherence. \citet{quinn2020deepcoda} bring log-contrasts to deep models for compositional \emph{inputs}. \textbf{Concurrent compositional views of attention.} Three concurrent works arrive at this geometry from other directions. \citet{zhu2026aitchisonattention} introduce the Aitchison distance to quantify within-row token distinguishability in long contexts and derive a linear relation between temperature scaling and Aitchison distance, independently corroborating our powering axiom. \citet{lee2026invariants} identifies the row-centered attention logit with the CLR transform and studies spectral invariants of the resulting field. \citet{yamada2026polyilr} construct tree-aligned orthonormal bases of the Aitchison simplex and note the isomorphism between shift-equivalent logits and the CLR hyperplane. None studies the measurement problem addressed here: cross-head similarity, entropy diagnostics, taxonomies, and pruning under sink conventions (\citet{nakis2026aitchison} embed graph nodes as compositions for learned representations). Information geometry (Fisher--Rao) equips the simplex with a different metric with different invariances; it is not perturbation-invariant and does not decompose sink from content, which are the properties the sink problem demands (\S\ref{sec:discussion}).

To our knowledge, this is the first work to bring log-ratio geometry to the measurement methodology of attention analysis, and the first to separate the sink and content channels within it.

\section{Aitchison geometry of attention rows}
\label{sec:background}

\textbf{Compositions.} Let $\Sm{D}=\{p\in\mathbb{R}^D_{>0}:\sum_{i}p_i=1\}$ denote the (interior of the) simplex with $D$ parts, and $\clo(x)=x/\sum_i x_i$ the closure operation. A softmax attention row over $D$ keys is a point of $\Sm{D}$ (structural zeros from masking are handled below). The simplex is a $(D{-}1)$-dimensional real vector space under \emph{perturbation} $p\oplus q=\clo(p_1q_1,\dots,p_Dq_D)$ and \emph{powering} $\alpha\odot p=\clo(p_1^\alpha,\dots,p_D^\alpha)$, with identity the uniform composition $e$ and inverse $\ominus p=\clo(1/p_1,\dots,1/p_D)$ \citep{aitchison1982,pawlowsky2015}. Writing $g(p)$ for the geometric mean of the parts, the \emph{centered log-ratio} is $\clr(p)=\big(\log\tfrac{p_i}{g(p)}\big)_{i}$, mapping $\Sm{D}$ isomorphically onto the hyperplane $\mathcal{H}=\{x\in\mathbb{R}^D:\sum_i x_i=0\}$, and an \emph{isometric log-ratio} (ILR) map is $\ilr(p)=V^\top\clr(p)$ for any orthonormal contrast basis $V\in\mathbb{R}^{D\times(D-1)}$, $V^\top V=I$, $V^\top\mathbf{1}=0$ \citep{egozcue2003ilr}. The \emph{Aitchison inner product, norm, and distance} are the Euclidean ones pulled back through $\clr$ (equivalently any $\ilr$):
\begin{equation}
\dA(p,q)\;=\;\lVert\clr(p)-\clr(q)\rVert_2\;=\;\lVert\ilr(p)-\ilr(q)\rVert_2 .
\label{eq:da}
\end{equation}

\textbf{Subcompositions and coherence.} For $S\subseteq\{1,\dots,D\}$ the \emph{subcomposition} $p^{(S)}=\clo\big((p_i)_{i\in S}\big)$ discards the other parts and renormalizes, exactly what an analyst does when masking the sink column. \emph{Subcompositional coherence} demands that conclusions about the parts in $S$ not depend on whether the analysis was run on $p$ or on $p^{(S)}$; a dissimilarity is \emph{subcompositionally dominant} if $\delta\big(p^{(S)},q^{(S)}\big)\le\delta(p,q)$ \citep{aitchison1992,greenacre2011incoherence}. In ILR coordinates, taking a subcomposition is an orthogonal projection \citep[Lemma~3.1]{nakis2026aitchison}, so $\dA$ is dominant and log-ratios among the parts of $S$ are untouched by the operation. Cosine, JS, and Euclidean distance are neither subcompositionally coherent nor dominant, and none is stable under sink removal, as \S\ref{sec:incoherence} makes constructive.

\textbf{Sink/content notation.} We index the sink coordinate as $0$ and the remaining \emph{content} coordinates as $1,\dots,D{-}1$; any row factors uniquely as $p=(s,(1-s)\pi)$ with sink mass $s=p_0$ and content composition $\pi\in\Sm{D-1}$. We write $p'=p^{(\{1,\dots,D-1\})}=\pi$ for the sink-dropped, renormalized row. (Multiple sinks group identically; Appendix~\ref{app:experiments}.)

\textbf{Structural zeros and masking.} Causal masking creates structural zeros, which log-ratios forbid; the principled route is comparison on the \emph{common unmasked support}, a subcomposition, hence an orthogonal projection, and our protocol additionally fixes the key support across aggregated query positions (Appendix~\ref{app:experiments}). Residual numerical zeros are handled by multiplicative replacement with $\varepsilon=10^{-6}$ \citep{martinfernandez2003zeros,palarea2015zcomp}, and we report sensitivity to $\varepsilon\in\{10^{-5},10^{-6},10^{-7}\}$.

\section{Theory}
\label{sec:theory}

\subsection{The standard toolkit is subcompositionally incoherent}
\label{sec:incoherence}

Write $c(p,q)$ for cosine similarity, $\JS(p,q)$ for the Jensen--Shannon divergence (natural log), $\lVert p-q\rVert$ for Euclidean distance, and $H(p)=-\sum_i p_i\log p_i$ for Shannon entropy.

\begin{theorem}[Sink-driven ranking reversal]
\label{thm:reversal}
Fix $D\ge3$, index rows $x\in\Sm{D}$ by $0,\dots,D-1$ with coordinate $0$ the sink, and write $R_0(x):=(x_1,\dots,x_{D-1})/(1-x_0)\in\Sm{D-1}$. Fix $s'\in(0,1)$ and $\pi\in\Sm{D-1}$, for each permutation $\tau$ of the content coordinates with $\tau\pi\ne\pi$ and $s\in(s',1)$ set $p_s:=(s,(1-s)\pi)$, $q:=(s',(1-s')\pi)$, and $r_{s,\tau}:=(s,(1-s)\tau\pi)$; fix one such $\sigma$ and write $\widetilde\pi:=\sigma\pi$ and $r_s:=r_{s,\sigma}$. Then:
\emph{(i)} $R_0(p_s)=R_0(q)=\pi$ and $R_0(r_s)=\widetilde\pi$, so after sink removal $q$ is strictly closer to $p_s$ than $r_s$ under cosine, $\JS$, Euclidean distance, and \emph{every} metric on $\Sm{D-1}$.
\emph{(ii)} There is a threshold $s_\ast=s_\ast(\pi,s')\in(s',1)$, uniform over all permutations $\tau$ with $\tau\pi\ne\pi$, such that for every $s\in(s_\ast,1)$ and every such $\tau$ the with-sink verdict is simultaneously reversed for all three classical measures: $c(p_s,r_{s,\tau})>c(p_s,q)$, $\JS(p_s,r_{s,\tau})<\JS(p_s,q)$, and $\lVert p_s-r_{s,\tau}\rVert<\lVert p_s-q\rVert$.
\emph{(iii)} The set of triples in $(\Sm{D})^3$ on which all three classical verdicts flip between the two conventions is open and has positive $3(D-1)$-dimensional Lebesgue measure in the affine hull.
\emph{(iv)} $\dAc(x,y):=\dA(R_0(x),R_0(y))$ is a pseudometric on $\Sm{D}$ (a metric on $\Sm{D}/{\sim_0}$, $x\sim_0 y$ iff $R_0(x)=R_0(y)$); on the family, $\dAc(p_s,q)=0$ and $\dAc(p_s,r_s)=\dA(\pi,\widetilde\pi)=:B>0$ under \emph{both} conventions. Full $\dA$ satisfies $\dA(p_s,q)=\alpha_D\log\tfrac{s(1-s')}{s'(1-s)}$ ($\alpha_D=\sqrt{(D-1)/D}$) and $\dA(p_s,r_s)=B$, so it ranks $q$ closer exactly when $s<s_{\mathrm A}$ for an explicit cutoff $s_{\mathrm A}\in(s',1)$; the Table~\ref{tab:witness} witness satisfies $s<s_{\mathrm A}$, which is why full $\dA$ agrees across conventions there. (Proof: Appendix~\ref{app:incoherence}.)
\end{theorem}

\begin{table}[t]
\centering
\small
\caption{\textbf{Exact witness for Theorem~\ref{thm:reversal}} with $p=(0.90,0.08,0.02)$, $q=(0.70,0.24,0.06)$, $r=(0.90,0.02,0.08)$: $q$ shares $p$'s content ratio ($4{:}1$) and $r$ shares $p$'s sink mass with mirrored content. ``Closer'' means larger cosine similarity or smaller dissimilarity. All three classical verdicts reverse; the JS and Euclidean values also violate subcompositional dominance ($\JS(p',r')>\JS(p,r)$, $\lVert p'-r'\rVert>\lVert p-r\rVert$).}
\label{tab:witness}
\begin{tabular}{lcccccccc}
\toprule
 & \multicolumn{3}{c}{with sink} & & \multicolumn{3}{c}{sink dropped} & \\
\cmidrule{2-4}\cmidrule{6-8}
metric & $d(p,q)$ & $d(p,r)$ & closer & & $d(p',q')$ & $d(p',r')$ & closer & \\
\midrule
cosine & $0.9693$ & $0.9956$ & $r$ & & $1.0000$ & $0.4706$ & $q$ & reversal \\
JS & $0.0324$ & $0.0193$ & $r$ & & $0.0000$ & $0.1927$ & $q$ & reversal \\
Euclidean & $0.2592$ & $0.0849$ & $r$ & & $0.0000$ & $0.8485$ & $q$ & reversal \\
$\dA$ & $1.1022$ & $1.9605$ & $q$ & & $0.0000$ & $1.9605$ & $q$ & stable \\
\bottomrule
\end{tabular}
\end{table}

The witness entries are exact rationals; every displayed quantity is evaluated in closed form in Appendices~\ref{app:incoherence} and~\ref{app:characterization}.

\begin{proposition}[Entropy inversion under sink removal]
\label{prop:entropy}
For every $D\ge3$ there exist $p,q\in\Sm{D}$ with $H(p)<H(q)$ but $H(p')>H(q')$: entropy-collapse comparisons, across heads, layers, or training steps, can invert with the analyst's convention. A witness at $D=3$ is $P=\big(\tfrac{19}{20},\tfrac1{40},\tfrac1{40}\big)$, $Q=\big(\tfrac12,\tfrac{999}{2000},\tfrac1{2000}\big)$: $H(P)\approx0.233<0.697\approx H(Q)$, yet $H(P')=\log2>H(Q')\approx0.0079$. The convention-invariant content statistics agree with the conditional view: $H(P')=\log2$ is maximal and $\CAc(P)=0$ (content exactly uniform), versus $H(Q')\approx0.008$ and $\CAc(Q)\approx4.88$ (content nearly degenerate). (Proof and sharpness: Appendix~\ref{app:incoherence}.)
\end{proposition}

The mechanism is the chain rule $H(p)=H_b(s)+(1-s)H(\pi)$ (Appendix~\ref{app:incoherence}): a growing sink drives $H$ down regardless of what the content distribution does, so $H$ conflates two unrelated quantities, sink strength and content spread, that the Aitchison geometry separates exactly (Lemma~\ref{lem:pythagoras}). The constructive remedy is to report the pair $(s,H(\pi))$ in place of $H(p)$; $\CAc$ is a complementary log-ratio dispersion of the same content channel, not an entropy substitute (Remark~\ref{rem:schur}).

\subsection{A coherent toolkit, and why it is essentially unique}
\label{sec:toolkit}

\begin{definition}[Coherent attention diagnostics]
\label{def:toolkit}
For rows $p,q\in\Sm{D}$ with sink part $0$:
\emph{(a)} the \emph{head signature} of a set of rows $\{p_t\}$ is their Aitchison mean $\cen(\{p_t\})=\clo\big(\exp\big(\tfrac1T\sum_t\log p_t\big)\big)$, the Fr\'echet mean under $\dA$;
\emph{(b)} head dissimilarity is $\dA$ of \eqref{eq:da}, and the direction similarity is the ILR-cosine $\rhoA(p,q)=\tfrac{\langle\clr p,\clr q\rangle}{\lVert\clr p\rVert\lVert\clr q\rVert}$;
\emph{(c)} the \emph{sink balance} is $b(p)=\sqrt{\tfrac{D-1}{D}}\,\log\tfrac{p_0}{g(p_1,\dots,p_{D-1})}$;
\emph{(d)} the \emph{content distance} is $\dAc(p,q)=\dA(p',q')$ and the \emph{content log-ratio dispersion} is $\CAc(p)=\lVert\clr(p')\rVert_2=\dA(p',e)$, the Aitchison distance of the content to uniform (it orders by log-ratio dispersion, not by majorization; Remark~\ref{rem:schur});
\emph{(e)} dispersion of a set of rows is the total Aitchison variance $\tfrac1n\sum_i\dA\big(p_i,\cen\big)^2$.
\end{definition}

\begin{lemma}[Pythagorean sink decomposition]
\label{lem:pythagoras}
Let $D\ge2$, $p,q\in\Sm{D}$ with designated sink $0$, and $b(\cdot)$ the sink balance of Definition~\ref{def:toolkit}(c). In CLR space the exact orthogonal identity $\clr(p)=b(p)\,u_b+J\clr(p')$ holds, where $u_b$ is the unit balance direction and $J$ is the isometric zero-extension of the content CLR space; hence
\[
\dA(p,q)^2=\dAc(p,q)^2+\big(b(p)-b(q)\big)^2 .
\]
There is a sink-adapted ILR basis whose first coordinate is $b(\cdot)$ and whose remaining $D-2$ coordinates are ILR coordinates of the sink-dropped composition. Consequently $\dAc$ and $\CAc$ depend only on the content composition: varying the sink mass with the content held fixed changes neither, and both take identical values whether computed in the balance-orthogonal coordinates of the full rows or after dropping the sink and re-closing. (Proof: Appendix~\ref{app:characterization}.)
\end{lemma}

\begin{proposition}[Invariances and equivariances]
\label{prop:invariances}
With softmax rows $p=\clo(e^{\ell})$, $q=\clo(e^{\ell'})$:
\emph{(i)} a shared finite logit bias $u$ acts as perturbation by $c_u=\clo(e^u)$, and $\dA$ and $\dAc$ are invariant under it; $\CAc$ (for $D\ge3$), Euclidean distance, $\JS$, cosine, and $\rhoA$ are not.
\emph{(ii)} a temperature $\tau$ acts as powering by $1/\tau$: $\dA$, $\dAc$, and $\CAc$ scale by $1/\tau$ (so one common temperature preserves all rankings and ties), total Aitchison variance scales by $1/\tau^2$, and $\rhoA$ is invariant to \emph{separate} positive per-head temperatures applied uniformly within each head (row-dependent temperatures can change it).
\emph{(iii)} sink coherence: $\dAc(p,q)$ and $\CAc(p)$ are independent of the sink masses; full $\dA$ is not sink-invariant but decomposes exactly (Lemma~\ref{lem:pythagoras}).
\emph{(iv)} a subcomposition on $S$ ($|S|\ge2$) is an orthogonal projection in CLR space, so $\dA(p^{(S)},q^{(S)})\le\dA(p,q)$; hard masks require re-closure on a common positive support; and \emph{(v)} for $D\ge3$, Euclidean distance, $\JS$, and $1-\cos$ are \emph{not} subcompositionally dominant.
(Full statements with all qualifications, proofs, and scope: Appendix~\ref{app:characterization}.)
\end{proposition}

\textbf{The axioms are the transformer's own knobs.} Given the operational meaning the transformer itself assigns to $\oplus$ (shared logit biases) and $\odot$ (temperature), four interpretable axioms force the Aitchison distance, and the invariance axioms among them are transformations the architecture performs. A quotient version characterizes the content distance itself (Lemma~\ref{lem:quotient}).

\begin{theorem}[Characterization of the Aitchison distance]
\label{thm:characterization}
Let $\delta:\Sm{D}\times\Sm{D}\to\mathbb{R}_{\ge0}$ satisfy:
\emph{(A1)} $\delta$ is a metric;
\emph{(A2)} perturbation invariance: $\delta(c\oplus p,\,c\oplus q)=\delta(p,q)$ for all $c$;
\emph{(A3)} powering homogeneity: $\delta(\alpha\odot p,\,\alpha\odot q)=|\alpha|\,\delta(p,q)$ for all $\alpha\in\mathbb{R}$;
\emph{(A4)} the parallelogram law: writing $N(v)=\delta(v,e)$, $\;N(v\oplus w)^2+N(v\ominus w)^2=2N(v)^2+2N(w)^2$;
\emph{(A5)} permutation invariance: $\delta(\sigma p,\sigma q)=\delta(p,q)$ for all $\sigma\in S_D$.
Then $\delta=c\,\dA$ for some constant $c>0$.
\end{theorem}

The proof (Appendix~\ref{app:characterization}) reduces $\delta$ to a norm via A2, upgrades it to an inner-product norm via A4 by \citet{jordan1935inner}, and pins the inner product to a multiple of the CLR one via A5 and irreducibility of the permutation representation on the sum-zero hyperplane. The axioms are tight: $N_1(v)=\lVert\clr v\rVert_1$ satisfies A1--A3 and A5 but not A4 (Remark~\ref{rem:l1}). A2 and A3 are exactly invariance to shared logit biases and consistency under temperature, the two nuisance transformations the transformer parameterization itself supplies. Cross-dimensional comparisons need one common calibration constant, under which dominance, the sink decomposition, and the bias/temperature behavior follow rather than being assumed (Corollary~\ref{cor:consequences}, Remark~\ref{rem:aitchison-characterization-scope}).

\subsection{Consequences: dynamic range, taxonomies, and training curves}
\label{sec:consequences}

Three consequences sharpen what \S\ref{sec:experiments} measures.

\begin{theorem}[Dynamic-range collapse in the sink-dominated regime]
\label{thm:range}
Fix $D\ge3$ and $s_0\in(\tfrac12,1)$. \emph{(i)} If $p,q\in\Sm{D}$ have sink masses at least $s_0$, then, strictly,
\[
1-c(p,q)<\frac{(1-s_0)^2}{s_0^2+(1-s_0)^2},
\qquad
\lVert p-q\rVert<\sqrt{2}\,(1-s_0),
\qquad
\JS(p,q)<(1-s_0)\log 2,
\]
and each right-hand side is the exact, nonattained supremum (the attainable values fill $[0,\cdot)$); the cosine band shrinks quadratically, $\sim(1-s_0)^2$. \emph{(ii)} The content-preserving sink transformations $T_t r:=\big(t,(1-t)r'\big)$ leave the entire $\dAc$ dissimilarity matrix and all $\CAc$ values of any finite head family exactly unchanged; in particular, sinks can strengthen arbitrarily while the content geometry stays fixed. (Proof: Appendix~\ref{app:consequences}.)
\end{theorem}

Concretely, heads with sink mass $\ge0.8$ (Llama-3.2-1B's mean is $0.84$) have \emph{all} pairwise cosine similarities above $0.94$, Euclidean distances below $0.29$, and $\JS$ below $0.14$ nats: the classical dynamic range shrinks into a band comparable to estimation noise, so the orderings inside it, nearest neighbors, redundancy ranks, cluster merges, carry vanishing information (a mechanism candidate for the pruning catastrophe of \S\ref{sec:realmodels}).

\begin{proposition}[Taxonomies: stability and instability]
\label{prop:taxonomy}
\emph{(i)} For any finite labelled head set, the $\dAc$ dissimilarity matrix computed with the sink retained equals, entrywise, the $\dA$ matrix of the explicitly sink-dropped heads. Hence every clustering rule, deterministic or randomized, whose sole data-dependent input is this labelled matrix returns identical outputs in the two analyses (equally, in distribution and under a common coupling, when randomized). \emph{(ii)} There is a nonempty open set of head triples of positive $3(D-1)$-dimensional measure on which, simultaneously for $1-\cos$, $\JS$, and Euclidean distance, the unique closest pair is $(p,r)$ with the sink and $(p,q)$ after the drop; single- and complete-linkage clustering at $K{=}2$ therefore return $\{p,r\}\{q\}$ versus $\{p,q\}\{r\}$. The exact rational witness is Table~\ref{tab:witness}. (Proof: Appendix~\ref{app:consequences}.)
\end{proposition}

\begin{corollary}[Checkpoint entropy decomposition and sink-only collapse]
\label{cor:checkpoint}
Along any checkpoint family $p_t=\big(s_t,(1-s_t)\pi_t\big)$,
$H(p_t)=H_b(s_t)+(1-s_t)H(\pi_t)$ pointwise. If the content is fixed, $\pi_t\equiv\pi$, then $\CAc(p_t)$ is constant while $s\mapsto H\big(s,(1-s)\pi\big)$ is strictly decreasing on $[\tfrac12,1)$ (its unique maximizer is $\theta_\pi=\big(1+e^{H(\pi)}\big)^{-1}<\tfrac12$); with uniform content the apparent collapse can approach $\log D$ with $\CAc\equiv0$. This is a possibility result, not a claim that sink mass grows along any empirical trajectory. (Proof: Appendix~\ref{app:consequences}.)
\end{corollary}

\S\ref{sec:downstream} turns all three into prospectively specified measurements.

\section{Experiments}
\label{sec:experiments}

All numbers in this section are exact evaluations, Monte Carlo estimates, or single-seed model measurements produced by the released code with fixed seeds; \S\ref{sec:synthetic} isolates the mechanism in silico and \S\ref{sec:realmodels} measures it on pretrained models.

\subsection{Exact constructions and calibrated simulations}
\label{sec:synthetic}

\textbf{Exact witnesses.} Table~\ref{tab:witness} and Figure~\ref{fig:punchline} evaluate the Theorem~\ref{thm:reversal} witness; Proposition~\ref{prop:entropy} is instantiated in its statement; every invariance claim of Proposition~\ref{prop:invariances} is verified to machine precision in the released tests.

\textbf{How often do verdicts flip?} Figure~\ref{fig:reversal} samples head triples with sink structure calibrated to reported statistics \citep{barbero2025first,gu2025sink} and asks the basic analyst's question, \emph{which of $q,r$ is closer to $p$?}, under both sink conventions. At mean sink mass $\bar s=0.5/0.7/0.9$ the conventions disagree on $36/39/41\%$ (cosine), $23/31/40\%$ (JS), and $27/37/45\%$ (Euclidean) of triples, growing with sink mass toward the $50\%$ chance level. The content distance $\dAc$ is identical in both pipelines by Lemma~\ref{lem:pythagoras}, hence $0\%$ everywhere. The positive-measure region of Theorem~\ref{thm:reversal} is not thin at realistic sink levels; stars overlay the \S\ref{sec:realmodels} measurements at each model's signature sink mass.

\begin{figure}[t]
\centering
\includegraphics[width=0.72\textwidth]{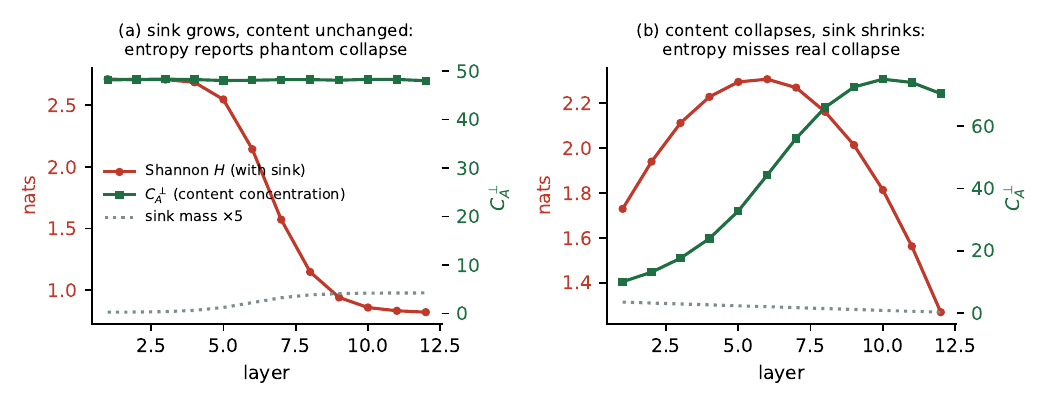}
\caption{\textbf{Entropy conflates sink strength with content spread; $\CAc$ separates them.} Synthetic 12-layer models, $4{,}000$ rows per layer. (a) Sink mass grows $0.06\to0.85$ with content log-ratio dispersion held fixed: mean Shannon entropy falls $2.71\to0.82$ nats ($-70\%$, a textbook ``entropy collapse'') while $\CAc$ moves $48.1\to48.0$ ($-0.3\%$): the collapse is phantom. (b) Content genuinely collapses while the sink recedes $0.70\to0.05$: entropy drifts $1.73\to1.27$ ($-27\%$, reads as mild) while $\CAc$ rises $10.0\to70.4$ ($7\times$): the real collapse is masked.}
\label{fig:collapse}
\end{figure}

\textbf{Collapse diagnostics.} Figure~\ref{fig:collapse} constructs two 12-layer scenarios. In Scenario (a), only the sink grows; Shannon entropy reports a $70\%$ collapse while content log-ratio dispersion is flat, exactly the failure Proposition~\ref{prop:entropy} predicts, and the mode to which entropy-collapse monitoring \citep{zhai2023stabilizing} is exposed when sinks strengthen \citep{gu2025sink}. In Scenario (b), the content genuinely collapses while the sink recedes; entropy barely moves. $\CAc$ responds in both cases because it is a statistic of the content subcomposition only.

\subsection{Pretrained models across five model families}

\begin{figure}[!t]
\centering
\includegraphics[width=0.50\textwidth]{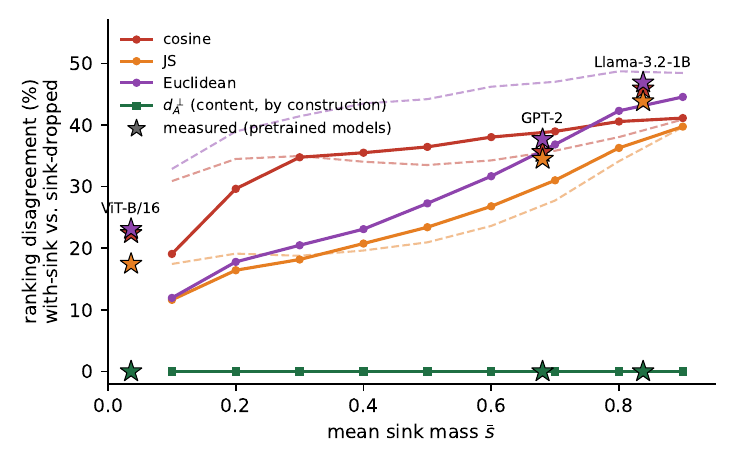}
\caption{\textbf{Calibrated reversal curves and ten measured models.} Curves: Monte Carlo closer-head disagreement rates between sink conventions for cosine, JS, and Euclidean distance, computed before any pretrained model was run (solid $\alpha{=}0.1$, dashed $\alpha{=}1$). Stars: the ten models, placed at the sink mass of the signatures the distances are computed on; mean absolute deviation from the solid curves is $4.8$ points (max $13$). The $\dAc$ rates are zero identically (Lemma~\ref{lem:pythagoras}), plotted as a code check.}
\label{fig:reversal}
\end{figure}
\label{sec:realmodels}

\textbf{Protocol} (released; two consumer GPUs, a few hours, no training; ten models, GPT-2, Pythia 70M--1.4B, Qwen2.5-1.5B, Llama-3.2-1B/3B, ViT-B/16; three resampling seeds each, varying data order and triple sampling for one fixed public checkpoint per model). \textbf{Extraction:} forward hooks on the attention softmax; for LMs, query positions $t\in[w,T]$ restricted to the first $W=64$ keys (a fixed common support; no zero imputation needed), with key $0$ (the $\bos$/first token) as the sink column; for ViT-B/16, the CLS column as sink. \textbf{Aggregation:} per-head Aitchison mean over query positions and inputs (Definition~\ref{def:toolkit}(a)). \textbf{Measurements:} (1) closer-head ranking-reversal rates over $50{,}000$ head triples between the with-sink and sink-dropped conventions, for cosine/JS/Euclidean versus $\dAc$; (2) layer-wise mean $H$, $H(\pi)$, and sink mass; (3) redundancy pruning at matched sparsity, prune the $m\in\{10,20,30\}\%$ most redundant heads (nearest-neighbor distance) under JS, under the content distance $\dAc$, and under the full $\dA$ (sink balance included), comparing perplexity (LMs) and top-1 agreement with the unpruned model (ViT, so no labels are needed); (4) an $\varepsilon$ sweep over $\{10^{-5},10^{-6},10^{-7}\}$. Data: $2{,}000$ WikiText-103 validation sequences (LMs), $2{,}048$ COCO val2017 images (ViT); details in Appendix~\ref{app:experiments}.

\begin{table}[!t]
\centering
\scriptsize\setlength{\tabcolsep}{4pt}\renewcommand{\arraystretch}{0.94}
\caption{\textbf{Measured convention-dependence on ten pretrained models} (mean$\pm$sd over three seeds). Reversal $=$ \% of $50{,}000$ head triples whose closer-head verdict flips between sink conventions (bit-identical across the $\varepsilon$ sweep). Rates rise with signature sink mass (Figure~\ref{fig:reversal}). Hell./FR: Hellinger reversal; Fisher--Rao is identical since both are strictly monotone in the Bhattacharyya coefficient, so every ranking verdict coincides. Collapse disagr.: fraction of adjacent-layer transitions on which $H$ and the content entropy $H(\pi)$ move in opposite directions. Sink (row/sig.): arithmetic row-level mean and the sink mass of the Aitchison-mean signatures the distances use. The content-distance column is omitted: its rate is zero by construction (Lemma~\ref{lem:pythagoras}).}
\label{tab:real}
\begin{tabular}{lcccccc}
\toprule
 & & \multicolumn{4}{c}{ranking-reversal rate} & \\
\cmidrule{3-6}
Model & sink (row/sig.) & cosine & JS & Euclid. & Hell./FR & collapse disagr. \\
\midrule
GPT-2 (124M) & $0.68/0.78$ & $36.1{\scriptstyle\pm 0.1}$\% & $35.0{\scriptstyle\pm 0.2}$\% & $38.1{\scriptstyle\pm 0.2}$\% & $34.4{\scriptstyle\pm 0.2}$\% & $27.3{\scriptstyle\pm 0.0}$\% \\
Llama-3.2-1B & $0.84/0.94$ & $46.1{\scriptstyle\pm 0.2}$\% & $44.1{\scriptstyle\pm 0.2}$\% & $47.2{\scriptstyle\pm 0.3}$\% & $43.6{\scriptstyle\pm 0.2}$\% & $46.7{\scriptstyle\pm 0.0}$\% \\
Qwen2.5-1.5B & $0.68/0.82$ & $43.9{\scriptstyle\pm 0.1}$\% & $42.5{\scriptstyle\pm 0.1}$\% & $44.3{\scriptstyle\pm 0.1}$\% & $42.3{\scriptstyle\pm 0.1}$\% & $29.6{\scriptstyle\pm 0.0}$\% \\
Llama-3.2-3B & $0.86/0.95$ & $45.0{\scriptstyle\pm 0.3}$\% & $43.7{\scriptstyle\pm 0.3}$\% & $46.7{\scriptstyle\pm 0.3}$\% & $43.5{\scriptstyle\pm 0.3}$\% & $37.0{\scriptstyle\pm 0.0}$\% \\
ViT-B/16 & $0.04/0.07$ & $22.4{\scriptstyle\pm 0.1}$\% & $17.3{\scriptstyle\pm 0.1}$\% & $23.2{\scriptstyle\pm 0.1}$\% & $18.1{\scriptstyle\pm 0.1}$\% & $9.1{\scriptstyle\pm 0.0}$\% \\
Pythia-70M & $0.08/0.43$ & $33.5{\scriptstyle\pm 0.2}$\% & $30.6{\scriptstyle\pm 0.2}$\% & $31.0{\scriptstyle\pm 0.2}$\% & $30.3{\scriptstyle\pm 0.2}$\% & $40.0{\scriptstyle\pm 0.0}$\% \\
Pythia-160M & $0.28/0.69$ & $43.2{\scriptstyle\pm 0.2}$\% & $43.9{\scriptstyle\pm 0.2}$\% & $44.3{\scriptstyle\pm 0.3}$\% & $43.4{\scriptstyle\pm 0.2}$\% & $18.2{\scriptstyle\pm 0.0}$\% \\
Pythia-410M & $0.52/0.74$ & $43.5{\scriptstyle\pm 0.2}$\% & $42.4{\scriptstyle\pm 0.2}$\% & $43.8{\scriptstyle\pm 0.2}$\% & $42.1{\scriptstyle\pm 0.2}$\% & $21.7{\scriptstyle\pm 0.0}$\% \\
Pythia-1B & $0.59/0.70$ & $41.0{\scriptstyle\pm 0.2}$\% & $40.3{\scriptstyle\pm 0.2}$\% & $41.3{\scriptstyle\pm 0.2}$\% & $40.3{\scriptstyle\pm 0.1}$\% & $33.3{\scriptstyle\pm 0.0}$\% \\
Pythia-1.4B & $0.64/0.80$ & $41.7{\scriptstyle\pm 0.2}$\% & $40.4{\scriptstyle\pm 0.1}$\% & $42.4{\scriptstyle\pm 0.2}$\% & $40.1{\scriptstyle\pm 0.1}$\% & $30.4{\scriptstyle\pm 0.0}$\% \\
\bottomrule
\end{tabular}
\vspace{5pt}
\caption{\textbf{Redundancy pruning at $20\%$ sparsity, all ten models} (mean$\pm$sd, three resampling seeds; $10/20/30\%$ sweep in Table~\ref{tab:prunefull}; random and total-$\dA$ baselines in Table~\ref{tab:prunev2}). LMs: perplexity, lower is better; ViT: top-1 agreement with the unpruned model, a stability rather than accuracy metric, higher is better. JS$'$ is JS on the sink-dropped rows; Hellinger on the content rows selects identical prune sets (Appendix~\ref{app:tailaudit}). ``-L1'' restricts neighbors to the same layer and defers a head whose nearest neighbor is already selected (one representative of each mutual pair is retained). Bold: best of the five criteria; the regime structure is described in the text.}
\label{tab:prune}
\setlength{\tabcolsep}{4pt}\renewcommand{\arraystretch}{0.92}
\begin{tabular}{lcccccc}
\toprule
Model & base & JS & JS$'$ & $\dAc$ & JS-L1 & $\dAc$-L1 \\
\midrule
GPT-2 (124M) & $43.5{\scriptstyle\pm 0.4}$ & $74.5{\scriptstyle\pm 0.7}$ & $106.3{\scriptstyle\pm 0.9}$ & $104.7{\scriptstyle\pm 0.7}$ & $\mathbf{70.6{\scriptstyle\pm 0.7}}$ & $448.5{\scriptstyle\pm 271.1}$ \\
Llama-3.2-1B & $15.3{\scriptstyle\pm 0.1}$ & $2094.1{\scriptstyle\pm 139.7}$ & $70.5{\scriptstyle\pm 9.6}$ & $63.6{\scriptstyle\pm 4.1}$ & $616.6{\scriptstyle\pm 30.2}$ & $\mathbf{45.0{\scriptstyle\pm 3.5}}$ \\
Qwen2.5-1.5B & $14.8{\scriptstyle\pm 0.1}$ & $111.5{\scriptstyle\pm 0.7}$ & $28.1{\scriptstyle\pm 0.2}$ & $30.3{\scriptstyle\pm 0.7}$ & $34.3{\scriptstyle\pm 1.1}$ & $\mathbf{26.9{\scriptstyle\pm 1.1}}$ \\
Llama-3.2-3B & $12.1{\scriptstyle\pm 0.1}$ & $310.8{\scriptstyle\pm 3.8}$ & $25.3{\scriptstyle\pm 0.1}$ & $40.2{\scriptstyle\pm 1.6}$ & $1733.2{\scriptstyle\pm 358.7}$ & $\mathbf{23.4{\scriptstyle\pm 3.0}}$ \\
ViT-B/16 & $1.000$ & $0.578{\scriptstyle\pm 0.003}$ & $0.557{\scriptstyle\pm 0.015}$ & $0.566{\scriptstyle\pm 0.013}$ & $0.572{\scriptstyle\pm 0.004}$ & $\mathbf{0.581{\scriptstyle\pm 0.006}}$ \\
Pythia-70M & $724.4{\scriptstyle\pm 10.7}$ & $\mathbf{661.7{\scriptstyle\pm 16.5}}$ & $1253.6{\scriptstyle\pm 19.0}$ & $1354.5{\scriptstyle\pm 16.1}$ & $1002.3{\scriptstyle\pm 15.9}$ & $1340.3{\scriptstyle\pm 24.5}$ \\
Pythia-160M & $131.4{\scriptstyle\pm 0.7}$ & $\mathbf{142.1{\scriptstyle\pm 1.8}}$ & $736.8{\scriptstyle\pm 15.7}$ & $721.5{\scriptstyle\pm 25.8}$ & $150.4{\scriptstyle\pm 11.3}$ & $476.4{\scriptstyle\pm 58.6}$ \\
Pythia-410M & $32.4{\scriptstyle\pm 0.1}$ & $75.2{\scriptstyle\pm 0.5}$ & $135.7{\scriptstyle\pm 11.4}$ & $194.8{\scriptstyle\pm 5.0}$ & $\mathbf{50.7{\scriptstyle\pm 0.7}}$ & $77.1{\scriptstyle\pm 13.7}$ \\
Pythia-1B & $21.8{\scriptstyle\pm 0.1}$ & $44.9{\scriptstyle\pm 0.1}$ & $63.0{\scriptstyle\pm 1.9}$ & $63.0{\scriptstyle\pm 1.9}$ & $\mathbf{37.6{\scriptstyle\pm 0.5}}$ & $47.4{\scriptstyle\pm 8.3}$ \\
Pythia-1.4B & $20.4{\scriptstyle\pm 0.1}$ & $33.4{\scriptstyle\pm 0.3}$ & $593.4{\scriptstyle\pm 41.4}$ & $377.8{\scriptstyle\pm 8.8}$ & $\mathbf{28.1{\scriptstyle\pm 0.1}}$ & $41.6{\scriptstyle\pm 2.2}$ \\
\bottomrule
\end{tabular}
\renewcommand{\arraystretch}{1.0}
\end{table}

\begin{figure}[!t]
\centering
\includegraphics[width=0.39\textwidth]{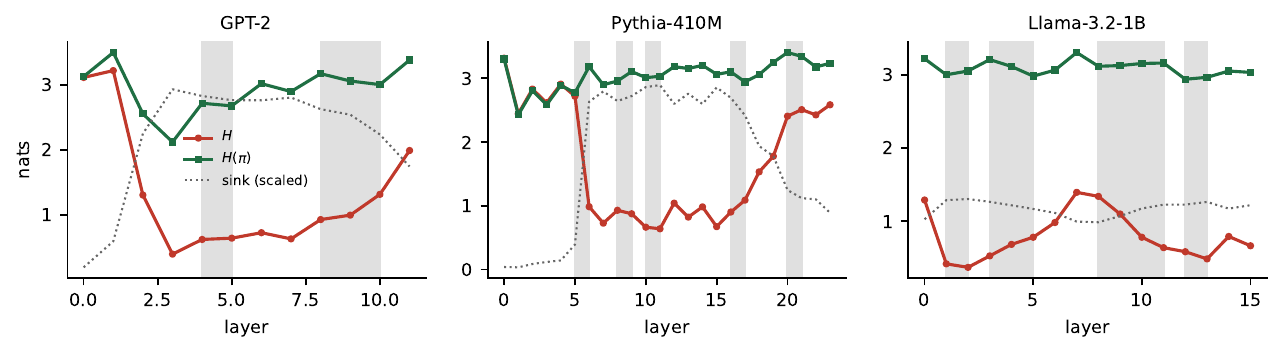}
\caption{\textbf{Measured layer-wise collapse decomposed.} Mean Shannon entropy $H$ (red) versus content entropy $H(\pi)$ (green), sink mass dotted (seed~0); shaded bands mark adjacent-layer transitions on which the two disagree in sign. Where the sink strengthens with depth, $H$ reports deepening collapse that the content channel does not show: Proposition~\ref{prop:entropy}'s conflation observed in pretrained models.}
\label{fig:realcollapse}
\end{figure}

\textbf{Results: measurement.} Table~\ref{tab:real} confirms the theory on ten pretrained models across five families. Row-level sink masses on the fixed $64$-key support span $0.04$ (ViT CLS) to $0.86$, corroborating \citet{barbero2025first}; within each family sink mass grows with scale. The signatures the distances are computed on concentrate further, geometric averaging amplifying a shared sink (row/signature masses in Table~\ref{tab:real}; Pythia-70M $0.08\to0.43$). Indexed by the signature sink, reversal rates rise from $17$--$23\%$ on ViT to $44$--$47\%$, approaching the $50\%$ chance level, on the Llamas, and track the calibrated curves of Figure~\ref{fig:reversal} with mean absolute deviation $4.8$ points (max $13$); reversal under $\dAc$ is zero by construction (verified to the bit); classical rates are $\varepsilon$-invariant with resampling sd $\le0.4$ points. The support choice is not decisive: the window retains $48$--$71\%$ of row mass, reversal moves at most $2$ points across $W\in\{32,64,128\}$, and $\dAc$ distances correlate at $\rho\ge0.87$ across windows (Appendix~\ref{app:clarkw}). Entropy and the content entropy $H(\pi)$ disagree in sign on $9$--$47\%$ of layer transitions across models (Figure~\ref{fig:realcollapse} shows three). (A measured Llama triple realizes Theorem~\ref{thm:reversal} verbatim: Fig.~\ref{fig:realtriple}.) Appendix~\ref{app:tailaudit} audits the estimator (tail coordinates carry $\le4.6\%$ of squared $\dAc$; dropped-row JS agrees on coarse ranking but not on nearest neighbors or taxonomies).

\begin{figure}[!t]
\centering
\includegraphics[width=0.44\textwidth]{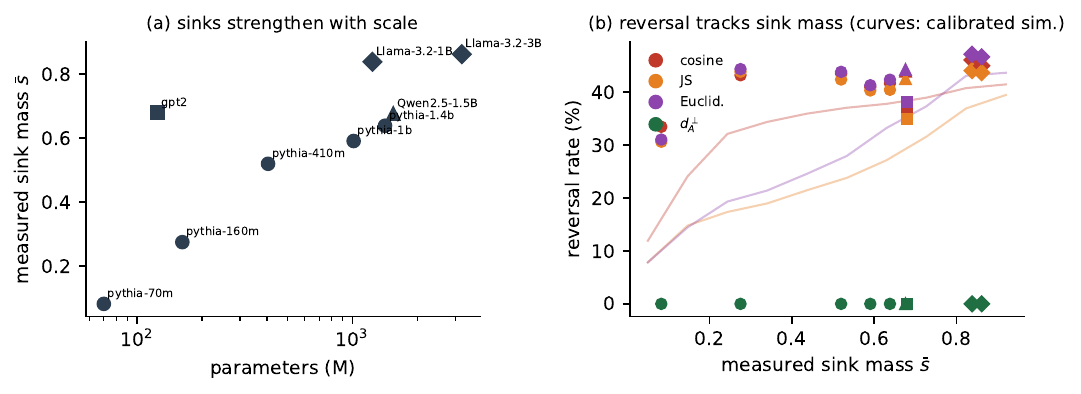}\hfill
\includegraphics[width=0.40\textwidth]{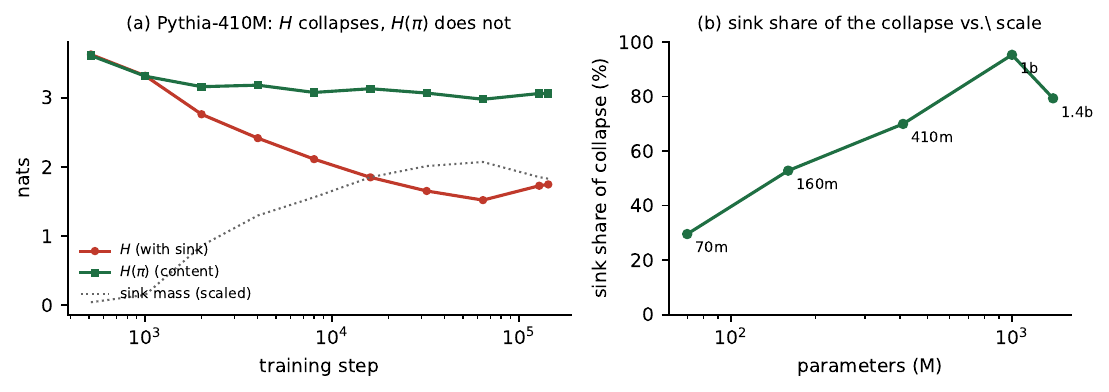}
\caption{\textbf{Prospectively specified downstream outcomes.} \emph{(a)} reversal tracks sink mass, which strengthens with scale within each family. \emph{(b)} at Pythia-410M $H$ collapses while $H(\pi)$ barely moves; the share $1-\Delta H(\pi)/\Delta H$ is $30/53/70/95/79\%$ across 70M--1.4B, rising to $95\%$ at 1B ($79\%$ at 1.4B).}
\label{fig:downstream}
\end{figure}

\textbf{Results: pruning.} Table~\ref{tab:prune} is a regime map. Under the global protocol, conditioning out the sink is the decisive step and the choice of geometry within the content channel is secondary: on Llama-3.2-1B, with-sink JS collapses ($15.3\to2094{\pm}140$) while dropped-row JS ($70.5{\pm}9.6$) and $\dAc$ ($63.6{\pm}4.1$) both avoid the failure, and dropped-row JS wins outright on Llama-3.2-3B ($25.3$ versus $40.2$); the full $\dA$, re-mixing the sink channel, fails like JS ($2480{\pm}211$), and random orderings beat global $\dAc$ on $15$ of $30$ sweep cells. The separation appears under a protocol closer to practice (same-layer neighbors, one representative kept per redundant pair): $\dAc$-L1 is the best of all nine criteria on every strong-sink model ($45.0$, $26.9$, $23.4$ on Llama-1B/Qwen/Llama-3B; $0.581$ on ViT), whereas each classical criterion fails severely on at least one of them (JS$'$-L1 reaches $2642$ on Llama-3.2-1B; JS-L1 reaches $1733$ on Llama-3.2-3B). $\dAc$-L1 in turn fails on GPT-2 ($449$), and with-sink JS(-L1) wins wherever sinks are weak. Across resampling seeds the $\dAc$-L1 sets are the less stable (Jaccard $0.36$--$0.77$ versus $0.92$--$1.00$ for JS-L1) even where their outcomes are stable: each criterion is stable along a different axis. Sink mass alone does not predict the regime (GPT-2 and Qwen2.5 share $s=0.68$ with opposite verdicts); the cross-head dispersion of the sink share separates all nine LMs (coefficient of variation, CV, $\le0.36$ in the content regime, $\ge0.42$ in the JS regime); a frozen out-of-sample test on three unseen models returned one confirmation, one boundary abstention, and one failure, so the rule is presented as descriptive, not validated (Appendix~\ref{app:conventions}). The experiments support two claims: measurements feeding any criterion are convention-dependent under classical metrics (Table~\ref{tab:real}); which channel carries the functional signal is model-dependent, a question the decomposition makes precise. Two stress tests separate the Aitchison choice from simply dropping the sink: rescaling the content logits by an unobservable temperature $\tau\in[0.5,2]$ leaves every $\dAc$ verdict and prune set bit-identical (Proposition~\ref{prop:invariances}) while dropped-row JS changes up to $38\%$ of its top-$20\%$ prune-set membership, and across sink definitions (first-1 versus first-4) $\dAc$ is the less sensitive on every model, with nested-projection dominance verified (Appendix~\ref{app:conventions}).

\subsection{Prospectively specified downstream tests}
\label{sec:downstream}

Are \emph{conclusions} built on these measurements convention-dependent too? Three predictions were stated before the runs (in the released repository; no external registry); we report each outcome.

\emph{(1) Taxonomies.} Clark-style hierarchical head clustering (agglomerative, $K\in\{4,6,8\}$) on BERT-base, GPT-2, and Llama-3.2-1B \citep{clark2019bert}. Prediction (Proposition~\ref{prop:taxonomy}): classical between-convention ARIs sit materially below $1$, decreasing with sink mass; the $\dAc$ partitions are identical by construction. \emph{Outcome}: at $64$ keys, ARIs sit at or below chance on all three (Table~\ref{tab:taxonomy}). On the full $128$-token support with \citet{clark2019bert}'s own pipeline (JS, average linkage), cross-convention ARI is $0.01$--$0.05$, and the SEP-head block ($68$/$144$ heads) drops from Jaccard $0.87$--$0.92$ to $\approx0.48$ when the sink is dropped (Appendix~\ref{app:clarkw}): its most prominent structure is the convention.

\emph{(2) Scale.} Sink mass and reversal rate across Pythia 70M--1.4B; prediction (Theorem~\ref{thm:range}; \citealp{gu2025sink}): both increase with parameters. \emph{Outcome}: sink mass is monotone within each family and reversal tracks it near the pre-computed curves (Fig.~\ref{fig:downstream}a). \emph{(3) Training curves.} $H$ versus $\CAc$ across Pythia checkpoints; prediction (Corollary~\ref{cor:checkpoint}): during sink emergence $H$ falls with $\CAc$ approximately flat. \emph{Outcome on the frozen endpoint: not confirmed.} $\CAc$ changes by $57/72/85\%$ at 70M/160M/410M ($16\%$ at 1B, $42\%$ at 1.4B), failing at the three smaller sizes and approximately holding at 1B. \emph{Exploratory analysis} via the exact identity $H=H_b(s)+(1-s)H(\pi)$ (Fig.~\ref{fig:downstream}b): at 410M $H$ collapses $3.63\to1.52$ while $H(\pi)$ moves $0.6$ nats, and at 1B $H(\pi)$ moves $0.07$ nats through a $1.6$-nat collapse; the share $1-\Delta H(\pi)/\Delta H$ (the fraction of the drop not attributable to a fall in $H(\pi)$) is $30/53/70/95/79\%$ across 70M--1.4B, rising to $95\%$ at 1B and remaining high ($79\%$) at 1.4B (signed decomposition: Appendix~\ref{app:conventions}).

\section{Conclusion}
\label{sec:discussion}

Attention rows are compositions, and the field's standard summaries mix a sink question with a content question under an unreported convention. We separated the channels exactly, showed the separation changes published-style conclusions (verdict flips near chance, a clustering whose main structure is the convention, a collapse that is mostly sink), and mapped when each channel matters functionally. \textbf{What the theory does and does not claim.} Theorems~\ref{thm:reversal}--\ref{thm:characterization} are statements about \emph{measurement}: for relative allocation among content tokens, the classical toolkit answers a convention-dependent question, the Aitchison toolkit a well-posed one. We make no claim that computation ``uses'' log-ratios. \textbf{Limitations.} The scale law rests on one family; the frozen regime test returned one confirmation, one abstention, one failure (OPT-125M: uniformly high sink, benign pruning, a third behavior the map lacks); $\dAc$-L1 selects unstable sets across resampling; audits are signature-level rather than per-row. \textbf{Future work:} fitting $d_\lambda^2=(\dAc)^2+\lambda(\Delta b)^2$ from held-out functional evidence, per-row audits, and value-weighted extensions.

\bibliography{refs}
\bibliographystyle{plainnat}

\appendix

\newtheorem{alemma}{Lemma}[section]
\newtheorem{atheorem}{Theorem}[section]
\newtheorem{aremark}{Remark}[section]

\section{Worked example and the geometry of the reversal}
\label{app:worked}

\textbf{The witness in full.} The triple of Figure~\ref{fig:punchline} and Table~\ref{tab:witness} is $p=(0.90,0.08,0.02)$, $q=(0.70,0.24,0.06)$, $r=(0.90,0.02,0.08)$ over parts $(\text{sink},a,b)$. In the parameterization of Theorem~\ref{thm:reversal}: $\pi=(0.8,0.2)$, $\tilde\pi=(0.2,0.8)$ (the transposition of the two content parts), $s=0.90$, $s'=0.70$. All values in Table~\ref{tab:witness} are exact evaluations (code: \texttt{experiments/exact\_examples.py}). The sink balances are $b(p)=b(r)=2.5422$ and $b(q)=1.4400$, and the Pythagorean decomposition of Lemma~\ref{lem:pythagoras} verifies to machine precision:
\[
\dA(p,q)^2=1.214868=0+(2.5422-1.4400)^2,\qquad
\dA(p,r)^2=3.843624=1.9605^2+0 .
\]
The $(p,q)$ discrepancy is entirely sink balance; the $(p,r)$ discrepancy is entirely content. Cosine, JS, and Euclidean distance mix the two components in a renormalization-dependent way, which is precisely what the reversal exploits.

\begin{figure}[t]
\centering
\includegraphics[width=0.5\textwidth]{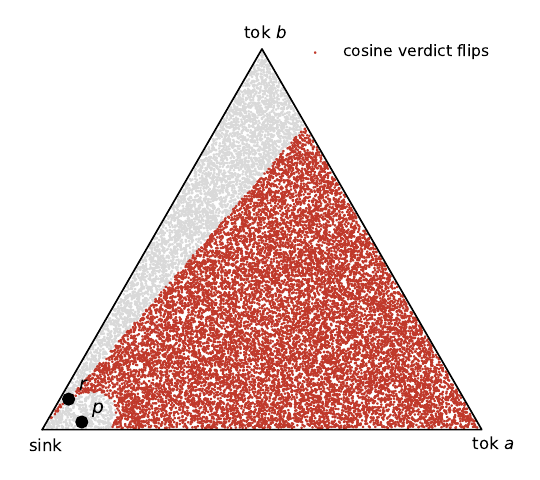}
\caption{\textbf{The reversal region is not thin.} Fixing $p$ and $r$ of the witness and sweeping the third row $q$ over the interior of the simplex ($20{,}000$ Dirichlet$(1,1,1)$ samples with all parts $>0.004$), the cosine verdict about which of $q,r$ is closer to $p$ differs between the with-sink and sink-dropped conventions on $77.9\%$ of positions (red).}
\label{fig:wedge}
\end{figure}

\textbf{How large is the flip region for one anchor pair?} Figure~\ref{fig:wedge} fixes $(p,r)$ and sweeps $q$: the cosine verdict is convention-dependent on $77.9\%$ of the simplex. This complements the positive-measure statement of Theorem~\ref{thm:reversal}(iii): near strong sinks, convention-dependence is the typical case, not the exception.

\textbf{Transparency note: total $\dA$ versus content $\dAc$.} In the Monte Carlo of Figure~\ref{fig:reversal}, the \emph{total} Aitchison distance and the \emph{content} distance $\dAc$ disagree about the closer head on $0.1\%$ of triples, at every sink level. The two are different, explicitly declared estimands; whenever they disagree the discrepancy is exactly the sink-balance term of Lemma~\ref{lem:pythagoras}, in contrast to the undisclosed convention-dependence of cosine/JS/Euclidean. An analyst who wants sink-inclusive similarity reports $\dA$; one who wants content similarity reports $\dAc$; both are stable under the keep-vs-drop choice.

\begin{figure}[h]
\centering
\includegraphics[width=0.92\textwidth]{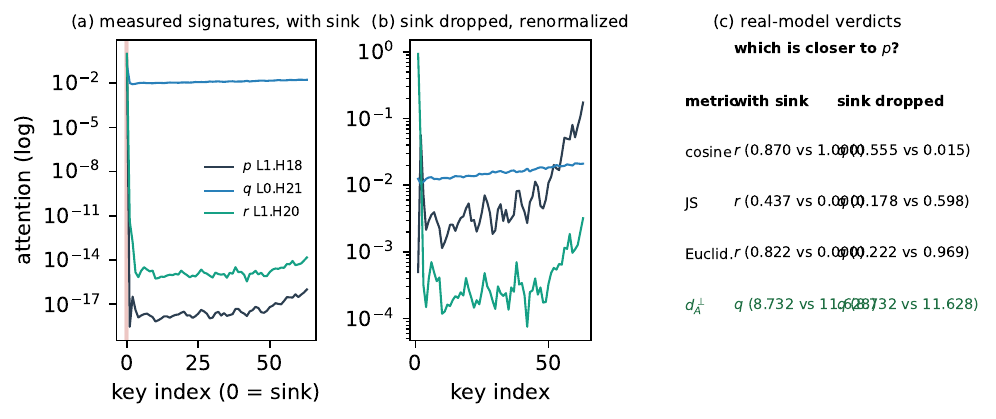}
\caption{\textbf{A measured Llama-3.2-1B head triple realizing Theorem~\ref{thm:reversal}.} Found automatically by \texttt{find\_real\_triple.py} from the saved seed-0 signatures: with the sink, cosine, JS, and Euclidean distance all rank $r$ closer to $p$; after dropping the sink and re-closing, all three rank $q$ closer; the $\dAc$ verdict is identical in both pipelines.}
\label{fig:realtriple}
\end{figure}

\section{Proofs for Section~\ref{sec:incoherence}}
\label{app:incoherence}

\paragraph{Standing convention.}
For every integer $m\geq2$, let
$\Sm{m}:=\{x\in(0,\infty)^m:\mathbf 1^\top x=1\}$.
For $x\in\Sm{m}$, let $\log x$ denote the componentwise logarithm and write
\[
\clr_m(x):=\log x-\frac1m\bigl(\mathbf 1^\top\log x\bigr)\mathbf 1 .
\]
Whenever $\dA$ is applied to two $m$-part compositions, it denotes the unrescaled CLR distance
$\dA(x,y):=\lVert\clr_m(x)-\clr_m(y)\rVert_2$;
the dimension of the CLR map is determined by the arguments of $\dA$, so full-row and sink-removed distances use $\clr_D$ and $\clr_{D-1}$, respectively. This is a metric on $\Sm{m}$: if $\clr_m(x)=\clr_m(y)$, then $\log x-\log y=c\mathbf 1$ for some $c\in\mathbb R$, so $x=e^c y$; since both vectors sum to one, $c=0$ and $x=y$. All logarithms are natural, and $H(x):=-\sum_j x_j\log x_j$ with $0\log0:=0$.

\begin{proof}[Proof of Theorem~\ref{thm:reversal}]
Put
\[
t:=1-s,\qquad
u:=1-s',\qquad
a:=\lVert\pi\rVert_2^2,\qquad
\beta:=\langle\pi,\widetilde\pi\rangle,\qquad
\Delta:=a-\beta.
\]
Because $\widetilde\pi$ is a permutation of $\pi$,
$\lVert\widetilde\pi\rVert_2^2=a$. Cauchy--Schwarz gives
$\beta\leq a$, and equality would force $\widetilde\pi$ to be a
positive scalar multiple of $\pi$. Since both vectors sum to one,
equality would imply $\widetilde\pi=\pi$, contrary to assumption.
Thus
\[
0<a<1,\qquad
0<\beta<a,\qquad
0<\Delta<a,\qquad
\lVert\pi-\widetilde\pi\rVert_2^2=2\Delta.
\]
Here
\[
1-a=2\sum_{i<j}\pi_i\pi_j>0,
\]
because $D-1\geq2$ and every coordinate of $\pi$ is positive.

\emph{Proof of \textup{(i)}.}
The identities for $R_0(p_s),R_0(q),R_0(r_s)$ follow directly from
the definition. Hence $R_0(p_s)=R_0(q)$ and
$R_0(p_s)\ne R_0(r_s)$. Metric definiteness proves the Euclidean and
arbitrary-metric assertions. Jensen--Shannon divergence is zero if
and only if its arguments agree. Finally,
\[
c\bigl(R_0(p_s),R_0(q)\bigr)=1,
\qquad
c\bigl(R_0(p_s),R_0(r_s)\bigr)=\frac{\beta}{a}<1.
\]

\emph{Proof of \textup{(ii)}.}
Direct calculation gives
\begin{align}
\lVert p_s-q\rVert_2^2
  &=(s-s')^2(1+a)=(u-t)^2(1+a),
  \label{eq:rev-euclid-pq}\\
\lVert p_s-r_s\rVert_2^2
  &=2\Delta t^2,
  \label{eq:rev-euclid-pr}\\
c(p_s,r_s)
  &=\frac{(1-t)^2+\beta t^2}{(1-t)^2+a t^2}
   =1-\frac{\Delta t^2}{(1-t)^2+a t^2}.
  \label{eq:rev-cos-pr}
\end{align}
Also,
\begin{equation}
c(p_s,q)
=\frac{(1-t)s'+tua}
{\sqrt{(1-t)^2+at^2}\sqrt{(s')^2+au^2}}.
\label{eq:rev-cos-pq}
\end{equation}

For Jensen--Shannon divergence,
\[
m_{pq}:=\frac{p_s+q}{2}
=\left(\frac{s+s'}2,\frac{t+u}{2}\pi\right).
\]
Consequently,
\begin{align*}
\KL(p_s\Vert m_{pq})
&=s\log\frac{2s}{s+s'}+t\log\frac{2t}{t+u},\\
\KL(q\Vert m_{pq})
&=s'\log\frac{2s'}{s+s'}+u\log\frac{2u}{t+u}.
\end{align*}
Therefore
\begin{equation}
\JS(p_s,q)=\JS_{\rm Ber}(s,s'),
\label{eq:rev-js-pq}
\end{equation}
where
\[
\JS_{\rm Ber}(v,w)
:=\JS\bigl((v,1-v),(w,1-w)\bigr).
\]
Similarly,
\[
m_{pr}:=\frac{p_s+r_s}{2}
=\left(s,t\frac{\pi+\widetilde\pi}{2}\right),
\]
and direct substitution gives
\begin{equation}
\JS(p_s,r_s)=t\,\JS(\pi,\widetilde\pi).
\label{eq:rev-js-pr}
\end{equation}

Set
\[
\gamma:=\frac{s'}{\sqrt{(s')^2+a(1-s')^2}}\in(0,1),
\qquad
\eta:=1-\gamma>0,
\]
and define
\[
\varepsilon_\ast
:=\min\left\{
\frac{u}{2},
\frac{u^2}{8\log 2},
\frac{\eta}{4\sqrt{1+a}},
\sqrt{\frac{\eta}{8a}}
\right\},
\qquad
s_\ast:=1-\varepsilon_\ast.
\]
Every entry in the minimum is strictly positive. Moreover,
$\varepsilon_\ast\leq u/2<u$, so $s_\ast\in(s',1)$.
This threshold depends on $s'$ and on $\pi$ only through
$a=\lVert\pi\rVert_2^2$; it is therefore uniform over all admissible
permutations.

Suppose $s>s_\ast$, so $0<t<\varepsilon_\ast$. Use
\[
\operatorname{TV}(P,Q):=\tfrac12\lVert P-Q\rVert_1.
\]
If $P,Q$ are Bernoulli laws with parameters $s,s'$ and
$M=(P+Q)/2$, then
\[
\operatorname{TV}(P,M)=\operatorname{TV}(Q,M)
=\frac{|s-s'|}{2}.
\]
Pinsker's inequality applied to both terms gives
\[
\JS_{\rm Ber}(s,s')
\geq\frac{(s-s')^2}{2}
=\frac{(u-t)^2}{2}.
\]
Furthermore, $\JS(\pi,\widetilde\pi)\leq\log 2$. Indeed, if
$m=(\pi+\widetilde\pi)/2$, then
$m_i\geq\pi_i/2$ and $m_i\geq\widetilde\pi_i/2$, so each KL divergence
in the definition of $\JS$ is at most $\log 2$. Consequently,
\[
\JS(p_s,q)
\geq\frac{(u-t)^2}{2}
>\frac{u^2}{8}
>t\log 2
\geq\JS(p_s,r_s).
\]

The Euclidean identities, together with $\Delta<a<1$ and $t<u/2$,
give
\begin{align*}
\lVert p_s-r_s\rVert_2
&=t\sqrt{2\Delta}
 <t\sqrt{2a}
 <\frac{u}{2}\sqrt{2a}
 <\frac{u}{2}\sqrt{1+a}\\
&<(u-t)\sqrt{1+a}
 =\lVert p_s-q\rVert_2.
\end{align*}

It remains to compare cosine similarities. Let
$e_0=(1,0,\ldots,0)$ and $\widehat x=x/\lVert x\rVert_2$ for
$x\ne0$. Since $t<u/2<1/2$,
\[
(1-t)^2+at^2>\frac14,
\]
and equation~\eqref{eq:rev-cos-pr} gives
\[
c(p_s,r_s)>1-4at^2>1-\frac{\eta}{2}.
\]
Moreover, $\gamma=\langle e_0,\widehat q\rangle$ and
\begin{align*}
\bigl|c(p_s,q)-\gamma\bigr|
&=\bigl|\langle\widehat p_s-e_0,\widehat q\rangle\bigr|
 \leq\lVert\widehat p_s-e_0\rVert_2\\
&\leq2\lVert p_s-e_0\rVert_2
 =2t\sqrt{1+a}
 <\frac{\eta}{2}.
\end{align*}
Here
\[
\begin{aligned}
\lVert\widehat p_s-e_0\rVert_2
&\leq\lVert\widehat p_s-p_s\rVert_2+\lVert p_s-e_0\rVert_2\\
&=|1-\lVert p_s\rVert_2|+\lVert p_s-e_0\rVert_2\\
&\leq2\lVert p_s-e_0\rVert_2,
\end{aligned}
\]
where the last inequality is the reverse triangle inequality.
Since $1-\eta/2=\gamma+\eta/2$,
\[
c(p_s,r_s)>1-\frac{\eta}{2}
=\gamma+\frac{\eta}{2}>c(p_s,q).
\]
This proves all three with-sink inequalities. Part~\textup{(i)}
then gives the simultaneous reversals.

The argument above applies verbatim to every permutation $\tau$ with $\tau\pi\ne\pi$ (replace $\widetilde\pi$ by $\tau\pi$); since there are finitely many such $\tau$, taking $s_\ast$ to be the maximum of the per-$\tau$ thresholds yields the threshold of (ii), uniform over $\tau$.

\emph{Proof of \textup{(iii)}.}
For $z=(x,y,w)\in(\Sm{D})^3$, define
\begin{align*}
F_1(z)&:=c(x,w)-c(x,y),\\
F_2(z)&:=c\bigl(R_0(x),R_0(y)\bigr)
        -c\bigl(R_0(x),R_0(w)\bigr),\\
F_3(z)&:=\JS(x,y)-\JS(x,w),\\
F_4(z)&:=\JS\bigl(R_0(x),R_0(w)\bigr)
        -\JS\bigl(R_0(x),R_0(y)\bigr),\\
F_5(z)&:=\lVert x-y\rVert_2^2-\lVert x-w\rVert_2^2,\\
F_6(z)&:=\lVert R_0(x)-R_0(w)\rVert_2^2
        -\lVert R_0(x)-R_0(y)\rVert_2^2.
\end{align*}
Then
\[
\mathcal R_D=\bigcap_{j=1}^6F_j^{-1}\bigl((0,\infty)\bigr).
\]
Squaring a nonnegative distance does not change a strict comparison.
Moreover,
\[
1-x_0=\sum_{i=1}^{D-1}x_i>0
\]
on $\Sm{D}$, so $R_0$ is smooth there. Cosine similarity,
Jensen--Shannon divergence, and squared Euclidean distance are
continuous on the relevant positive simplices. Thus every $F_j$ is
continuous. Choosing any $s\in(s_\ast,1)$,
parts~\textup{(i)}--\textup{(ii)} give
\[
(p_s,q,r_s)\in\mathcal R_D,
\]
so $\mathcal R_D$ is nonempty.

Let
\[
\mathcal A_D
:=\left\{x\in\mathbb R^D:
\sum_{i=0}^{D-1}x_i=1\right\}.
\]
Because
\[
\Sm{D}=\mathcal A_D\cap(0,\infty)^D,
\]
the set $(\Sm{D})^3$ is open in the $3(D-1)$-dimensional affine
space $\mathcal A_D^3$. Hence $\mathcal R_D$ is open in
$\mathcal A_D^3$. Every nonempty open subset of that affine space
contains a nondegenerate relative Euclidean ball and therefore has
positive intrinsic $3(D-1)$-dimensional Lebesgue measure. This proves
part~\textup{(iii)}.

\emph{Proof of \textup{(iv)}.}
Because $\dA$ is a metric and $R_0$ is a map, its pullback $\dAc$ is
nonnegative and symmetric and satisfies the triangle inequality.
Moreover,
\[
\dAc(x,y)=0
\quad\Longleftrightarrow\quad
R_0(x)=R_0(y)
\quad\Longleftrightarrow\quad
x\sim_0y.
\]
Therefore $\dAc$ is a pseudometric on $\Sm{D}$, and
\[
d_{A,0}([x],[y]):=\dAc(x,y)
\]
defines a genuine metric on $\Sm{D}/{\sim_0}$. It is well defined
because $x\sim_0x'$ and $y\sim_0y'$ imply
$R_0(x)=R_0(x')$ and $R_0(y)=R_0(y')$.

The identities for $\dAc$ follow from its definition and
part~\textup{(i)}. Moreover, $B>0$ because $\dA$ is a metric and
$\pi\ne\widetilde\pi$.

Put
\[
L_s:=\log\!\left(\frac{s(1-s')}{s'(1-s)}\right)>0.
\]
Direct centering of the coordinatewise log-ratio gives
\[
\clr_D(p_s)-\clr_D(q)
=\frac{L_s}{D}(D-1,-1,\ldots,-1).
\]
Therefore
\[
\dA(p_s,q)
=\left\lVert\clr_D(p_s)-\clr_D(q)\right\rVert_2
=\alpha_D L_s
=\alpha_D\bigl(\operatorname{logit}(s)
-\operatorname{logit}(s')\bigr).
\]

Next set
\[
h_i:=\log\frac{\pi_i}{\widetilde\pi_i}.
\]
Because $\widetilde\pi$ is a permutation of $\pi$,
\[
\sum_{i=1}^{D-1}h_i
=\log\!\left(
\frac{\prod_{i=1}^{D-1}\pi_i}
     {\prod_{i=1}^{D-1}\widetilde\pi_i}
\right)=0.
\]
Consequently,
\[
\clr_D(p_s)-\clr_D(r_s)=(0,h_1,\ldots,h_{D-1}),
\]
while
\[
\clr_{D-1}(\pi)-\clr_{D-1}(\widetilde\pi)
=(h_1,\ldots,h_{D-1}).
\]
Thus
\[
\dA(p_s,r_s)=\dA(\pi,\widetilde\pi)=B.
\]

Let
\[
E:=\exp(B/\alpha_D)>1.
\]
Then
\[
s_{\mathrm A}=\frac{s'E}{1-s'+s'E},
\]
and
\[
s_{\mathrm A}-s'
=\frac{s'(1-s')(E-1)}{1-s'+s'E}>0,
\qquad
1-s_{\mathrm A}
=\frac{1-s'}{1-s'+s'E}>0.
\]
Hence $s_{\mathrm A}\in(s',1)$. Moreover,
\[
L_{s_{\mathrm A}}=\log E=\frac{B}{\alpha_D},
\qquad
\frac{d}{ds}L_s=\frac{1}{s(1-s)}>0.
\]
Therefore $\dA(p_s,q)<B$ exactly when $s<s_{\mathrm A}$, with
equality at $s=s_{\mathrm A}$ and the reverse inequality when
$s>s_{\mathrm A}$.

In the special case $D=3$, for the Table~\ref{tab:witness} values,
\[
(h_1,h_2)=(\log4,-\log4),
\qquad
L_s=\log\frac{(9/10)(3/10)}{(7/10)(1/10)}
=\log\frac{27}{7}.
\]
Hence
\[
B=\sqrt2\log4,
\qquad
\dA(p_s,q)=\sqrt{\frac23}\log\frac{27}{7}.
\]
Finally,
\[
\sqrt{\frac23}\log\frac{27}{7}
<\sqrt{\frac23}\log4
<\sqrt2\log4,
\]
which completes the proof.
\end{proof}

\begin{aremark}[Scope of the Aitchison conclusion]
On the full simplex,
\[
\dAc(x,y)=0
\quad\Longleftrightarrow\quad
R_0(x)=R_0(y).
\]
The map $R_0$ is surjective because
$R_0\bigl((1/2,\pi/2)\bigr)=\pi$ for every $\pi\in\Sm{D-1}$.
Therefore
\[
\Phi:\Sm{D}/{\sim_0}\longrightarrow\Sm{D-1},
\qquad
\Phi([x]):=R_0(x),
\]
is a well-defined bijection. By the definition of $\dAc$, it preserves
distances, so the metric quotient is isometric to
$(\Sm{D-1},\dA)$.

Full $\dA$ is a genuine metric, but its ranking is not generally
invariant under sink deletion. Indeed,
\[
\dA(p_s,q)\longrightarrow\infty
\quad\text{as }s\uparrow1,
\qquad
\dA(p_s,r_s)=B.
\]
Thus the unqualified claim that full $\dA$ always ranks $q$ closer
in both views is false. It holds for the Table~\ref{tab:witness}
example because $s<s_{\mathrm A}$ there.
\end{aremark}

\begin{proof}[Proof of Proposition~\ref{prop:entropy}]
We first establish the entropy identity for the sink--content
factorization. Let
\[
r=\bigl(s,(1-s)\pi\bigr),
\qquad
s\in(0,1),\quad \pi\in\Sm{D-1},
\]
and define the binary entropy
\[
H_{\mathrm b}(s):=-s\log s-(1-s)\log(1-s).
\]
Since $\sum_{j=1}^{D-1}\pi_j=1$, direct expansion gives
\begin{align}
H(r)
&=-s\log s
  -\sum_{j=1}^{D-1}(1-s)\pi_j
       \log\!\bigl((1-s)\pi_j\bigr) \notag\\
&=-s\log s-(1-s)\log(1-s)
  -(1-s)\sum_{j=1}^{D-1}\pi_j\log\pi_j \notag\\
&=H_{\mathrm b}(s)+(1-s)H(\pi).
\label{eq:entropy-inversion-chain-rule}
\end{align}

Fix $D\geq3$, and put $m:=D-1\geq2$. Define
\[
\pi^{(p)}
:=\left(\frac1m,\ldots,\frac1m\right)\in\Sm{m},
\qquad
H\bigl(\pi^{(p)}\bigr)=\log m.
\]
For $0<\varepsilon<1/m$, define
\[
\pi^{(q)}_\varepsilon
:=\bigl(1-(m-1)\varepsilon,
        \varepsilon,\ldots,\varepsilon\bigr)
\in\Sm{m}
\]
and set
$\eta_\varepsilon:=H\bigl(\pi^{(q)}_\varepsilon\bigr)$.
The vector $\pi^{(q)}_\varepsilon$ has $m$ strictly positive
coordinates that sum to one, and
\[
\eta_\varepsilon
=-\bigl(1-(m-1)\varepsilon\bigr)
     \log\bigl(1-(m-1)\varepsilon\bigr)
  -(m-1)\varepsilon\log\varepsilon
\longrightarrow0
\]
as $\varepsilon\downarrow0$. Indeed, the first term tends to zero by
continuity of $-x\log x$ at $x=1$, while the second tends to zero by
$x\log x\to0$ as $x\downarrow0$. Because $m\geq2$ and every
coordinate of $\pi^{(q)}_\varepsilon$ lies in $(0,1)$, we also have
$\eta_\varepsilon>0$. We may therefore fix
$\varepsilon\in(0,1/m)$ such that
$0<\eta_\varepsilon<\log2$.

For $s\in(0,1)$, set
\[
p_s:=\bigl(s,(1-s)\pi^{(p)}\bigr),
\qquad
q:=\left(\frac12,\frac12\pi^{(q)}_\varepsilon\right).
\]
Both vectors belong to $\Sm{D}$. By
\eqref{eq:entropy-inversion-chain-rule},
\[
H(p_s)=H_{\mathrm b}(s)+(1-s)\log m\longrightarrow0
\qquad(s\uparrow1),
\]
whereas
\[
H(q)
=H_{\mathrm b}\!\left(\frac12\right)+\frac12\eta_\varepsilon
=\log2+\frac12\eta_\varepsilon>0.
\]
Since $H(q)>0$ is fixed, there exists $s_0\in(0,1)$, sufficiently
close to one, such that
$H(p_{s_0})<H(q)$.
Set $p:=p_{s_0}$. Removing the sink and re-closing gives
$p'=\pi^{(p)}$ and $q'=\pi^{(q)}_\varepsilon$,
and hence
\[
H(p')=\log m\geq\log2
>\eta_\varepsilon=H(q').
\]
This proves the asserted inversion for every $D\geq3$.

We now verify the displayed $D=3$ witness by exact inequalities. Direct
re-closure gives
\[
P'=\left(\frac12,\frac12\right),
\qquad
Q'=\left(\frac{999}{1000},\frac1{1000}\right).
\]
For $P$,
\[
H(P)
=\frac{19}{20}\log\frac{20}{19}+\frac1{20}\log40,
\]
so
\[
20H(P)
=\log\!\left[40\left(\frac{20}{19}\right)^{19}\right].
\]
Since
\[
\left(1+\frac1{19}\right)^{19}<e<3,
\]
we obtain
\[
40\left(\frac{20}{19}\right)^{19}<120<2^{20}.
\]
Taking logarithms and dividing by $20$ yields
$H(P)<\log2$.
Applying \eqref{eq:entropy-inversion-chain-rule} to $Q$ gives
\[
H(Q)
=\log2+\frac12H_{\mathrm b}\!\left(\frac1{1000}\right)
>\log2.
\]
Thus $H(P)<H(Q)$ exactly.

Next,
\[
H_{\mathrm b}'(t)=\log\frac{1-t}{t},
\qquad
H_{\mathrm b}''(t)=-\frac1{t(1-t)}<0
\qquad(0<t<1).
\]
Hence $H_{\mathrm b}$ is strictly concave and has its unique maximum
$\log2$ at $t=1/2$. Since $1/1000\neq1/2$,
\[
H(Q')
=H_{\mathrm b}\!\left(\frac1{1000}\right)
<\log2
=H(P').
\]

Finally, for a binary composition $u=(a,b)$,
\[
\clr_2(u)
=\left(\frac12\log\frac ab,-\frac12\log\frac ab\right).
\]
It follows that
\[
\clr_2(P')=(0,0),
\qquad
\clr_2(Q')
=\left(\frac12\log999,-\frac12\log999\right),
\]
and therefore
\[
\CAc(P)=0,
\qquad
\CAc(Q)=\frac{\log999}{\sqrt2}.
\]
This completes the verification of the witness and the concentration values.
\end{proof}

\begin{aremark}[Sharpness and scope of the entropy inversion]
The restriction $D\geq3$ is necessary. If $D=2$, every sink-dropped
composition equals $(1)$, and therefore $H(p')=H(q')=0$ for all
$p,q\in\Sm{2}$.
Moreover, an entropy inversion necessarily uses unequal sink masses. Indeed,
if $p=\bigl(s,(1-s)\pi_p\bigr)$ and $q=\bigl(s,(1-s)\pi_q\bigr)$,
then \eqref{eq:entropy-inversion-chain-rule} implies
\[
H(p)-H(q)
=(1-s)\bigl(H(p')-H(q')\bigr).
\]
Because $1-s>0$, the entropy differences have the same sign.
Finally, for a fixed normalized content composition $\pi$,
\[
\CAc\bigl(s,(1-s)\pi\bigr)
=\lVert\clr_{D-1}(\pi)\rVert_2,
\qquad 0<s<1.
\]
Thus $\CAc$ is independent of the mass assigned to the fixed
designated sink when $\pi$ is held fixed, and it equals the CLR norm
computed directly from the re-closed content composition. This does not
assert invariance under changing \emph{which} coordinate is designated as the
sink, nor does it identify $\CAc$ with the CLR norm of the full composition.
\end{aremark}

\section{Proofs for Section~\ref{sec:toolkit}}
\label{app:characterization}

\begin{table}[h]
\centering
\small
\caption{\textbf{Invariance matrix.} Perm.\ $=$ permutation invariance; Perturb.\ $=$ invariance to a shared logit bias (Prop.~\ref{prop:invariances}(i)); Temp.\ $=$ global-temperature ranking preservation; Sink $=$ coherent under keep-vs-drop of the sink column; Dom.\ $=$ subcompositional dominance. $\dagger$: $\dA$ is $1/\tau$-equivariant, so rankings are exact; $\rhoA$ is invariant even to per-head temperatures. $\ddagger$: $\dA$ decomposes exactly (Lemma~\ref{lem:pythagoras}); $\dAc,\CAc$ are exactly sink-invariant. ``n/a'': not applicable.}
\label{tab:invariance}
\begin{tabular}{lccccc}
\toprule
 & Perm. & Perturb.\ ($\oplus$) & Temp. & Sink & Dom. \\
\midrule
Euclidean & \cmark & \xmark & \xmark & \xmark & \xmark \\
cosine & \cmark & \xmark & \xmark & \xmark & \xmark \\
JS & \cmark & \xmark & \xmark & \xmark & \xmark \\
Shannon $H$ & \cmark & \xmark & \xmark & \xmark & n/a \\
\midrule
$\dA$ & \cmark & \cmark & \cmark$^\dagger$ & \cmark$^\ddagger$ & \cmark \\
$\rhoA$ (ILR-cosine) & \cmark & \xmark & \cmark$^\dagger$ & n/a & n/a \\
$\dAc$, $\CAc$ & \cmark & \cmark & \cmark$^\dagger$ & \cmark$^\ddagger$ & \cmark \\
\bottomrule
\end{tabular}
\end{table}

\paragraph{Standing notation.}
All logarithms are natural. For every positive vector
$x=(x_0,\ldots,x_{m-1})$, write $\clo(x):=x/\sum_i x_i$. For every integer
$m\geq1$, let $\Sm{m}$ be the open $m$-part simplex, $e_m:=\frac1m\mathbf 1_m$,
\[
  g(x_1,\ldots,x_m):=\Bigl(\prod_{i=1}^m x_i\Bigr)^{1/m},
  \qquad
  \clr_m(x):=\bigl(\log x_i-m^{-1}\textstyle\sum_j\log x_j\bigr)_i,
\]
and throughout $\dA(x,y):=\lVert\clr_m(x)-\clr_m(y)\rVert_2$ for
$x,y\in\Sm{m}$, with no dimension-dependent rescaling. For $D\geq2$ and
$p,q\in\Sm{D}$ with designated sink coordinate $0$,
\[
  p':=\clo(p_1,\ldots,p_{D-1}),
  \qquad
  b(p):=\sqrt{\tfrac{D-1}{D}}\log\tfrac{p_0}{g(p_1,\ldots,p_{D-1})},
\]
\[
  \dAc(p,q):=\dA(p',q'),
  \qquad
  \CAc(p):=\lVert\clr_{D-1}(p')\rVert_2 .
\]
Perturbation, powering, and the Aitchison center are
$x\oplus y:=\clo((x_iy_i)_i)$, $\alpha\odot x:=\clo((x_i^\alpha)_i)$, and
$\cen(x_1,\ldots,x_n):=\clo(\exp\{\frac1n\sum_t\log x_t\})$, all componentwise.
When $x\neq e_m$ and $y\neq e_m$,
$\rhoA(x,y):=\langle\clr_m(x),\clr_m(y)\rangle/(\lVert\clr_m(x)\rVert_2\lVert\clr_m(y)\rVert_2)$,
and for nonzero Euclidean vectors $\cos(x,y):=x^\top y/(\lVert x\rVert_2\lVert y\rVert_2)$.
For $m=1$, set $\Sm{1}=\{(1)\}$ and regard its CLR space and ILR coordinate
vector as zero-dimensional, so its CLR vector, norm, and all distances are zero.

\begin{proof}[Proof of Lemma~\ref{lem:pythagoras}]
For $m\geq1$, let
\[
  \mathcal H_m
  :=\left\{x\in\mathbb R^m:\mathbf1_m^{\top}x=0\right\}.
\]
The map $\clr_m$ takes values in $\mathcal H_m$.  Define
\[
  u_b
  :=\sqrt{\frac{D-1}{D}}
    \left(1,-\frac1{D-1},\ldots,-\frac1{D-1}\right)
  \in\mathcal H_D
\]
and
\[
  W:=\{(0,z):z\in\mathcal H_{D-1}\}\subset\mathcal H_D,
  \qquad
  J:\mathcal H_{D-1}\longrightarrow W,
  \qquad Jz=(0,z).
\]
The map $J$ is a linear isometry.  Moreover,
\[
  \lVert u_b\rVert_2^2
  =\frac{D-1}{D}
    \left(1+\frac{D-1}{(D-1)^2}\right)=1,
\]
and for every $z\in\mathcal H_{D-1}$,
\[
  \langle u_b,Jz\rangle
  =-\frac{1}{\sqrt{D(D-1)}}\sum_{i=1}^{D-1}z_i=0.
\]
Since $\dim W=D-2$ and $\dim\mathcal H_D=D-1$, it follows that
\begin{equation}
  \mathcal H_D
  =\operatorname{span}\{u_b\}\oplus W,
  \qquad \operatorname{span}\{u_b\}\perp W.
  \label{eq:sink-orthogonal-sum}
\end{equation}

Write uniquely
\[
  p=(s,(1-s)\pi),
  \qquad s=p_0\in(0,1),
  \qquad \pi=p'\in\Sm{D-1}.
\]
Because the entries of $u_b$ sum to zero, the common centering term
in $\clr_D(p)$ vanishes in its inner product with $u_b$.  Hence
\begin{align}
  \langle\clr_D(p),u_b\rangle
  &=\sqrt{\frac{D-1}{D}}
    \left(\log p_0-
      \frac1{D-1}\sum_{i=1}^{D-1}\log p_i\right)\notag\\
  &=\sqrt{\frac{D-1}{D}}
    \log\frac{p_0}{g(p_1,\ldots,p_{D-1})}
   =b(p).
  \label{eq:sink-balance-coordinate}
\end{align}

Let $P_W$ denote Euclidean orthogonal projection onto $W$.  For
$x\in\mathcal H_D$,
\[
  (P_Wx)_0=0,
  \qquad
  (P_Wx)_i=x_i-
    \frac1{D-1}\sum_{j=1}^{D-1}x_j,
  \quad 1\leq i\leq D-1.
\]
Applying this to $x=\clr_D(p)$ gives, for $i\geq1$,
\begin{align*}
  (P_W\clr_D(p))_i
  &=\log p_i-
    \frac1{D-1}\sum_{j=1}^{D-1}\log p_j\\
  &=\log\pi_i-
    \frac1{D-1}\sum_{j=1}^{D-1}\log\pi_j
   =(\clr_{D-1}(\pi))_i.
\end{align*}
Therefore
\begin{equation}
  P_W\clr_D(p)=J\clr_{D-1}(p').
  \label{eq:sink-projection}
\end{equation}
Equations \eqref{eq:sink-orthogonal-sum}--\eqref{eq:sink-projection}
yield the exact orthogonal identity
\begin{equation}
  \boxed{\;
  \clr_D(p)=b(p)u_b+J\clr_{D-1}(p').
  \;}
  \label{eq:sink-clr-decomposition}
\end{equation}

Apply \eqref{eq:sink-clr-decomposition} to $p$ and $q$, subtract,
and use the orthogonality in \eqref{eq:sink-orthogonal-sum} and the
isometry of $J$.  Then
\begin{align*}
  \dA(p,q)^2
  &=\lVert\clr_D(p)-\clr_D(q)\rVert_2^2\\
  &=\bigl(b(p)-b(q)\bigr)^2
    +\lVert\clr_{D-1}(p')-\clr_{D-1}(q')\rVert_2^2\\
  &=\bigl(b(p)-b(q)\bigr)^2+\dAc(p,q)^2,
\end{align*}
which proves the Pythagorean identity.

To verify the ILR assertion, let $w_2,\ldots,w_{D-1}$ be any
orthonormal basis of $W$; the list is empty when $D=2$.  Then
$(u_b,w_2,\ldots,w_{D-1})$ is an orthonormal contrast basis of
$\mathcal H_D$ and hence defines an ILR map.  Its first coordinate is
$b(p)$ by \eqref{eq:sink-balance-coordinate}.  The vectors
$J^{-1}w_2,\ldots,J^{-1}w_{D-1}$ form an orthonormal basis of
$\mathcal H_{D-1}$, and \eqref{eq:sink-projection} gives
\[
  \langle\clr_D(p),w_j\rangle
  =\langle\clr_{D-1}(p'),J^{-1}w_j\rangle,
  \qquad 2\leq j\leq D-1.
\]
Thus the remaining coordinates are precisely an ILR coordinate
system for $p'$.  In particular,
\[
  \dAc(p,q)
  =\bigl\lVert P_W(\clr_D(p)-\clr_D(q))\bigr\rVert_2,
  \qquad
  \CAc(p)=\lVert P_W\clr_D(p)\rVert_2.
\]
The first expression depends only on $p'$ and $q'$, and the second
only on $p'$; neither depends on a sink mass.  Equation
\eqref{eq:sink-projection} also makes the projection
claim precise: it is a statement in CLR space under $J$, not a
Euclidean projection of probability vectors.
\end{proof}

\paragraph{Exact evaluation of the Table~\ref{tab:witness} witness.}
For
\[
  p=\left(\frac9{10},\frac2{25},\frac1{50}\right),\qquad
  q=\left(\frac7{10},\frac6{25},\frac3{50}\right),\qquad
  r=\left(\frac9{10},\frac1{50},\frac2{25}\right),
\]
interpreting the decimal entries in Table~\ref{tab:witness} as these exact
rational numbers, one has
\[
  p'=q'=\left(\frac45,\frac15\right),
  \qquad
  r'=\left(\frac15,\frac45\right),
\]
and
\begin{align*}
  b(p)=b(r)
  &=\sqrt{\frac23}\log\frac{45}{2}
    \approx2.542174604636,\\
  b(q)
  &=\sqrt{\frac23}\log\frac{35}{6}
    \approx1.439964055745.
\end{align*}
Consequently, the pure-sink and pure-content identities are
\begin{align*}
  \dAc(p,q)&=0,\\
  \dA(p,q)^2
  &=\frac23\log^2\frac{27}{7}
    \approx1.214868094088,\\
  b(p)-b(r)&=0,\\
  \dA(p,r)^2
  &=\dA(p',r')^2
    =2\log^2 4
    =8\log^2 2
    \approx3.843624111346,
\end{align*}
with
\[
  \dA(p,r)=\sqrt2\log4\approx1.960516286937.
\]
Thus rounded decimal displays must use $\approx$; the exact equalities
are the logarithmic identities above.

\subsection*{Full statement of Proposition~\ref{prop:invariances}}
Let $D\geq2$, let $\ell,\ell'\in\mathbb R^D$ be finite logit rows, and put
$p=\clo(e^\ell)$, $q=\clo(e^{\ell'})$, where exponentials, products, and
powers are componentwise.
\begin{enumerate}
\item[\rm(i)]
\emph{Shared finite logit bias equals perturbation.}
For every finite $u\in\mathbb R^D$, with $c_u=\clo(e^u)$,
\[
  \clo(e^{\ell+u})=p\oplus c_u,
  \qquad
  \clo(e^{\ell'+u})=q\oplus c_u,
\]
and
$\dA(p\oplus c_u,q\oplus c_u)=\dA(p,q)$.
For a fixed designated sink, $\dAc$ has the same shared-perturbation
invariance.  If $D\geq3$, $\CAc$ is not perturbation-invariant in
general; when $D=2$, it is identically zero.  Euclidean distance,
Jensen--Shannon divergence, and cosine similarity likewise fail to
be invariant under every shared perturbation.  The directional
similarity $\rhoA$ is also not perturbation-invariant in general, even
when it is defined both before and after perturbation.
\item[\rm(ii)]
\emph{Temperature equals powering.}
For every finite temperature $\tau>0$,
$\clo(e^{\ell/\tau})=\frac1\tau\odot p$ and
$\clo(e^{\ell'/\tau})=\frac1\tau\odot q$, and hence
\[
  \dA\left(\tfrac1\tau\odot p,\tfrac1\tau\odot q\right)
  =\tfrac1\tau\dA(p,q),
  \qquad
  \dAc\left(\tfrac1\tau\odot p,\tfrac1\tau\odot q\right)
  =\tfrac1\tau\dAc(p,q),
  \qquad
  \CAc\left(\tfrac1\tau\odot p\right)
  =\tfrac1\tau\CAc(p).
\]
Thus a single common temperature applied to every member of a family
preserves all weak and strict rankings among the resulting distances
or concentration values; these are equivariances, not value
invariances.
For a nonempty finite collection $P=(p_1,\ldots,p_n)$, define
$\operatorname{Var}_A(P):=\frac1n\sum_{t}\dA(p_t,\cen(P))^2$. Then
\[
  \operatorname{Var}_A
  \left(\tfrac1\tau\odot p_1,\ldots,\tfrac1\tau\odot p_n\right)
  =\tfrac1{\tau^2}\operatorname{Var}_A(P),
\]
where the center on the left is recomputed after powering.
Whenever $p\neq e_D$ and $q\neq e_D$, so that $\rhoA(p,q)$ is defined,
separate positive temperatures cancel: for every $\tau_h,\tau_{h'}>0$,
\[
  \rhoA\left(\tfrac1{\tau_h}\odot p,\tfrac1{\tau_{h'}}\odot q\right)
  =\rhoA(p,q).
\]
More formally, for $h\in\{1,2\}$ let $n_h\geq1$ and
$p_{ht}\in\Sm{D}$, put $P_h:=(p_{h1},\ldots,p_{hn_h})$ and
$s_h:=\cen(P_h)\neq e_D$, and give every row in head $h$ the same
exponent $\alpha_h>0$.  If
$s_h^{(\alpha_h)}:=\cen(\alpha_h\odot p_{h1},\ldots,\alpha_h\odot p_{hn_h})$,
then
\[
  s_h^{(\alpha_h)}=\alpha_h\odot s_h,
  \qquad
  \rhoA\bigl(s_1^{(\alpha_1)},s_2^{(\alpha_2)}\bigr)
  =\rhoA(s_1,s_2).
\]
Thus one positive temperature per head cancels in $\rhoA$ when the
rows have a common support and the two resulting signatures are
nonuniform.  This conclusion can fail for row-dependent temperatures.
Under unequal head temperatures or row-dependent temperatures, the
separate-scaling cancellation asserted here is specific to $\rhoA$.
No general invariance or ranking-preservation claim is made for
$\dA$, $\dAc$, $\CAc$, or Aitchison variance.
\item[\rm(iii)]
\emph{Sink coherence.}
For a fixed designated sink, write
$p=(s,(1-s)\pi)$ and $q=(t,(1-t)\eta)$ with $\pi,\eta\in\Sm{D-1}$.
Then
$\dAc(p,q)=\dA(\pi,\eta)$ and
$\CAc(p)=\lVert\clr_{D-1}(\pi)\rVert_2$.
These values are independent of $s$ and $t$ and are identical whether
computed in the balance-orthogonal coordinates of the full compositions
or after dropping the sink and re-closing.  Full $\dA$ is not itself
sink-invariant; Lemma~\ref{lem:pythagoras} gives its exact
sink--content decomposition.
\item[\rm(iv)]
\emph{Subcomposition and masking.}
Let $S\subseteq\{0,\ldots,D-1\}$ with $|S|\geq2$, and define
$p^{(S)}=\clo((p_i)_{i\in S})$ and $q^{(S)}=\clo((q_i)_{i\in S})$.
After canonically embedding the $|S|$-part CLR space into
$\mathcal H_D$, subcomposition is an orthogonal projection and
\[
  \dA\bigl(p^{(S)},q^{(S)}\bigr)\leq\dA(p,q).
\]
Coordinate deletion is literal in an $S$-adapted ILR basis; in an
arbitrary fixed ILR basis the same map is generally a non-diagonal
orthogonal projection.
Outside the standing finite-logit hypothesis, let $\widetilde p$ and
$\widetilde q$ be hard-masked probability rows in the closed simplex
with positive supports $T_{\widetilde p}$ and $T_{\widetilde q}$.
If $S=T_{\widetilde p}\cap T_{\widetilde q}$ has at least two parts,
$\dA(\widetilde p^{(S)},\widetilde q^{(S)})$ (of the re-closed
restrictions) is a well-defined within-common-support comparison.  If
the two supports differ, this comparison discards the support mismatch.
Since structural-zero rows do not lie in the open simplex, their full
Aitchison distance is undefined, and no contraction inequality relative
to that nonexistent full distance is asserted.  If positive pre-mask
rows in $\Sm{D}$ are available and the same mask retains $S$, the
preceding contraction inequality does apply to those pre-mask rows and
their $S$-subcompositions.
\item[\rm(v)]
\emph{Dominance.}
Part \textup{(iv)} shows that $\dA$ is subcompositionally dominant.
For every $D\geq3$, Euclidean distance, Jensen--Shannon divergence
$\JS$ (not its square root), and cosine dissimilarity $1-\cos$ are
not subcompositionally dominant.  For $D=2$, there is no nontrivial
proper subcomposition containing at least two parts.
\end{enumerate}

\begin{proof}[Proof of Proposition~\ref{prop:invariances}]
For \textup{(i)}, closure is unchanged by multiplication by a common
positive scalar.  Therefore
\begin{align*}
  \clo(e^{\ell+u})
  &=\clo\bigl((e^{\ell_i}e^{u_i})_{i=0}^{D-1}\bigr)\\
  &=\clo\bigl((p_i(c_u)_i)_{i=0}^{D-1}\bigr)
   =p\oplus c_u,
\end{align*}
and similarly for $q$.  Since
$\clr(a\oplus b)=\clr(a)+\clr(b)$,
we obtain
\begin{align*}
  \dA(p\oplus c_u,q\oplus c_u)
  &=\bigl\lVert
      \clr(p)+\clr(c_u)-\clr(q)-\clr(c_u)
    \bigr\rVert_2\\
  &=\dA(p,q).
\end{align*}

For the content distance, let $c_u'$ be the sink-dropped and re-closed
composition.  Directly from the definition of perturbation,
\[
  (p\oplus c_u)'=p'\oplus c_u',
  \qquad
  (q\oplus c_u)'=q'\oplus c_u'.
\]
Applying the same CLR perturbation identity in dimension $D-1$
(including the stated zero-dimensional convention when $D=2$) gives
$\dAc(p\oplus c_u,q\oplus c_u)=\dAc(p,q)$.
In contrast,
\[
  \CAc(p\oplus c_u)
  =\bigl\lVert\clr_{D-1}(p')+\clr_{D-1}(c_u')\bigr\rVert_2.
\]
For $D\geq3$, this equals $\CAc(p)$ for every $p$ if and only if
$\clr_{D-1}(c_u')=0$, equivalently, the content coordinates of $u$
are all equal.  Thus the analogous invariance under arbitrary shared
perturbations is false for $\CAc$.  Indeed, let
$p=e_D$ and $c=\clo(1,2,1,\ldots,1)$,
where coordinate $0$ is the sink and the factor $2$ is on a content
coordinate.  Then
\[
  \CAc(p)=0,
  \qquad
  (p\oplus c)'=\clo(2,1,\ldots,1),
  \qquad
  \CAc(p\oplus c)>0.
\]
At $D=3$, the last value is exactly $(\log2)/\sqrt2$.  When $D=2$,
the content simplex has one part and $\CAc\equiv0$.

To prove failure of the corresponding classical invariances, fix
distinct $p,q\in\Sm{D}$, choose a coordinate $j$, let
$u_t=t\mathbf e_j$, and put $c_t=\clo(e^{u_t})$.  Since all entries
of $p$ and $q$ are positive,
\[
  c_t\oplus p\longrightarrow\mathbf e_j,
  \qquad
  c_t\oplus q\longrightarrow\mathbf e_j
  \qquad(t\to\infty).
\]
Using the standard continuous extension of $\JS$ to the closed
simplex (with $0\log0=0$), and continuity of Euclidean distance and
cosine similarity at probability vectors, gives
\begin{align*}
  \lVert c_t\oplus p-c_t\oplus q\rVert_2&\longrightarrow0,\\
  \JS(c_t\oplus p,c_t\oplus q)&\longrightarrow0,\\
  \cos(c_t\oplus p,c_t\oplus q)&\longrightarrow1.
\end{align*}
For $p\neq q$, the original Euclidean distance and $\JS$ divergence
are strictly positive, whereas the original cosine similarity is
strictly less than one.  Consequently none is invariant under every
shared perturbation.

The same conclusion holds for $\rhoA$.  Fix
$0\neq v\in\mathcal H_D$ and set
\[
  p=\clo(e^v),
  \qquad q=\clo(e^{-v}),
  \qquad c=\clo(e^{2v}).
\]
Then $\clr(p)=v$, $\clr(q)=-v$, and $\clr(c)=2v$, so all relevant
directional similarities are defined and
\[
  \rhoA(p,q)=-1,
  \qquad
  \rhoA(p\oplus c,q\oplus c)
  =\rhoA\bigl(\clo(e^{3v}),\clo(e^v)\bigr)=1.
\]

For \textup{(ii)}, set $\alpha=1/\tau>0$.  Then
$\clo(e^{\ell/\tau})=\clo((e^{\ell_i})^\alpha)=\alpha\odot p$.
The elementary identity
\begin{equation}
  \clr(\alpha\odot p)=\alpha\clr(p)
  \label{eq:clr-powering}
\end{equation}
implies
$\dA(\alpha\odot p,\alpha\odot q)=\alpha\dA(p,q)$.
Moreover, dropping and re-closing commutes with powering:
$(\alpha\odot p)'=\alpha\odot p'$.
Applying \eqref{eq:clr-powering} in the $(D-1)$-part content simplex
proves the displayed scaling identities for $\dAc$ and $\CAc$.

For the variance assertion, put $x_t=\clr(p_t)$ and
$\bar x=n^{-1}\sum_{t=1}^n x_t$.  From the definition of the
Aitchison center,
$\clr(\cen(P))=\bar x$.
After common powering by $\alpha$, the CLR vectors and their
recomputed center are $\alpha x_t$ and $\alpha\bar x$.  Hence
\[
  \frac1n\sum_{t=1}^n
  \lVert\alpha x_t-\alpha\bar x\rVert_2^2
  =\alpha^2\frac1n\sum_{t=1}^n
  \lVert x_t-\bar x\rVert_2^2,
\]
which proves the factor $1/\tau^2$.

If $p,q\neq e_D$ and $\alpha,\beta>0$, then
\[
  \frac{\langle\alpha\clr(p),\beta\clr(q)\rangle}
       {\lVert\alpha\clr(p)\rVert_2
        \lVert\beta\clr(q)\rVert_2}
  =\rhoA(p,q),
\]
which proves invariance of $\rhoA$ to separate positive temperatures.

For the aggregation qualification, let
$s_h=\cen(p_{h1},\ldots,p_{hn_h})$ and let
$\alpha_h=1/\tau_h>0$ be common to all rows in head $h$.  Applying
\eqref{eq:clr-powering} and the center identity above gives
\[
  \cen(\alpha_h\odot p_{h1},\ldots,
       \alpha_h\odot p_{hn_h})
  =\alpha_h\odot s_h.
\]
Thus separate positive head temperatures cancel in $\rhoA$ whenever
the resulting signatures are nonuniform.  With row-dependent
coefficients $\alpha_{ht}$, the new center has CLR vector
$n_h^{-1}\sum_t\alpha_{ht}\clr(p_{ht})$, which need not be a scalar
multiple of $\clr(s_h)$.  An explicit failure is available in every
$D\geq2$.  Fix $0\neq v\in\mathcal H_D$.  Let the first head contain
\[
  p_{11}=\clo(e^{2v}),
  \qquad
  p_{12}=\clo(e^{-v}),
\]
and let the second head consist of $p_{21}=\clo(e^v)$.  Before
row-dependent powering, their signatures satisfy
\[
  \clr(s_1)=\frac12v,
  \qquad
  \clr(s_2)=v,
  \qquad
  \rhoA(s_1,s_2)=1.
\]
Now use exponents $\alpha_{11}=1/4$ and $\alpha_{12}=1$ in the first
head, leaving the second head unchanged.  The new first-head
signature $\widetilde s_1$ has
\[
  \clr(\widetilde s_1)
  =\frac12\left(\frac14(2v)-v\right)
  =-\frac14v,
  \qquad
  \rhoA(\widetilde s_1,s_2)=-1.
\]
Hence row-dependent positive temperatures can change $\rhoA$.

For \textup{(iii)}, closure of the content block gives $p'=\pi$ and
$q'=\eta$.  The two displayed identities follow immediately from
the definitions
$\dAc(p,q):=\dA(p',q')$ and $\CAc(p):=\lVert\clr_{D-1}(p')\rVert_2$.
Their equality with the full balance-orthogonal calculation is
\eqref{eq:sink-projection}.  The final qualification follows from the
Pythagorean identity: the balance term generally changes when sink
masses change.

For \textup{(iv)}, define
\[
  W_S:=\left\{w\in\mathbb R^D:
    w_i=0\text{ for }i\notin S,
    \ \sum_{i\in S}w_i=0\right\}\subset\mathcal H_D,
\]
and let $I_S:\mathcal H_{|S|}\to W_S$ be zero-extension, using the
inherited coordinate order on $S$.  This is a linear isometry.  If
$P_{W_S}$ is orthogonal projection onto $W_S$, then for every
$x\in\mathcal H_D$,
\[
  (P_{W_S}x)_i=
  \begin{cases}
    x_i-|S|^{-1}\sum_{j\in S}x_j,&i\in S,\\
    0,&i\notin S.
  \end{cases}
\]
Applying this to $x=\clr_D(p)$ shows that
\begin{equation}
  I_S\clr_{|S|}(p^{(S)})=P_{W_S}\clr_D(p),
  \qquad
  I_S\clr_{|S|}(q^{(S)})=P_{W_S}\clr_D(q).
  \label{eq:subcomposition-projection}
\end{equation}
Subtracting the identities in
\eqref{eq:subcomposition-projection}, using the isometry of $I_S$,
and using contractivity of an orthogonal projection gives
\begin{align*}
  \dA(p^{(S)},q^{(S)})
  &=\bigl\lVert
      P_{W_S}(\clr_D(p)-\clr_D(q))
    \bigr\rVert_2\\
  &\leq\lVert\clr_D(p)-\clr_D(q)\rVert_2
   =\dA(p,q).
\end{align*}
This proves both the projection identity and dominance.  The masking
qualifications in the statement follow because CLR coordinates exist
only for strictly positive parts.

For \textup{(v)}, first consider $D=3$ and use the
Table~\ref{tab:witness} compositions
\[
  p=\left(\frac9{10},\frac2{25},\frac1{50}\right),
  \qquad
  r=\left(\frac9{10},\frac1{50},\frac2{25}\right),
  \qquad
  S=\{1,2\}.
\]
Then
\[
  p^{(S)}=\left(\frac45,\frac15\right),
  \qquad
  r^{(S)}=\left(\frac15,\frac45\right).
\]
For Euclidean distance,
\[
  \lVert p-r\rVert_2=\frac{3\sqrt2}{50}
  <\frac{3\sqrt2}{5}
  =\lVert p^{(S)}-r^{(S)}\rVert_2.
\]
For cosine dissimilarity, direct calculation gives
\[
  \lVert p\rVert_2^2=\lVert r\rVert_2^2=\frac{2042}{2500},
  \qquad
  p^\top r=\frac{2033}{2500},
\]
and
\[
  \lVert p^{(S)}\rVert_2^2
  =\lVert r^{(S)}\rVert_2^2=\frac{17}{25},
  \qquad
  (p^{(S)})^\top r^{(S)}=\frac8{25}.
\]
Consequently,
\[
  1-\cos(p,r)=\frac9{2042}
  <\frac9{17}
  =1-\cos(p^{(S)},r^{(S)}).
\]

For $a,b$ in a common simplex, let
\[
  \JS(a,b)
  :=\frac12\KL\left(a\middle\|\frac{a+b}{2}\right)
   +\frac12\KL\left(b\middle\|\frac{a+b}{2}\right).
\]
Put
\[
  \pi=\left(\frac45,\frac15\right),
  \qquad
  \widetilde\pi=\left(\frac15,\frac45\right),
  \qquad
  m_c=\left(\frac12,\frac12\right),
\]
and let
\[
  J_{\pi}:=\frac45\log\frac85+\frac15\log\frac25>0.
\]
By symmetry,
\[
  \KL(\pi\|m_c)
  =\KL(\widetilde\pi\|m_c)=J_{\pi},
  \qquad
  \JS(\pi,\widetilde\pi)=J_{\pi}.
\]
Here
\[
  p=\left(\frac9{10},\frac1{10}\pi\right),
  \qquad
  r=\left(\frac9{10},\frac1{10}\widetilde\pi\right),
\]
and their midpoint is $(9/10,(1/10)m_c)$.  Therefore
\[
  \KL\bigl(p\|(p+r)/2\bigr)
  =\KL\bigl(r\|(p+r)/2\bigr)
  =\frac1{10}J_{\pi},
\]
and hence
\[
  \JS(p,r)=\frac1{10}J_{\pi}
  <J_{\pi}
  =\JS(p^{(S)},r^{(S)}).
\]

It remains to extend the counterexample to every $D>3$.  Put
$k=D-3$, choose $\varepsilon\in(0,1/k)$, set
$a=1-k\varepsilon\in(0,1)$, and append $k$ equal parts:
\[
  \bar p=(ap_0,ap_1,ap_2,
           \underbrace{\varepsilon,\ldots,\varepsilon}_{k}),
  \qquad
  \bar r=(ar_0,ar_1,ar_2,
           \underbrace{\varepsilon,\ldots,\varepsilon}_{k}).
\]
These are strictly positive $D$-part compositions, and their
subcompositions on $S=\{1,2\}$ remain $p^{(S)}$ and $r^{(S)}$.
The appended parts agree, so
\[
  \lVert\bar p-\bar r\rVert_2
  =a\lVert p-r\rVert_2
  <\lVert p^{(S)}-r^{(S)}\rVert_2
\]
and, because the common appended parts contribute zero to both KL
terms while the factor $a$ cancels inside each logarithm,
\begin{align*}
  \KL\left(
    \bar p\middle\|\frac{\bar p+\bar r}{2}\right)
  &=a\KL\left(p\middle\|\frac{p+r}{2}\right),\\
  \KL\left(
    \bar r\middle\|\frac{\bar p+\bar r}{2}\right)
  &=a\KL\left(r\middle\|\frac{p+r}{2}\right).
\end{align*}
Consequently,
\[
  \JS(\bar p,\bar r)=a\JS(p,r)
  <\JS(p^{(S)},r^{(S)}).
\]
Furthermore,
\begin{align*}
  1-\cos(\bar p,\bar r)
  &=\frac{a^2(\lVert p\rVert_2^2-p^\top r)}
          {a^2\lVert p\rVert_2^2+k\varepsilon^2}\\
  &<\frac{\lVert p\rVert_2^2-p^\top r}
           {\lVert p\rVert_2^2}
   =\frac9{2042}
   <\frac9{17}
   =1-\cos(p^{(S)},r^{(S)}).
\end{align*}
The first inequality is strict because $k\varepsilon^2>0$.
Thus, for every $D\geq3$, each classical dissimilarity can increase
strictly after taking a subcomposition, so none is subcompositionally
dominant.
\end{proof}

\begin{aremark}[Scope of the invariance claims]
\label{rem:invariance-scope}
A hard mask uses an offset $-\infty$ and is not an interior
perturbation, so it lies outside
Proposition~\ref{prop:invariances}\textup{(i)} on the full simplex.
One must first restrict and re-close on a common positive support.
Within that reduced simplex, finite shared offsets are again governed
by part~\textup{(i)}; if positive pre-mask rows and a common retained
support are available, part~\textup{(iv)} supplies contraction relative
to their full distance.  With unequal structural-zero supports, the
common-support comparison discards the mismatch and no full Aitchison
distance between the masked rows exists.  Sink coherence always refers
to a fixed designated sink (or to a simultaneous relabeling of both the
composition and the sink designation).  In particular, full $\dA$ is
decomposable but not sink-invariant; $\dAc$ is perturbation-invariant,
whereas $\CAc$ is not when $D\geq3$ (and is identically zero when
$D=2$); and $\dAc$ and $\CAc$ are temperature-equivariant rather than
value-invariant.
\end{aremark}

\subsection*{Full statement and proof of Theorem~\ref{thm:characterization}}

Here $\mathcal S^m$ is the open $m$-part simplex and
$\mathfrak S_m$ the symmetric group. Under the operations above, the map
$T_m:=\clr_m:(\Sm{m},\oplus,\odot)\to(\mathcal H_m,+,\cdot)$
is a linear isomorphism with inverse $T_m^{-1}(x)=\clo(\exp x)$; for
$\sigma\in\mathfrak S_m$ we use the action $(\sigma p)_i=p_{\sigma^{-1}(i)}$.
The full statement of Theorem~\ref{thm:characterization} is: if
$\delta_D:\Sm{D}\times\Sm{D}\to[0,\infty)$ satisfies \textup{(A1)}
$\delta_D$ is a metric; \textup{(A2)}
$\delta_D(c\oplus p,c\oplus q)=\delta_D(p,q)$ for all $c$;
\textup{(A3)} $\delta_D(\alpha\odot p,\alpha\odot q)=|\alpha|\delta_D(p,q)$
for all $\alpha\in\mathbb R$; \textup{(A4)} with $N(v):=\delta_D(v,e_D)$,
$N(v\oplus w)^2+N(v\ominus w)^2=2N(v)^2+2N(w)^2$ for all $v,w$; and
\textup{(A5)} $\delta_D(\sigma p,\sigma q)=\delta_D(p,q)$ for every
$\sigma\in\mathfrak S_D$; then there is a unique constant $\kappa_D>0$
with $\delta_D=\kappa_D d_{A}$ on $\Sm{D}$, and conversely every
$\kappa_D\dA$ satisfies \textup{(A1)}--\textup{(A5)}.

\begin{proof}[Proof of Theorem~\ref{thm:characterization}]
We give the complete argument.

\smallskip
\noindent
\emph{Reduction to a norm.}
Because $T=\clr_D$ is bijective, the identities
\[
 \clr_D(p\oplus q)
 =\clr_D(p)+\clr_D(q),
 \qquad
 \clr_D(\alpha\odot p)
 =\alpha\clr_D(p)
\]
transport the usual vector-space laws on $\mathcal H_D$ to
$(\Sm{D},\oplus,\odot)$, whose zero is $e_D$.  Applying
\textup{(A2)} with $c=\ominus q$ gives
\begin{equation}
 \delta_D(p,q)
 =\delta_D(p\ominus q,e_D)
 =N(p\ominus q).
 \label{eq:aitchison-translation-reduction}
\end{equation}
By \textup{(A1)}, $N(v)\ge0$ and $N(v)=0$ if and only if $v=e_D$.
Since $\alpha\odot e_D=e_D$, \textup{(A3)} yields
$N(\alpha\odot v)=|\alpha|N(v)$.
Finally, the triangle inequality in \textup{(A1)} and perturbation
invariance give
\[
\begin{aligned}
 N(v\oplus w)
 &=\delta_D(v\oplus w,e_D)\\
 &\le \delta_D(v\oplus w,w)+\delta_D(w,e_D)\\
 &=\delta_D(v,e_D)+\delta_D(w,e_D)
 =N(v)+N(w),
\end{aligned}
\]
where the third line follows by perturbing the first distance by
$\ominus w$.  Hence $N$ is a norm on the Aitchison vector space.

\smallskip
\noindent
\emph{The parallelogram identity.}
Assumption \textup{(A4)} is exactly the parallelogram identity for $N$.
The real Jordan--von Neumann theorem therefore implies that
\begin{equation}
 B_\delta(v,w)
 :=\frac14\left\{N(v\oplus w)^2-N(v\ominus w)^2\right\}
 \label{eq:aitchison-polarization}
\end{equation}
is an inner product and that $N(v)^2=B_\delta(v,v)$.

\smallskip
\noindent
\emph{Transport to CLR coordinates.}
Let $T:=\clr_D$.  Transporting $B_\delta$ through the
linear isomorphism $T:\Sm{D}\to\mathcal H_D$ gives an inner product
on $\mathcal H_D$.  Because $\mathcal H_D$ is finite-dimensional, there
is a unique Euclidean-self-adjoint positive-definite operator
$M:\mathcal H_D\to\mathcal H_D$ such that
\begin{equation}
 B_\delta(v,w)
 =\left\langle T(v),M T(w)\right\rangle_2
 \qquad(v,w\in\Sm{D}).
 \label{eq:aitchison-M-representation}
\end{equation}

\smallskip
\noindent
\emph{Permutation invariance forces isotropy.}
Let $P_\sigma$ denote the permutation matrix corresponding to
$\sigma\in\mathfrak S_D$, and put
$Q_\sigma:=P_\sigma|_{\mathcal H_D}$.  The space $\mathcal H_D$ is
invariant under $P_\sigma$, and $Q_\sigma$ is orthogonal.  Coordinate
permutations are Aitchison-linear: they commute with $\oplus$, $\ominus$,
and $\odot$, and $T(\sigma v)=Q_\sigma T(v)$.  Since
$\sigma e_D=e_D$, assumption
\textup{(A5)} gives $N(\sigma v)=N(v)$.  Polarization in
\eqref{eq:aitchison-polarization} consequently gives
$B_\delta(\sigma v,\sigma w)=B_\delta(v,w)$, and hence
\[
 Q_\sigma^* M Q_\sigma=M.
\]
As $Q_\sigma$ is orthogonal, $M$ commutes with every $Q_\sigma$.

For completeness, the coordinate-permutation representation on
$\mathcal H_D$ is irreducible over $\mathbb R$.  Indeed, let
$U$ be a nonzero linear subspace of $\mathcal H_D$ that is invariant
under every coordinate permutation, and choose $0\ne x\in U$.  Not all
coordinates of $x$ can be equal, because $\mathbf{1}_D^\top x=0$; hence
$x_i\ne x_j$ for some $i\ne j$.  If $(ij)$ is the corresponding
transposition and
$\mathbf e_i$ denotes the $i$th standard basis vector of $\mathbb R^D$,
then
\[
 x-P_{(ij)}x=(x_i-x_j)(\mathbf e_i-\mathbf e_j)\in U.
\]
Thus $\mathbf e_i-\mathbf e_j\in U$.  Permuting coordinates shows that
$\mathbf e_k-\mathbf e_\ell\in U$ for every $k\ne\ell$, and these
vectors span $\mathcal H_D$.  Therefore $U=\mathcal H_D$.

By the finite-dimensional spectral theorem, $M$ has a real eigenvalue
$\lambda$ and a nonzero eigenspace $E_\lambda$.  Since $M$ commutes with
all $Q_\sigma$, the space $E_\lambda$ is invariant under
every coordinate permutation.  Irreducibility forces
$E_\lambda=\mathcal H_D$, so
\begin{equation}
 M=\lambda I_{\mathcal H_D}.
 \label{eq:aitchison-M-scalar}
\end{equation}
Positive definiteness gives $\lambda>0$.

\smallskip
\noindent
\emph{Conclusion and converse.}
Equations \eqref{eq:aitchison-translation-reduction},
\eqref{eq:aitchison-M-representation}, and
\eqref{eq:aitchison-M-scalar} yield
\[
\begin{aligned}
 \delta_D(p,q)^2
 &=N(p\ominus q)^2\\
 &=\lambda\left\|T(p\ominus q)\right\|_2^2\\
 &=\lambda\left\|T(p)-T(q)\right\|_2^2
 =\lambda d_{A}(p,q)^2.
\end{aligned}
\]
Thus $\delta_D=\kappa_D\dA$ with
$\kappa_D:=\sqrt\lambda>0$.  Since $D\ge2$, $\dA$ is nonzero for
some pair, so $\kappa_D$ is unique.

Conversely, fix $\kappa_D>0$ and define
$\delta_D(p,q):=\kappa_D\|T(p)-T(q)\|_2$.  Since $T$ is injective and
Euclidean distance is a metric, \textup{(A1)} holds.  The CLR identities
give
\[
\begin{aligned}
 \delta_D(c\oplus p,c\oplus q)
 &=\kappa_D\|T(c)+T(p)-T(c)-T(q)\|_2
 =\delta_D(p,q),\\
 \delta_D(\alpha\odot p,\alpha\odot q)
 &=\kappa_D\|\alpha(T(p)-T(q))\|_2
 =|\alpha|\,\delta_D(p,q),
\end{aligned}
\]
which prove \textup{(A2)} and \textup{(A3)}.  Moreover,
$N(v)=\kappa_D\|T(v)\|_2$, so the Euclidean parallelogram identity gives
\[
\begin{aligned}
 &N(v\oplus w)^2+N(v\ominus w)^2\\
 &\quad=\kappa_D^2\bigl(\|T(v)+T(w)\|_2^2
                         +\|T(v)-T(w)\|_2^2\bigr)\\
 &\quad=2N(v)^2+2N(w)^2,
\end{aligned}
\]
which is \textup{(A4)}.  Finally, because every permutation matrix is
orthogonal,
\[
 \delta_D(\sigma p,\sigma q)
 =\kappa_D\|Q_\sigma(T(p)-T(q))\|_2
 =\delta_D(p,q),
\]
which is \textup{(A5)}.
\end{proof}

\begin{corollary}[Consequences under a dimension-compatible calibration]
\label{cor:consequences}
Assume $D\ge3$ and let $\delta_D$ satisfy the hypotheses of
Theorem~\ref{thm:characterization}.  Write
$\delta_D=\kappa_Dd_{A}$ and, for each $2\le m<D$, define the
compatible common-calibration distance
$\delta_m^{[\kappa_D]}(r,s):=\kappa_D\dA(r,s)$ for $r,s\in\Sm{m}$
(with $\dA$ the $m$-part CLR distance).  Also put
$\delta_D^{[\kappa_D]}:=\delta_D$.  Then:
\begin{enumerate}
\item[\textup{(i)}] \emph{Scale invariance.}
For $x,y\in(0,\infty)^D$, define
$\bar\delta_D(x,y):=\delta_D(\clo(x),\clo(y))$.
For all $a,b>0$,
$\bar\delta_D(ax,by)=\bar\delta_D(x,y)$.
Thus $\bar\delta_D$ is a scale-invariant pseudometric on the positive
cone; it is not a metric there, since proportional vectors have distance
zero.
\item[\textup{(ii)}] \emph{Subcompositional dominance.}
If $S\subseteq\{1,\ldots,D\}$ has $m:=|S|\ge2$, let
$r^{(S)}:=\clo((r_i)_{i\in S})$ for $r\in\Sm{D}$,
with the coordinates in their inherited order.  Then
\[
 \delta_m^{[\kappa_D]}(p^{(S)},q^{(S)})
 \le \delta_D(p,q).
\]
\item[\textup{(iii)}] \emph{Exact sink decomposition.}
Relabel the $D$ coordinates by $0,1,\ldots,D-1$, with $0$ the sink, and
for any $r\in\Sm{D}$ put $r':=\clo((r_1,\ldots,r_{D-1}))$.  Define
$b_\delta(r):=\kappa_D\,b(r)$ (with $b$ the sink balance) and the content
pseudometric
$\delta_D^\perp(p,q):=\delta_{D-1}^{[\kappa_D]}(p',q')=\kappa_D\dA(p',q')$.
Then
\begin{equation}
 \delta_D(p,q)^2
 =\delta_D^\perp(p,q)^2
  +\bigl(b_\delta(p)-b_\delta(q)\bigr)^2.
 \label{eq:aitchison-scaled-sink-decomposition}
\end{equation}
More precisely, with
$W_c:=\{x\in\mathcal H_D:x_0=0\}$
and $P_{W_c}$ Euclidean orthogonal projection onto $W_c$,
\[
 \delta_D^\perp(p,q)
 =\kappa_D\left\|P_{W_c}\bigl(
 \clr_D(p)-\clr_D(q)\bigr)\right\|_2
 =\kappa_D\dA(p',q').
\]
Thus the orthogonal projection recenters the kept CLR coordinates and is
not mere deletion of the sink coordinate.  Moreover, $\delta_D^\perp$ and
the content log-ratio dispersion
$\mathcal K_\delta(r):=\delta_{D-1}^{[\kappa_D]}(r',e_{D-1})$
depend only on content log-ratios and not on the sink mass.  The full
distance $\delta_D$ need not be sink-invariant, because it also contains
the balance term in \eqref{eq:aitchison-scaled-sink-decomposition}.
\item[\textup{(iv)}] \emph{Shared coordinate-bias invariance and temperature
equivariance.}
Let $\operatorname{sm}_D(\ell):=\clo(\exp\ell)$.  For all
$\ell,\ell',u\in\mathbb R^D$ and $\tau>0$,
\[
\begin{aligned}
 \delta_D\bigl(\operatorname{sm}_D(\ell+u),
                \operatorname{sm}_D(\ell'+u)\bigr)
 &=\delta_D\bigl(\operatorname{sm}_D(\ell),
                  \operatorname{sm}_D(\ell')\bigr),\\
 \delta_D\bigl(\operatorname{sm}_D(\ell/\tau),
                \operatorname{sm}_D(\ell'/\tau)\bigr)
 &=\tau^{-1}\delta_D\bigl(\operatorname{sm}_D(\ell),
                            \operatorname{sm}_D(\ell')\bigr).
\end{aligned}
\]
Consequently, one common temperature applied to every row preserves all
pairwise-distance rankings and ties.
For $n\ge1$ and rows $p^{(1)},\ldots,p^{(n)}\in\Sm{D}$, with
$\cen$ the Aitchison center and
\[
 V_\delta(p^{(1)},\ldots,p^{(n)})
 :=\frac1n\sum_{j=1}^n
 \delta_D\!\left(p^{(j)},
 \cen(p^{(1)},\ldots,p^{(n)})\right)^2,
\]
one has, writing $c_u:=\operatorname{sm}_D(u)$,
\[
 V_\delta(c_u\oplus p^{(1)},\ldots,c_u\oplus p^{(n)})
 =V_\delta(p^{(1)},\ldots,p^{(n)}),
 \qquad
 V_\delta\!\left(\tfrac1\tau\odot p^{(1)},\ldots,
                  \tfrac1\tau\odot p^{(n)}\right)
 =\tau^{-2}V_\delta(p^{(1)},\ldots,p^{(n)}).
\]
Finally, put $N_\delta(r):=\delta_D(r,e_D)$ and, for
$v,w\in\Sm{D}$, define
\[
 B_\delta(v,w)
 :=\frac14\left\{N_\delta(v\oplus w)^2
                 -N_\delta(v\ominus w)^2\right\},
 \qquad
 \rho_\delta(v,w)
 :=\frac{B_\delta(v,w)}{N_\delta(v)N_\delta(w)}
 \quad(v,w\ne e_D).
\]
Then
\[
 \rho_\delta(v,w)
 =\frac{\langle\clr_D(v),\clr_D(w)\rangle_2}
        {\|\clr_D(v)\|_2\|\clr_D(w)\|_2},
\]
and, for all $\alpha,\beta>0$,
$\rho_\delta(\alpha\odot v,\beta\odot w)=\rho_\delta(v,w)$:
the cosine is invariant under separate positive powerings, equivalently
arbitrary separate temperatures $\tau_v,\tau_w>0$.  A shared nonconstant
coordinate-bias vector generally does not preserve the cosine.
\end{enumerate}
\end{corollary}

\begin{proof}
Part~\textup{(i)} follows from
$\clo(ax)=\clo(x)$ and $\clo(by)=\clo(y)$.

For part~\textup{(ii)}, put
\[
 a_i:=\log\frac{p_i}{q_i},
 \qquad
 \bar a_D:=\frac1D\sum_{i=1}^D a_i,
 \qquad
 \bar a_S:=\frac1m\sum_{i\in S}a_i,
\]
so that
$z:=\clr_D(p)-\clr_D(q)$ satisfies $z_i=a_i-\bar a_D$.
Also define
$W_S:=\{x\in\mathcal H_D:x_i=0\text{ for }i\notin S\}$.
The orthogonal projection of $z$ onto $W_S$ is
\[
 (P_{W_S}z)_i=
 \begin{cases}
  a_i-\bar a_S,&i\in S,\\[4pt]
  0,&i\notin S.
 \end{cases}
\]
Indeed, the displayed vector belongs to $W_S$, and, for every $y\in W_S$,
\[
 \left\langle z-P_{W_S}z,y\right\rangle_2
 =(\bar a_S-\bar a_D)\sum_{i\in S}y_i=0.
\]
Its restriction to the coordinates in $S$ is precisely
$\clr_m(p^{(S)})-\clr_m(q^{(S)})$, and it vanishes on $S^c$.
Therefore
\[
 \dA(p^{(S)},q^{(S)})
 =\|P_{W_S}z\|_2
 \le\|z\|_2
 =\dA(p,q).
\]
Multiplication by the common positive constant $\kappa_D$ proves
part~\textup{(ii)}.

For part~\textup{(iii)}, after the stated relabeling define
\[
 u_b:=\sqrt{\frac{D-1}{D}}
 \left(1,-\frac1{D-1},\ldots,-\frac1{D-1}\right),
 \qquad
 W_c:=\{x\in\mathcal H_D:x_0=0\}.
\]
Then $\|u_b\|_2=1$, $u_b\perp W_c$, and
$\mathcal H_D=\operatorname{span}\{u_b\}\mathbin{\overset{\perp}{\oplus}}W_c$.
Direct calculation gives
$\left\langle\clr_D(r),u_b\right\rangle_2=b(r)$
and
\[
 P_{W_c}\clr_D(r)
 =\bigl(0,\clr_{D-1}(r')_1,\ldots,
          \clr_{D-1}(r')_{D-1}\bigr).
\]
Applying Pythagoras to
$\clr_D(p)-\clr_D(q)$ therefore yields
\[
 \dA(p,q)^2
 =\dA(p',q')^2
  +\bigl(b(p)-b(q)\bigr)^2.
\]
Multiplying by $\kappa_D^2$ and using the definitions of
$\delta_D^\perp$ and $b_\delta$ proves
\eqref{eq:aitchison-scaled-sink-decomposition}.  The displayed formula
for $P_{W_c}\clr_D(r)$ proves the remaining sink claims.
In particular, the zero-padded vector on the right is the orthogonal
projection; simple deletion of the zeroth CLR coordinate would not give
this vector without the displayed recentering.

For part~\textup{(iv)}, coordinatewise calculation gives
\[
 \operatorname{sm}_D(\ell+u)
 =\operatorname{sm}_D(u)\oplus\operatorname{sm}_D(\ell),
 \qquad
 \operatorname{sm}_D(\ell/\tau)
 =\frac1\tau\odot\operatorname{sm}_D(\ell).
\]
The two distance identities follow from \textup{(A2)} and
\textup{(A3)}.  Moreover,
\[
 \clr_D\!\left(
 \cen(p^{(1)},\ldots,p^{(n)})\right)
 =\frac1n\sum_{j=1}^n\clr_D(p^{(j)}).
\]
Consequently, for every $c\in\Sm{D}$ and $\alpha\in\mathbb R$,
\[
\begin{aligned}
 &\cen(c\oplus p^{(1)},\ldots,c\oplus p^{(n)})
 =c\oplus\cen(p^{(1)},\ldots,p^{(n)}),\\
 &\cen(\alpha\odot p^{(1)},\ldots,
                     \alpha\odot p^{(n)})
 =\alpha\odot\cen(p^{(1)},\ldots,p^{(n)}).
\end{aligned}
\]
Applying the two distance identities term by term proves the formulas
for $V_\delta$.

Finally, Theorem~\ref{thm:characterization} and
\eqref{eq:aitchison-M-scalar} give
\[
 B_\delta(v,w)
 =\kappa_D^2
 \left\langle\clr_D(v),\clr_D(w)\right\rangle_2.
\]
For $\alpha,\beta>0$, bilinearity and absolute homogeneity give
\[
 B_\delta(\alpha\odot v,\beta\odot w)
 =\alpha\beta B_\delta(v,w),\qquad
 N_\delta(\alpha\odot v)=\alpha N_\delta(v).
\]
The factor $\kappa_D^2$, and then the positive powering factors, cancel
upon normalizing, proving the cosine claims.  To see that every
nonconstant shared coordinate bias can change a cosine, fix $c\ne e_D$
and put $y:=\clr_D(c)\ne0$.  Surjectivity of the CLR map
provides $v,w\in\Sm{D}$ with CLR vectors $y/2$ and $-y/2$,
respectively.  Then $\rho_\delta(v,w)=-1$, whereas
$\rho_\delta(c\oplus v,c\oplus w)=1$.
\end{proof}

\begin{aremark}[Why the compatibility and scaling clauses are necessary]
\label{rem:aitchison-characterization-scope}
Theorem~\ref{thm:characterization} is a fixed-dimensional
result.  If metrics are separately postulated in different dimensions,
the theorem gives $\delta_m=\kappa_m d_{A}$ (in each dimension $m$),
but \textup{(A1)}--\textup{(A5)} do not force the constants $\kappa_m$
to agree.  In general, the projection inequality gives only
\[
 \delta_m(p^{(S)},q^{(S)})
 \le \frac{\kappa_m}{\kappa_D}\,\delta_D(p,q).
\]
Thus $\kappa_m\le\kappa_D$ is sufficient for dominance, while a common
calibration is the clean equality convention used above.  For example, let
$\delta_3=\dA$, $\delta_2=100\,\dA$, $q=e_3$,
$p=\clo((\exp(a),\exp(-a),1))$ with $a\ne0$,
and take $S=\{1,2\}$.  Then
$\dA(p,q)=\dA(p^{(S)},q^{(S)})=\sqrt2\,|a|$
(three- and two-part CLR distances, respectively), and therefore
\[
 \delta_2(p^{(S)},q^{(S)})=100\sqrt2\,|a|
 >\sqrt2\,|a|=\delta_3(p,q).
\]
Thus cross-dimensional dominance fails even though both metrics separately
satisfy \textup{(A1)}--\textup{(A5)}.  The common-calibration family in
the corollary is therefore an explicit compatibility convention, not a
consequence of the fixed-$D$ axioms alone.
Likewise, the unscaled balance identity must be adjusted when
$\kappa_D\ne1$:
$\delta_D(p,q)^2=\kappa_D^2\dA(p',q')^2+\kappa_D^2(b(p)-b(q))^2$,
equivalently one uses $\delta_D^\perp=\kappa_D\dA$ on contents and
$b_\delta=\kappa_Db$, as in the corollary.  If instead an independently
calibrated distance $\delta_{D-1}=\kappa_{D-1}\dA$ is used, the exact
identity is
\[
 \delta_D(p,q)^2
 =\left(\frac{\kappa_D}{\kappa_{D-1}}\right)^2
   \delta_{D-1}(p',q')^2
  +\bigl(b_\delta(p)-b_\delta(q)\bigr)^2.
\]
Finally, shared-bias invariance applies when both endpoints of a distance,
or all rows and their center, receive the same perturbation.  It does not
imply that the one-row concentration $\dA(p',e_{D-1})$, or the CLR cosine,
is invariant under an arbitrary nonconstant content bias with the reference
$e_{D-1}$ held fixed.
\end{aremark}

\begin{aremark}[Tightness of the axioms]
\label{rem:l1}
$N_1(v)=\lVert\clr v\rVert_1$ satisfies (A1)--(A3) and (A5): it is a norm on the Aitchison vector space (hence induces a perturbation-invariant metric), it is absolutely homogeneous under powering because $\clr(\alpha\odot v)=\alpha\clr v$, and the $\ell_1$ norm is permutation-symmetric. It fails (A4) whenever $D\ge3$: for $u=(1,-1,0,\dots)$ and $v=(1,0,-1,0,\dots)$ in $\mathcal{H}$, $N_1(u)=N_1(v)=2$, $N_1(u+v)=4$, $N_1(u-v)=2$, and $16+4\ne2\cdot4+2\cdot4$. So the parallelogram axiom is not implied by the others; dropping it admits genuinely different coherent geometries, still perturbation- and permutation-invariant and temperature-homogeneous, among which (A4) selects the unique Euclidean one. This is the precise sense in which ``quasi-coherent'' log-ratio alternatives \citep{greenacre2022reappraisal} relate to $\dA$.
\end{aremark}

\section{Proofs for Section~\ref{sec:consequences}}
\label{app:consequences}

Throughout this section, coordinate $0$ is the designated sink, $x'$ denotes
the sink-dropped, re-closed content composition of $x\in\Sm{D}$ ($D\ge3$),
$H(x):=-\sum_j x_j\log x_j$ with $0\log0:=0$, and $\dAc(x,y):=\dA(x',y')$,
$\CAc(x):=\lVert\clr_{D-1}(x')\rVert_2$ with no dimension-dependent rescaling.

\subsection*{Full statement of Theorem~\ref{thm:range}}
Fix an integer $D\geq3$ and $s_0\in(\tfrac12,1)$, and designate
coordinate $0$ as the sink coordinate.
\textup{(i)} For nonzero $x,y\in\mathbb R^D$, write
$c(x,y):=\langle x,y\rangle/(\lVert x\rVert_2\lVert y\rVert_2)$,
and let $\JS$ denote the equal-weight Jensen--Shannon divergence formed
from Kullback--Leibler divergence with natural logarithms. If
$p,q\in\Sm{D}$ satisfy $p_0,q_0\geq s_0$, then
\[
\begin{aligned}
1-c(p,q)
&<
\frac{(1-s_0)^2}{s_0^2+(1-s_0)^2}
\leq
\left(\frac{1-s_0}{s_0}\right)^2,\\
\lVert p-q\rVert_2
&<
\sqrt{2}\,(1-s_0)
\leq
\sqrt{5}\,(1-s_0),\\
\JS(p,q)
&<
(1-s_0)\log 2
\leq
\frac{3}{2}(1-s_0).
\end{aligned}
\]
All three sharp endpoints are nevertheless approached from within the open
simplex. More precisely,
\[
\begin{aligned}
\sup_{\substack{p,q\in\Sm{D}\\p_0,q_0\geq s_0}}
  \{1-c(p,q)\}
&=
\frac{(1-s_0)^2}{s_0^2+(1-s_0)^2},\\
\sup_{\substack{p,q\in\Sm{D}\\p_0,q_0\geq s_0}}
  \lVert p-q\rVert_2
&=
\sqrt{2}\,(1-s_0),\\
\sup_{\substack{p,q\in\Sm{D}\\p_0,q_0\geq s_0}}
  \JS(p,q)
&=
(1-s_0)\log 2.
\end{aligned}
\]
Thus none of the three suprema is attained. In fact, their respective ranges
are exactly
\[
\left[0,\frac{(1-s_0)^2}{s_0^2+(1-s_0)^2}\right),
\qquad
\left[0,\sqrt{2}(1-s_0)\right),
\qquad
\left[0,(1-s_0)\log 2\right),
\]
and their upper endpoints are sharp. In particular, all three range widths
tend to zero as $s_0\uparrow1$, while the cosine range has the quadratic
asymptotic behavior
$\frac{(1-s_0)^2}{s_0^2+(1-s_0)^2}\sim(1-s_0)^2$ as $s_0\uparrow1$.
\textup{(ii)} For $r\in\Sm{D}$, define the conditional content composition
and the content-preserving sink-mass transformation by
\[
  r':=\frac{(r_1,\ldots,r_{D-1})}{1-r_0},
  \qquad
  T_t r:=\bigl(t,(1-t)r'\bigr),
  \qquad 0<t<1.
\]
Then, for arbitrary $r,z\in\Sm{D}$ and $t,\tau\in(0,1)$,
\[
  \dAc(T_t r,T_\tau z)=\dAc(r,z),
  \qquad
  \CAc(T_t r)=\CAc(r),
  \qquad
  \CAc(T_\tau z)=\CAc(z).
\]
Consequently, for every integer $n\geq1$, every
$r^{(1)},\ldots,r^{(n)}\in\Sm{D}$, and every
$t_1,\ldots,t_n\in(0,1)$,
\[
  \dAc\bigl(T_{t_i}r^{(i)},T_{t_j}r^{(j)}\bigr)
  =
  \dAc\bigl(r^{(i)},r^{(j)}\bigr)
  \quad(1\leq i,j\leq n),
  \qquad
  \CAc\bigl(T_{t_i}r^{(i)}\bigr)
  =
  \CAc\bigl(r^{(i)}\bigr)
  \quad(1\leq i\leq n).
\]
Hence composition-specific changes of sink mass that preserve each conditional
content composition leave the entire content-distance matrix and all
content-concentration values unchanged. In particular, choosing
$r^{(i)}_0<t_i<1$ increases every designated sink mass while leaving the
conditional-content geometry exactly unchanged.

\begin{proof}[Proof of Theorem~\ref{thm:range}]
Fix $p,q\in\Sm{D}$ with $p_0,q_0\geq s_0$.
Put $\varepsilon:=1-s_0$, and set
\[
  s_p:=p_0,
  \qquad
  s_q:=q_0,
  \qquad
  \pi_p:=\frac{(p_1,\ldots,p_{D-1})}{1-s_p},
  \qquad
  \pi_q:=\frac{(q_1,\ldots,q_{D-1})}{1-s_q}.
\]
Because $p,q\in\Sm{D}$, the vectors $\pi_p,\pi_q$ are strictly
positive probability vectors with $D-1\geq2$ coordinates. Hence
\[
  p=\bigl(s_p,(1-s_p)\pi_p\bigr),
  \qquad
  q=\bigl(s_q,(1-s_q)\pi_q\bigr),
\]
and
\[
  \lVert\pi_p\rVert_2<1,
  \qquad
  \lVert\pi_q\rVert_2<1,
  \qquad
  \langle\pi_p,\pi_q\rangle>0.
\]

\emph{Cosine dissimilarity.}
We have
\[
\begin{aligned}
\langle p,q\rangle
&=
s_ps_q+(1-s_p)(1-s_q)\langle\pi_p,\pi_q\rangle,\\
\lVert p\rVert_2^2
&=
s_p^2+(1-s_p)^2\lVert\pi_p\rVert_2^2
\leq
s_p^2+(1-s_p)^2,
\end{aligned}
\]
and the analogous norm inequality holds for $q$. Define
\[
  h(s):=\frac{s}{\sqrt{s^2+(1-s)^2}},
  \qquad 0<s<1.
\]
Since $s_p,s_q<1$ and $\langle\pi_p,\pi_q\rangle>0$,
\[
\begin{aligned}
  c(p,q)
  &>
  \frac{s_ps_q}{\lVert p\rVert_2\lVert q\rVert_2}\\
  &\geq
  \frac{s_ps_q}
  {\sqrt{s_p^2+(1-s_p)^2}
   \sqrt{s_q^2+(1-s_q)^2}}
  =
  h(s_p)h(s_q).
\end{aligned}
\]
Moreover,
\[
  h'(s)
  =
  \frac{1-s}{\{s^2+(1-s)^2\}^{3/2}}
  >0.
\]
Since $s_p,s_q\geq s_0$, it follows that
\[
  c(p,q)
  >
  h(s_0)^2
  =
  \frac{s_0^2}{s_0^2+(1-s_0)^2},
\]
which proves the first strict bound. The coarser bound follows because
$s_0^2+(1-s_0)^2\geq s_0^2$.

\emph{Euclidean distance.}
Let $a:=1-s_p$ and $b:=1-s_q$. Then
$0<a,b\leq\varepsilon$, and
\[
\begin{aligned}
\lVert p-q\rVert_2^2
&=
(a-b)^2+\lVert a\pi_p-b\pi_q\rVert_2^2\\
&=
(a-b)^2
+a^2\lVert\pi_p\rVert_2^2
+b^2\lVert\pi_q\rVert_2^2
-2ab\langle\pi_p,\pi_q\rangle\\
&<
(a-b)^2+a^2+b^2\\
&=
2(a^2+b^2-ab).
\end{aligned}
\]
If $a\geq b$, then
\[
  a^2+b^2-ab
  =
  a^2-b(a-b)
  \leq a^2;
\]
if $b\geq a$, the symmetric argument gives
$a^2+b^2-ab\leq b^2$. Therefore
\[
  \lVert p-q\rVert_2^2
  <
  2\max\{a,b\}^2
  \leq
  2\varepsilon^2,
\]
which proves the second strict bound; the displayed $\sqrt{5}$-bound is the
immediate numerical relaxation $\sqrt{2}\leq\sqrt{5}$.

\emph{Jensen--Shannon divergence.}
Let $e_0,\ldots,e_{D-1}$ be the standard coordinate vectors in
$\mathbb R^D$, and define
\[
  u:=\frac{p-s_0e_0}{\varepsilon},
  \qquad
  v:=\frac{q-s_0e_0}{\varepsilon}.
\]
The assumptions $p_0,q_0\geq s_0$ imply that $u,v$ have nonnegative
coordinates, and
\[
  \sum_{i=0}^{D-1}u_i
  =
  \sum_{i=0}^{D-1}v_i
  =
  \frac{1-s_0}{\varepsilon}
  =
  1.
\]
Thus $u,v$ are probability vectors in the closed simplex, and
\[
  p=s_0e_0+\varepsilon u,
  \qquad
  q=s_0e_0+\varepsilon v.
\]
Set
\[
  w:=\frac{u+v}{2},
  \qquad
  m:=\frac{p+q}{2}
    =s_0e_0+\varepsilon w.
\]
Then
$\operatorname{supp}(u)\cup\operatorname{supp}(v)
=\operatorname{supp}(w)$.
We use the standard closed-simplex convention
\[
  \KL(x\Vert y)
  :=
  \sum_{i:x_i>0}x_i\log\frac{x_i}{y_i}
\]
when
$\operatorname{supp}(x)\subseteq\operatorname{supp}(y)$, and set the
divergence equal to $+\infty$ otherwise. Thus all divergences below are
finite. Joint convexity of Kullback--Leibler divergence gives
\[
\begin{aligned}
  \KL(p\Vert m)
  &\leq
  s_0\KL(e_0\Vert e_0)
  +\varepsilon\KL(u\Vert w)
  =
  \varepsilon\KL(u\Vert w),\\
  \KL(q\Vert m)
  &\leq
  s_0\KL(e_0\Vert e_0)
  +\varepsilon\KL(v\Vert w)
  =
  \varepsilon\KL(v\Vert w).
\end{aligned}
\]
Consequently,
\[
  \JS(p,q)
  \leq
  \varepsilon\JS(u,v).
\]

For any closed-simplex probability vectors $x,y$, put $z=(x+y)/2$.
Whenever $x_i>0$,
\[
  \log\frac{x_i}{z_i}
  =
  \log\frac{2x_i}{x_i+y_i}
  \leq
  \log 2,
\]
with equality if and only if $y_i=0$. Hence
$\KL(x\Vert z)\leq\log 2$,
and the analogous conclusion holds with $x,y$ interchanged. It follows that
\[
  \JS(x,y)\leq\log 2,
  \qquad
  \JS(x,y)=\log 2
  \quad\Longleftrightarrow\quad
  \operatorname{supp}(x)\cap\operatorname{supp}(y)=\varnothing.
\]
For every $i=1,\ldots,D-1$, however,
\[
  u_i=\frac{p_i}{\varepsilon}>0,
  \qquad
  v_i=\frac{q_i}{\varepsilon}>0.
\]
Thus the supports of $u$ and $v$ overlap, and therefore
\[
  \JS(p,q)
  \leq
  \varepsilon\JS(u,v)
  <
  \varepsilon\log 2
  =
  (1-s_0)\log 2.
\]
This proves the third strict bound; $\log 2<\frac{3}{2}$ gives the stated
coarser bound.

\emph{Sharpness.}
The assumption $D\geq3$ provides two distinct nonsink coordinates. In the
closed simplex, let
\[
  p^*:=s_0e_0+\varepsilon e_1,
  \qquad
  q^*:=s_0e_0+\varepsilon e_2,
  \qquad
  m^*:=\frac{p^*+q^*}{2}.
\]
Direct calculation gives
\[
\begin{aligned}
  1-c(p^*,q^*)
  &=
  \frac{\varepsilon^2}{s_0^2+\varepsilon^2},\\
  \lVert p^*-q^*\rVert_2
  &=
  \sqrt{2}\,\varepsilon,\\
  \KL(p^*\Vert m^*)
  &=
  \KL(q^*\Vert m^*)
  =
  \varepsilon\log 2,
\end{aligned}
\]
and hence
$\JS(p^*,q^*)=\varepsilon\log 2$.

To approximate this pair from within $\Sm{D}$, fix
$0<\delta<1/(D-1)$ and, for $j\in\{1,\ldots,D-1\}$, define
\[
  (\pi_{p,\delta})_j
  :=
  \begin{cases}
    1-(D-2)\delta,&j=1,\\
    \delta,&j\neq1,
  \end{cases}
  \qquad
  (\pi_{q,\delta})_j
  :=
  \begin{cases}
    1-(D-2)\delta,&j=2,\\
    \delta,&j\neq2.
  \end{cases}
\]
Set
\[
  p_\delta
  :=
  \bigl(s_0,\varepsilon\pi_{p,\delta}\bigr),
  \qquad
  q_\delta
  :=
  \bigl(s_0,\varepsilon\pi_{q,\delta}\bigr).
\]
Each $\pi_{p,\delta}$ and $\pi_{q,\delta}$ is a strictly positive
probability vector, so $p_\delta,q_\delta\in\Sm{D}$, and
\[
  (p_\delta,q_\delta)
  \longrightarrow
  (p^*,q^*)
  \qquad\text{as }\delta\downarrow0.
\]
Cosine dissimilarity and Euclidean distance are continuous at this
closed-simplex limit. Jensen--Shannon divergence is also continuous there
because
\[
  \JS(x,y)
  =
  H\!\left(\frac{x+y}{2}\right)
  -\frac{H(x)+H(y)}{2},
\]
with $0\log 0:=0$, and $t\mapsto-t\log t$ is continuous on $[0,1]$
under this convention. The three upper endpoints are therefore the exact
suprema. Their nonattainment follows from the strict bounds already proved.

Finally, all three dissimilarities are nonnegative: for cosine dissimilarity
this follows from Cauchy--Schwarz, for Euclidean distance it is immediate,
and for Jensen--Shannon divergence it follows from nonnegativity of
Kullback--Leibler divergence. Each takes the value zero when $p=q$.
The feasible set of pairs is convex and hence path connected, and each of the
three dissimilarities is continuous on it. The image of this set under each
dissimilarity is therefore an interval. Because that interval contains zero,
has the displayed supremum, and does not contain its supremum, it is exactly
the corresponding half-open interval stated in part~\textup{(i)}. This proves
all assertions in part~\textup{(i)}.

For part~\textup{(ii)}, fix $r,z\in\Sm{D}$ and $t,\tau\in(0,1)$.
For $j=1,\ldots,D-1$, one has $r'_j>0$, and
\[
  \sum_{j=1}^{D-1}r'_j
  =
  \frac{\sum_{j=1}^{D-1}r_j}{1-r_0}
  =
  1.
\]
Thus $r'$ is a strictly positive probability vector, and the same holds for
$z'$. Moreover, all coordinates of $T_t r$ are strictly positive and
\[
  \sum_{j=0}^{D-1}(T_t r)_j
  =
  t+(1-t)\sum_{j=1}^{D-1}r'_j
  =
  1.
\]
Hence $T_t r\in\Sm{D}$, and similarly $T_\tau z\in\Sm{D}$. Directly,
\[
  (T_t r)'
  =
  \frac{(1-t)r'}{1-t}
  =
  r',
  \qquad
  (T_\tau z)'
  =
  \frac{(1-\tau)z'}{1-\tau}
  =
  z'.
\]
Therefore
\[
\begin{aligned}
  \dAc(T_t r,T_\tau z)
  &=
  \dA\bigl((T_t r)',(T_\tau z)'\bigr)
  =
  \dA(r',z')
  =
  \dAc(r,z),\\
  \CAc(T_t r)
  &=
  \lVert\clr_{D-1}((T_t r)')\rVert_2
  =
  \lVert\clr_{D-1}(r')\rVert_2
  =
  \CAc(r),
\end{aligned}
\]
and the same calculation yields
$\CAc(T_\tau z)=\CAc(z)$.
Applying these identities to every pair $(i,j)$ and every index $i$ in a
finite family proves the two finite-family claims. Finally, if
$r^{(i)}_0<t_i<1$, then
$\bigl(T_{t_i}r^{(i)}\bigr)_0=t_i>r^{(i)}_0$,
which proves the last assertion. The restrictions $t_i\in(0,1)$ ensure that
every transformed vector remains in the open simplex; no endpoint assertion
is required.
\end{proof}

\subsection*{Full statement of Proposition~\ref{prop:taxonomy}}
Fix an integer $D\geq3$, index the sink coordinate by $0$, and,
for $x=(x_0,\ldots,x_{D-1})\in\Sm{D}$, write
$x':=\clo(x_1,\ldots,x_{D-1})=(x_1,\ldots,x_{D-1})/(1-x_0)\in\Sm{D-1}$.
\emph{(i)}
Let $h_1,\ldots,h_n\in\Sm{D}$, $n\geq2$, be a finite labelled
collection of heads. Define the content-distance matrix computed without
first modifying the stored heads and the Aitchison-distance matrix computed
after explicit sink removal by
\[
\Delta^{\mathrm{content}}_{ij}
:=\dAc(h_i,h_j),
\qquad
\Delta^{\mathrm{explicit\text{-}drop}}_{ij}
:=\dA(h_i',h_j').
\]
Then
$\Delta^{\mathrm{content}}=\Delta^{\mathrm{explicit\text{-}drop}}$
entrywise.
Consequently, any deterministic clustering rule of the form
$\mathcal A(\Delta;\eta)$, whose sole data-dependent input is the
labelled dissimilarity matrix and whose ancillary specification $\eta$,
including matrix-level preprocessing, linkage, hyperparameters,
tie-breaking, and stopping rules, is held fixed, gives identical outputs
in the two analyses. In particular, the rule may not additionally inspect
the raw heads, their sink masses, or their ambient dimension.
If the same measurable randomized implementation is written as
$\mathcal A(\Delta,U;\eta)$, and in both analyses $U$ has conditional
law $K(du\mid\Delta,\eta)$, then the outputs have the same distribution:
for every measurable output event $B$, both output probabilities equal
\[
\int
\mathbf 1\!\left\{\mathcal A(\Delta,u;\eta)\in B\right\}
K(du\mid\Delta,\eta).
\]
Moreover, there exists a common-randomness coupling: draw one
$U\sim K(\,\cdot\mid\Delta,\eta)$
and use it in both analyses; under this coupling the outputs are equal
almost surely. This identity concerns $\dAc$; it does not assert that
the ordinary full-composition distance $\dA(h_i,h_j)$ is unchanged by
sink removal.
\emph{(ii)}
For probability vectors $x,y$ of the same finite length, possibly
with zero coordinates, define
$c(x,y):=\langle x,y\rangle/(\|x\|_2\|y\|_2)$,
$d_{\cos}(x,y):=1-c(x,y)$, $d_{\mathrm E}(x,y):=\|x-y\|_2$, and
\[
\JS(x,y)
:=
H\!\left(\frac{x+y}{2}\right)
-\frac12H(x)-\frac12H(y)
=
\frac12\KL\!\left(x\middle\Vert\frac{x+y}{2}\right)
+\frac12\KL\!\left(y\middle\Vert\frac{x+y}{2}\right),
\]
where $\KL(x\Vert y):=\sum_{\ell:x_\ell>0}x_\ell\log(x_\ell/y_\ell)$,
with $\KL(x\Vert y)=+\infty$ if $x_\ell>0$ and $y_\ell=0$ for some
$\ell$; in the displayed Jensen--Shannon formula the midpoint has support
containing those of both arguments, so both relative entropies are finite.
There exists a nonempty relatively open set
$\mathcal U_D\subset(\Sm{D})^3$
having positive relative $3(D-1)$-dimensional Lebesgue measure in
the affine hull, and hence positive $3(D-1)$-dimensional Hausdorff
measure, such that, for every $(p,q,r)\in\mathcal U_D$, simultaneously
for $\delta\in\{d_{\cos},\JS,d_{\mathrm E}\}$,
the unique closest pair with the sink retained is $(p,r)$, whereas
the unique closest pair after sink removal is $(p,q)$. Hence, for
each of the three dissimilarities, both standard single-linkage and
standard complete-linkage agglomerative clustering, initialized at
singleton clusters and stopped at $K=2$, return
\[
\mathcal P_{\mathrm{keep}}
=
\bigl\{\{p,r\},\{q\}\bigr\},
\qquad
\mathcal P_{\mathrm{drop}}
=
\bigl\{\{p,q\},\{r\}\bigr\}.
\]
For $D=3$, the set $\mathcal U_3$ contains the exact witness
\[
p=\frac1{50}(45,4,1),\qquad
q=\frac1{50}(35,12,3),\qquad
r=\frac1{50}(45,1,4).
\]

\begin{proof}[Proof of Proposition~\ref{prop:taxonomy}]
For part~\emph{(i)}, the standing definition gives, for every $i,j$,
\[
\Delta^{\mathrm{content}}_{ij}
=
\dAc(h_i,h_j)
=
\dA(h_i',h_j')
=
\Delta^{\mathrm{explicit\text{-}drop}}_{ij}.
\]
Thus both analyses supply exactly the same labelled matrix to the same
clustering implementation. The deterministic outputs are therefore equal.
In the randomized case, the displayed integral in the statement depends
only on the common pair $(\Delta,\eta)$, so the two output laws coincide.
If the same draw $U$ is used in both analyses, the two arguments of
$\mathcal A$ agree pointwise, and hence the coupled outputs are equal
almost surely.

For part~\emph{(ii)}, first establish nonemptiness in every dimension
$D\geq3$. Choose $\pi\in\Sm{D-1}$ whose first two coordinates are
unequal, and let $\widetilde\pi$ be obtained by interchanging those
two coordinates. Fix $a\in(0,1)$, and, for $s\in(a,1)$, define
\[
p_s=\bigl(s,(1-s)\pi\bigr),\qquad
q=\bigl(a,(1-a)\pi\bigr),\qquad
r_s=\bigl(s,(1-s)\widetilde\pi\bigr).
\]
Then
$p_s'=q'=\pi$ and $r_s'=\widetilde\pi\neq\pi$.

Each of $d_{\cos}$, $\JS$, and $d_{\mathrm E}$ is nonnegative
and vanishes on probability vectors precisely when its two arguments
are equal. This is immediate for $d_{\mathrm E}$. For $d_{\cos}$,
it follows from the equality case of Cauchy--Schwarz, because proportional
unit-sum probability vectors must be equal. For $\JS$, it follows from
the strict concavity of $H$. Hence, for every
$\delta\in\{d_{\cos},\JS,d_{\mathrm E}\}$,
\[
\delta(p_s',q')=0
<
\delta(p_s',r_s')
=
\delta(q',r_s').
\]

Let $e_0=(1,0,\ldots,0)$.
As $s\uparrow1$, both $p_s$ and $r_s$ converge to $e_0$,
whereas $q\neq e_0$. Euclidean and cosine dissimilarities are
continuous on the closed simplex; for cosine, the denominator never
vanishes there. The entropy definition of $\JS$, together with
$0\log0=0$, shows that $\JS$ is also continuous on the closed
simplex. Therefore, for each such $\delta$,
\[
\delta(p_s,r_s)\longrightarrow0,
\qquad
\delta(p_s,q)\longrightarrow\delta(e_0,q)>0,
\qquad
\delta(q,r_s)\longrightarrow\delta(q,e_0)>0.
\]
Because the family of dissimilarities is finite, there exists a common
$s_*\in(a,1)$, sufficiently close to $1$, such that, simultaneously
for all three,
\[
\delta(p_{s_*},r_{s_*})
<
\min\bigl\{
\delta(p_{s_*},q),\delta(q,r_{s_*})
\bigr\}.
\]

Define $\mathcal U_D$ to be the set of triples
$(p,q,r)\in(\Sm{D})^3$ satisfying, for every
$\delta\in\{d_{\cos},\JS,d_{\mathrm E}\}$,
\[
\begin{aligned}
\delta(p,r)&<\delta(p,q),
&
\delta(p,r)&<\delta(q,r),
\\
\delta(p',q')&<\delta(p',r'),
&
\delta(p',q')&<\delta(q',r').
\end{aligned}
\]
The preceding construction shows that
$(p_{s_*},q,r_{s_*})\in\mathcal U_D$,
so $\mathcal U_D$ is nonempty. The sink-removal map is continuous
on the open simplex because $1-x_0>0$, and all three dissimilarities
are continuous there. Thus $\mathcal U_D$, being defined by finitely
many strict inequalities, is relatively open in $(\Sm{D})^3$.
Every nonempty relatively open subset of this $3(D-1)$-dimensional
open set in its affine hull contains a relative open ball. It therefore
has positive relative Lebesgue measure and positive
$3(D-1)$-dimensional Hausdorff measure.

It remains to verify the displayed $D=3$ witness exactly. For
Euclidean distance,
\[
\|p-r\|_2^2
=
\frac{18}{2500}
<
\frac{168}{2500}
=
\|p-q\|_2^2
<
\frac{222}{2500}
=
\|q-r\|_2^2.
\]
Thus $(p,r)$ is the unique Euclidean closest pair.

Multiplying all three vectors by $50$, which does not change cosine
similarity, gives
$P=(45,4,1)$, $Q=(35,12,3)$, $R=(45,1,4)$.
Consequently,
\[
c(p,r)=\frac{2033}{2042},\qquad
c(p,q)=\frac{1626}{\sqrt{2042\cdot1378}},\qquad
c(q,r)=\frac{1599}{\sqrt{2042\cdot1378}}.
\]
The last two quantities have the same positive denominator. Moreover,
all three displayed cosine similarities are positive, and therefore
$c(p,r)>c(p,q)$ is equivalent, after squaring, to
\[
2033^2\cdot1378-1626^2\cdot2042
=
296601850>0.
\]
This comparison and $1626>1599$ imply
$c(p,r)>c(p,q)>c(q,r)$.
Thus $(p,r)$ is also the unique closest pair under
$d_{\cos}=1-c$.

We next record the two Jensen--Shannon mixture identities used below.
Let $u,v$ be probability vectors of the same size, put
$w=(u+v)/2$, and let $s,t\in(0,1)$. Then
\[
\KL\!\left(
(s,(1-s)u)\middle\Vert(s,(1-s)w)
\right)
=
(1-s)\KL(u\Vert w),
\]
and the analogous identity holds with $v$. It follows that
\[
\JS\bigl((s,(1-s)u),(s,(1-s)v)\bigr)
=
(1-s)\JS(u,v).
\]
Next put $\bar s=(s+t)/2$. The midpoint of
$(s,(1-s)u)$ and $(t,(1-t)u)$ is
$(\bar s,(1-\bar s)u)$, and direct summation gives
\[
\KL\!\left(
(s,(1-s)u)\middle\Vert
(\bar s,(1-\bar s)u)
\right)
=
\KL\bigl(
(s,1-s)\Vert(\bar s,1-\bar s)
\bigr).
\]
Applying the analogous identity to $t$ and averaging yields
\[
\JS\bigl((s,(1-s)u),(t,(1-t)u)\bigr)
=
\JS\bigl((s,1-s),(t,1-t)\bigr).
\]

For the witness, put
\[
\pi_+=\left(\frac45,\frac15\right),
\qquad
\pi_-=\left(\frac15,\frac45\right).
\]
The first mixture identity and direct calculation give
\[
\begin{aligned}
\JS(p,r)
&=
\frac1{10}\JS(\pi_+,\pi_-)
=
\frac1{10}\left(
\frac45\log\frac85+\frac15\log\frac25
\right)
=
\frac1{50}\log\!\left(\frac{8192}{3125}\right)
<
\frac1{50}.
\end{aligned}
\]
Indeed,
\[
\frac{8192}{3125}
<
\frac{65}{24}
=
\sum_{k=0}^{4}\frac1{k!}
<
\mathrm e.
\]

Because $p$ and $q$ have the same conditional content composition,
the second mixture identity gives
\[
\JS(p,q)
=
\JS\!\left(
 \left(\frac9{10},\frac1{10}\right),
 \left(\frac7{10},\frac3{10}\right)
\right).
\]
The midpoint of these two binary probability vectors is
$m=\left(\frac45,\frac15\right)$,
and each endpoint has total-variation distance $1/10$ from $m$,
where $\operatorname{TV}(u,v):=\frac12\|u-v\|_1$.
Pinsker's inequality for natural logarithms,
$\KL(u\Vert v)\geq2\operatorname{TV}(u,v)^2$,
therefore gives
\[
\JS(p,q)
\geq
\frac12\left(\frac1{50}+\frac1{50}\right)
=
\frac1{50}
>
\JS(p,r).
\]

Now define the deterministic coarse-graining
$T(x_0,x_1,x_2):=(x_0,x_1+x_2)$.
Since
$T\!\left(\frac{x+y}{2}\right)=\frac{Tx+Ty}{2}$,
the data-processing inequality applied to each relative-entropy term
defining Jensen--Shannon divergence gives
\[
\begin{aligned}
\JS(q,r)
&\geq
\JS(Tq,Tr)
=
\JS\!\left(
 \left(\frac7{10},\frac3{10}\right),
 \left(\frac9{10},\frac1{10}\right)
\right)
=
\JS(p,q)
>
\JS(p,r).
\end{aligned}
\]
The last equality uses symmetry of $\JS$ and the second mixture
identity. Thus $(p,r)$ is the unique Jensen--Shannon closest pair.

After sink removal,
\[
p'=q'=\left(\frac45,\frac15\right),
\qquad
r'=\left(\frac15,\frac45\right).
\]
The $(p,q)$ dissimilarity is therefore zero. Since $p'=q'$, the
common dissimilarity of each of the other two pairs is
\[
\begin{aligned}
d_{\cos}(p',r')
&=\frac9{17}>0,
\\
\JS(p',r')
&=\frac15\log\!\left(\frac{8192}{3125}\right)>0,
\\
d_{\mathrm E}(p',r')
&=\frac{3\sqrt2}{5}>0.
\end{aligned}
\]
Hence $(p,q)$ is uniquely closest after sink removal under all three
dissimilarities.

Finally, with three singleton clusters, both single linkage and complete
linkage assign to a pair of clusters exactly the underlying pairwise
dissimilarity. Their first merge is therefore the unique closest pair.
Stopping immediately after that merge, at $K=2$, gives the two asserted
partitions.
\end{proof}

\subsection*{Full statement of Corollary~\ref{cor:checkpoint}}
Fix an integer $D\geq3$, let
$\varnothing\neq\mathcal T$ be a checkpoint index set, and let
$p_t\in\Sm{D}$ for every $t\in\mathcal T$. Write uniquely
\[
p_t=\bigl(s_t,(1-s_t)\pi_t\bigr),
\qquad
s_t=(p_t)_0\in(0,1),
\qquad
\pi_t=p_t'\in\Sm{D-1},
\]
and define $H_b(s):=-s\log s-(1-s)\log(1-s)$.
Then, for every $t\in\mathcal T$,
\begin{equation}
H(p_t)
=
H_b(s_t)+(1-s_t)H(\pi_t).
\label{eq:checkpoint-entropy-decomposition}
\end{equation}
Suppose that, for some $\pi\in\Sm{D-1}$,
$\pi_t=\pi$ for every $t\in\mathcal T$.
Then $\CAc(p_t)=\|\clr_{D-1}(\pi)\|_2$ is independent of $t$. If
\[
f_\pi(s)
:=
H\bigl(s,(1-s)\pi\bigr),
\qquad s\in(0,1),
\qquad
\theta_\pi
:=
\frac{1}{1+\exp(H(\pi))},
\]
then $f_\pi$ is strictly increasing on $(0,\theta_\pi]$ and
strictly decreasing on $[\theta_\pi,1)$, with unique maximizer
$\theta_\pi$ on $(0,1)$. In particular,
$\theta_\pi<1/2$, and, for any $t,u\in\mathcal T$,
\[
\frac12\leq s_t<s_u<1
\quad\Longrightarrow\quad
H(p_u)<H(p_t)
\quad\text{and}\quad
\CAc(p_u)=\CAc(p_t).
\]
Moreover, let
$u_{D-1}:=(\tfrac1{D-1},\ldots,\tfrac1{D-1})$
and $p_s:=\bigl(s,(1-s)u_{D-1}\bigr)$ for $s\in(0,1)$.
Then
\[
p_s'=u_{D-1},
\qquad
\CAc(p_s)=0,
\qquad
H(p_{1/D})=\log D,
\qquad
\lim_{s\uparrow1}H(p_s)=0.
\]
The map $s\mapsto H(p_s)$ is strictly decreasing on
$[1/D,1)$, and
\[
\sup_{s\in(1/D,1)}
\left\{
H(p_{1/D})-H(p_s)
\right\}
=
\log D.
\]
Consequently, for every $\varepsilon\in(0,\log D)$, there exists
$s_\varepsilon\in(1/2,1)$ such that
$H(p_{1/D})-H(p_{s_\varepsilon})>\log D-\varepsilon$,
even though the conditional content composition and $\CAc$ are
unchanged between $p_{1/D}$ and $p_{s_\varepsilon}$.
Thus, variation in sink mass along a path with fixed conditional content
composition can produce an apparent entropy collapse relative to the
content-only diagnostic, of size arbitrarily close to the nonattained
supremum $\log D$. This is a possibility result; it does not by itself
assert that sink mass increases along any empirical training trajectory.

\begin{proof}[Proof of Corollary~\ref{cor:checkpoint}]
Since $p_t\in\Sm{D}$, one has $0<s_t<1$. Therefore
\[
\pi_t
=
\frac{\bigl((p_t)_1,\ldots,(p_t)_{D-1}\bigr)}{1-s_t}
\in\Sm{D-1}
\]
is well defined, and the displayed representation of $p_t$ is unique.

Using $\sum_{i=1}^{D-1}\pi_{t,i}=1$, we obtain
\begin{align*}
H(p_t)
&=
-s_t\log s_t
-\sum_{i=1}^{D-1}
 (1-s_t)\pi_{t,i}
 \log\bigl((1-s_t)\pi_{t,i}\bigr)
\\
&=
-s_t\log s_t
-(1-s_t)\log(1-s_t)
-(1-s_t)
 \sum_{i=1}^{D-1}\pi_{t,i}\log\pi_{t,i}
\\
&=
H_b(s_t)+(1-s_t)H(\pi_t),
\end{align*}
which proves \eqref{eq:checkpoint-entropy-decomposition}. This is a
pointwise identity, so it requires no topology or ordering on
$\mathcal T$.

If $\pi_t=\pi$ for every $t$, then $p_t'=\pi$, and the definition
of content Aitchison concentration gives
\[
\CAc(p_t)
=
\|\clr_{D-1}(p_t')\|_2
=
\|\clr_{D-1}(\pi)\|_2.
\]
Furthermore,
$f_\pi(s)=H_b(s)+(1-s)H(\pi)$,
and hence
\[
f_\pi'(s)
=
\log\frac{1-s}{s}-H(\pi),
\qquad
f_\pi''(s)
=
-\frac{1}{s(1-s)}
<0,
\qquad 0<s<1,
\]
so $f_\pi$ is strictly concave. Because
$s\mapsto\log((1-s)/s)$ is strictly decreasing on $(0,1)$,
\[
f_\pi'(s)=0
\iff
\frac{1-s}{s}=\exp(H(\pi))
\iff
s=\frac{1}{1+\exp(H(\pi))}
=\theta_\pi.
\]
Consequently,
\[
\begin{cases}
f_\pi'(s)>0, & 0<s<\theta_\pi,\\
f_\pi'(s)=0, & s=\theta_\pi,\\
f_\pi'(s)<0, & \theta_\pi<s<1.
\end{cases}
\]
This proves the asserted strict monotonicity and unique maximizer.

Since $D-1\geq2$ and every coordinate of $\pi$ is strictly
positive, $H(\pi)>0$, so $\theta_\pi<1/2$. Therefore, if
$\frac12\leq s_t<s_u<1$,
then both $s_t$ and $s_u$ lie in the strictly decreasing region
of $f_\pi$, giving
$H(p_u)=f_\pi(s_u)<f_\pi(s_t)=H(p_t)$.
The equality of the two content log-ratio dispersions was established above.

Finally,
$H(u_{D-1})=\log(D-1)$ and $\clr_{D-1}(u_{D-1})=0$,
and $p_s'=u_{D-1}$. Therefore
$\CAc(p_s)=0$ for every $s\in(0,1)$.
The corresponding threshold is
\[
\theta_{u_{D-1}}
=
\frac{1}{1+\exp(\log(D-1))}
=
\frac1D.
\]
The preceding monotonicity result shows that
$s\mapsto H(p_s)$ is strictly decreasing on $[1/D,1)$.

At $s=1/D$, every coordinate of $p_s$ equals $1/D$, so
$H(p_{1/D})=\log D$.
Moreover,
\[
H(p_s)
=
H_b(s)+(1-s)\log(D-1)
\longrightarrow0
\qquad\text{as }s\uparrow1.
\]
Since $H(p_s)>0$ for every $s\in(0,1)$, it follows that
$H(p_{1/D})-H(p_s)<\log D$ for every $s\in(1/D,1)$.
On the other hand,
$H(p_{1/D})-H(p_s)\longrightarrow\log D$ as $s\uparrow1$.
Therefore
\[
\sup_{s\in(1/D,1)}
\left\{
H(p_{1/D})-H(p_s)
\right\}
=
\log D.
\]
The supremum is not attained because $s=1$ is excluded and
$H(p_s)>0$ for every admissible $s$.

Given $\varepsilon\in(0,\log D)$, choose
$s_\varepsilon\in(1/2,1)$, sufficiently close to $1$, such that
$H(p_{s_\varepsilon})<\varepsilon$.
Then
\[
H(p_{1/D})-H(p_{s_\varepsilon})
=
\log D-H(p_{s_\varepsilon})
>
\log D-\varepsilon,
\]
as claimed.
\end{proof}

\begin{remark}[$\CAc$ orders by dispersion, not majorization]\label{rem:schur}
$\CAc$ is not Schur-convex and is not an entropy substitute. For $p=(0.800,0.199,0.001)$ and $q=(0.780,0.219,0.001)$, $p$ majorizes $q$ and $H(p)\approx0.507<0.533\approx H(q)$, yet $\lVert\clr(p)\rVert_2\approx4.988<5.000\approx\lVert\clr(q)\rVert_2$: near-zero coordinates dominate the log-ratio norm. Accordingly the paper uses $\CAc$ as a dispersion statistic of the content log-ratios and routes concentration and collapse claims through the exact decomposition $H(p)=H_b(s)+(1-s)H(\pi)$, whose content term $H(\pi)$ carries the entropy order.
\end{remark}

\begin{lemma}[Quotient characterization of $\dAc$]\label{lem:quotient}
Define $p\sim q$ iff $R_0p=R_0q$ (equal content subcompositions). Then $\dAc$ is the quotient metric induced by $\dA$ on $\mathcal{S}/{\sim}$:
$\dAc(p,q)=\min_{p^*\sim p,\;q^*\sim q}\dA(p^*,q^*)$,
and Theorem~\ref{thm:characterization} applied on the content simplex characterizes $\dAc$ as the unique metric on the quotient satisfying A1--A4 for content transformations.
\end{lemma}
\begin{proof}
By Lemma~\ref{lem:pythagoras}, $\dA(p^*,q^*)^2=\dAc(p,q)^2+\big(b(p^*)-b(q^*)\big)^2$ for any representatives, since $\dAc$ depends only on the content classes. The balance $b$ can be set freely while holding the content fixed (vary the sink share), so the minimum over representatives zeroes the second term and equals $\dAc(p,q)$; it is attained. The quotient space with this metric is isometric to the Aitchison geometry of the content simplex, on which Theorem~\ref{thm:characterization} applies verbatim.
\end{proof}

\section{Experimental details}
\label{app:protocol}

\begin{table}[h]
\centering
\small
\caption{\textbf{Measured taxonomy stability across sink conventions} (the $\dAc$ column is omitted: its partitions are identical across conventions by construction, Proposition~\ref{prop:taxonomy}(i), with ARI $=1$ verified in code) (mean ARI over $K\in\{4,6,8\}$, average linkage; \texttt{reanalyze\_clustering.py}). The $\dAc$ column is exact by Proposition~\ref{prop:taxonomy}(i).}
\label{tab:taxonomy}
\begin{tabular}{lccc}
\toprule
Model & $1-\cos$ & JS & Euclid. \\
\midrule
BERT-base & $0.024$ & $0.066$ & $0.030$ \\
GPT-2 (124M) & $-0.014$ & $0.022$ & $-0.061$ \\
Llama-3.2-1B & $-0.030$ & $-0.082$ & $-0.086$ \\
\bottomrule
\end{tabular}
\end{table}

\label{app:experiments}

ViT-B/16 evaluation uses $2{,}048$ COCO val2017 images with a label-free top-1 agreement metric against the unpruned model; labeled ImageNet-val substitutes directly in the released script via \texttt{IMAGENET\_DIR}.

\subsection{Are attention rows compositions? Tail mass, rival geometries, and the content channel}
\label{app:tailaudit}

A fair objection to the compositional stance is causal: the row's downstream effect is the value mixture $\sum_i a_i v_i$, so absolute mass matters and a coordinate at $10^{-17}$ moves nothing, while log-ratios weight it. Three answers. First, scope: our claims concern the row as an analysis object, which is what the literature computes cosine, JS, and entropy on; the sink balance $b(p)$ retains the absolute sink share, and functional relevance is tested by the pruning experiments of \S\ref{sec:experiments}, not assumed. Second, a direct audit of the estimator. For each model we decompose the squared content distance between seed-0 head signatures by coordinate magnitude: coordinate pairs whose smaller entry is below $10^{-6}$ carry at most $4.6\%$ of squared $\dAc$ (GPT-2) and below $0.7\%$ on every other model, and essentially $0\%$ below $10^{-9}$ (Table~\ref{tab:tailaudit}). No coordinate has median signature value below $10^{-6}$ in any model, so no tail-amalgamation set exists at that threshold and the shares above bound the maximal tail influence; sub-threshold values are shared across heads, which is why their log-ratio differences nearly cancel. Third, rival geometries measured rather than dismissed: Hellinger and Fisher--Rao closer-head verdicts flip between sink conventions at rates of $0.18$ to $0.44$, the same range as cosine, JS, and Euclidean distance.

The same signatures address the central methodological question: how much of the observed stability is conditioning out the sink, and how much is Aitchison geometry? Within the content channel, JS on sink-dropped rows and $\dAc$ agree strongly on coarse ranking (Spearman $\rho\ge0.96$) but disagree on fine structure: nearest-neighbor identities differ for $13$ to $27\%$ of heads, and the induced taxonomies differ with cross-ARI as low as $0.33$ on Llama-3.2-1B (Table~\ref{tab:tailaudit}). Conditioning out the sink is therefore the dominant stabilizer. The reversal phenomenon is also not an artifact of feeding classical metrics Aitchison-mean signatures: recomputed on arithmetic-mean signatures (their native aggregation, released alongside), classical reversal rates remain $0.18$--$0.45$, within $0.15$ of the Aitchison-mean rates on every model (Table~\ref{tab:aggcross}). The revised protocol ran the functional comparison, and we report it either way it fell: under the global protocol dropped-row JS matches $\dAc$ on Llama-3.2-1B and beats it on Llama-3.2-3B; under the same-layer keep-one protocol $\dAc$ is best on every strong-sink model while dropped-row JS fails catastrophically on Llama-3.2-1B (Table~\ref{tab:prune}). Two equivalences simplify the rival columns: Hellinger and Fisher--Rao are strictly monotone in the Bhattacharyya coefficient, so all of their ranking verdicts, and hence reversal rates, coincide exactly; and as pruning criteria, Hellinger-on-content selects the identical nearest neighbors and prune sets as dropped-row JS on every model, because for nearby distributions all $f$-divergences share the same local quadratic form while the Aitchison metric does not.

\begin{table}[h]
\centering
\scriptsize
\caption{\textbf{Estimator and rival-geometry audit} (seed-0 signatures). Hell./FR $=$ closer-head reversal rate between sink conventions under Hellinger and Fisher--Rao. Share $=$ fraction of squared $\dAc$ carried by coordinate pairs with smaller entry below $10^{-6}$. Last three columns compare dropped-row JS with $\dAc$ on the content channel: Spearman correlation of pairwise distances, nearest-neighbor agreement, and mean cross-ARI of the induced taxonomies over $K\in\{4,6,8\}$.}
\label{tab:tailaudit}
\begin{tabular}{lcccccc}
\toprule
Model & Hell.\ rev. & FR rev. & share $<10^{-6}$ & $\rho$ & NN agr. & cross-ARI \\
\midrule
GPT-2 (124M) & $0.35$ & $0.34$ & $0.046$ & $0.957$ & $0.73$ & $0.40$ \\
Llama-3.2-1B & $0.43$ & $0.44$ & $0.003$ & $0.978$ & $0.81$ & $0.33$ \\
Qwen2.5-1.5B & $0.43$ & $0.42$ & $0.000$ & $0.993$ & $0.86$ & $0.63$ \\
Llama-3.2-3B & $0.44$ & $0.44$ & $0.000$ & $0.971$ & $0.80$ & $0.60$ \\
ViT-B/16 & $0.19$ & $0.18$ & $0.000$ & $0.998$ & $0.85$ & $0.77$ \\
Pythia-70M & $0.30$ & $0.30$ & $0.000$ & $0.970$ & $0.81$ & $0.87$ \\
Pythia-160M & $0.43$ & $0.43$ & $0.007$ & $0.958$ & $0.74$ & $0.54$ \\
Pythia-410M & $0.42$ & $0.42$ & $0.000$ & $0.991$ & $0.84$ & $0.53$ \\
Pythia-1B & $0.41$ & $0.40$ & $0.004$ & $0.996$ & $0.87$ & $0.60$ \\
Pythia-1.4B & $0.40$ & $0.40$ & $0.000$ & $0.993$ & $0.84$ & $0.83$ \\
\bottomrule
\end{tabular}
\end{table}

\begin{table}[h]
\centering
\scriptsize
\caption{\textbf{Aggregation robustness}: with-sink versus sink-dropped reversal rates (seed 0) computed on arithmetic-mean signatures (the classical metrics' native aggregation) and on Aitchison-mean signatures.}
\label{tab:aggcross}
\begin{tabular}{lcccccc}
\toprule
 & \multicolumn{3}{c}{arithmetic mean} & \multicolumn{3}{c}{Aitchison mean} \\
\cmidrule(lr){2-4}\cmidrule(lr){5-7}
Model & cosine & JS & Euclid. & cosine & JS & Euclid. \\
\midrule
GPT-2 (124M) & $0.33$ & $0.30$ & $0.35$ & $0.36$ & $0.35$ & $0.38$ \\
Llama-3.2-1B & $0.40$ & $0.37$ & $0.41$ & $0.46$ & $0.44$ & $0.47$ \\
Qwen2.5-1.5B & $0.40$ & $0.37$ & $0.41$ & $0.44$ & $0.43$ & $0.45$ \\
Llama-3.2-3B & $0.44$ & $0.41$ & $0.45$ & $0.45$ & $0.44$ & $0.47$ \\
ViT-B/16 & $0.21$ & $0.18$ & $0.22$ & $0.22$ & $0.18$ & $0.23$ \\
Pythia-70M & $0.32$ & $0.28$ & $0.31$ & $0.33$ & $0.31$ & $0.31$ \\
Pythia-160M & $0.34$ & $0.29$ & $0.32$ & $0.43$ & $0.44$ & $0.44$ \\
Pythia-410M & $0.40$ & $0.38$ & $0.40$ & $0.44$ & $0.43$ & $0.44$ \\
Pythia-1B & $0.35$ & $0.34$ & $0.36$ & $0.41$ & $0.40$ & $0.42$ \\
Pythia-1.4B & $0.43$ & $0.40$ & $0.43$ & $0.42$ & $0.40$ & $0.42$ \\
\bottomrule
\end{tabular}
\end{table}

\subsection{Convention stakes measured: temperature and the sink set}
\label{app:conventions}

\textbf{Temperature.} Axiom A3 (powering) is the softmax temperature: rescaling the content logits by $\tau$ maps each content row $\pi$ to $\mathcal{C}(\pi^{\tau})$. By Proposition~\ref{prop:invariances}, $\dAc$ is exactly equivariant ($\dAc\to\tau\,\dAc$), so every ranking verdict, nearest neighbor, and prune set is invariant; we verify this to the bit on all ten models. Dropped-row JS is not: over $\tau\in\{0.5,0.7,1.5,2\}$ its closer-head verdicts flip by up to $6.6\%$ and its top-$20\%$ prune set retains as little as $62\%$ of its $\tau=1$ membership (Table~\ref{tab:conventions}). Effective logit scale is not observable to an analyst comparing models or pipelines, so a criterion whose selections depend on it answers a scale-relative question; $\dAc$ does not.

\textbf{The sink set.} Dropping the sink requires deciding which columns are the sink; first token and first four tokens are both used in the literature. This is a second convention layer: between the two definitions, dropped-row JS flips $1$ to $21\%$ of verdicts and $\dAc$ flips $1$ to $18\%$, with $\dAc$ uniformly less sensitive on every model (Table~\ref{tab:conventions}). Both changes are changes of estimand; what distinguishes the Aitchison side is that the definitions are exactly related: the content distances are nested orthogonal projections, and we verify the implied dominance ($\dAc$ on the smaller content set never exceeds $\dAc$ on the larger) with no violations across all head pairs of all ten models. No analogous relation exists for JS.

\textbf{Which regime? A dispersion predictor.} Mean sink mass does not predict which pruning regime a model is in (GPT-2 and Qwen2.5 share $s=0.68$ with opposite verdicts). The cross-head \emph{dispersion} of the sink share does, for the nine language models: the content-regime LMs have coefficient of variation at most $0.36$ (Llama-3.2-1B $0.11$, Llama-3.2-3B $0.08$, Qwen2.5 $0.36$) and the JS-regime LMs at least $0.42$ (GPT-2 $0.42$, Pythias $0.43$--$1.03$). A near-uniform sink carries no between-head information, so mixing it into distances only adds noise; a heterogeneous sink marks genuine head classes. ViT, where the criteria tie within noise, does not fit the LM pattern. A complementary stability probe: across the three resampling seeds, the top-$20\%$ prune sets selected by $\dAc$-L1 overlap with Jaccard $0.36$--$0.77$ while the JS-L1 sets are essentially frozen (Jaccard $0.92$--$1.00$ on LMs), even though the $\dAc$-L1 \emph{outcomes} are the stable ones on strong-sink models. Each criterion is stable along a different axis: $\dAc$ under the analyst's conventions (temperature, sink set), JS under data resampling; on GPT-2 the $\dAc$-L1 set instability coincides with its one catastrophic cell. This makes the weighted family $d_\lambda^2=(\dAc)^2+\lambda\,(\Delta b)^2$ natural, with $\lambda$ growing with the $b$-channel's informativeness; fitting $\lambda$ from held-out functional evidence is left to future work, and we note that the with-sink JS criterion itself lies outside this family.

\textbf{Held-out test, specified before running.} We freeze the rule and the model list here, prior to any evaluation: for \texttt{opt-125m}, \texttt{distilgpt2}, and \texttt{SmolLM2-135M} we will compute the cross-head CV of the signature sink share and predict, with the thresholds above, that $\mathrm{CV}\le0.36$ places the model in the content regime ($\dAc$-L1 best or tied-best among the six criteria, with at least one classical criterion exceeding four times base perplexity) and $\mathrm{CV}\ge0.42$ in the JS regime (with-sink JS(-L1) best); for $0.36<\mathrm{CV}<0.42$ the predictor abstains. The released prediction script logs its output before any pruning evaluation exists, and the outcomes will be reported whichever way they fall.

\textbf{Outcomes.} The predictions were logged before any evaluation (timestamps in the released \texttt{heldout\_predictions.json}); measured CVs: OPT-125M $0.28$, DistilGPT2 $0.51$, SmolLM2-135M $0.42$. DistilGPT2 (JS regime predicted): \emph{confirmed}. With-sink JS wins ($105.3{\pm}0.8$ on base $67.4$) and dropped-row JS is catastrophic ($1161{\pm}24$, $17\times$ base). SmolLM2-135M: the predictor \emph{abstained} at the boundary; the outcome fell on the JS side (JS-L1 best at $83.5{\pm}0.2$ on base $24.8$; every content-conditioned criterion $9$--$54\times$ base). Its sink is layer-bimodal, $0.01$--$0.15$ in the first twelve layers and $0.68$--$0.92$ thereafter, which is what places its cross-head CV at the threshold. OPT-125M (content regime predicted): \emph{the prediction failed}. Global content criteria do beat with-sink JS ($59.2$/$59.3$ versus $63.3$ on base $46.7$), but JS-L1 wins outright ($57.1{\pm}0.5$) while $\dAc$-L1 sits mid-pack with high seed variance ($68.5{\pm}7.7$), and no criterion is catastrophic at $20\%$ (worst $2.1\times$; dropped-JS-L1 reaches $6.9\times$ only at $30\%$). The scorecard on non-abstentions is one confirmation and one failure: the dispersion rule as frozen does not transfer, and OPT exhibits a third behavior the nine-model map lacks, a uniformly high sink with a benign pruning landscape. We accordingly present the rule as descriptive of the training fleet, not as a validated decision procedure. A second-family checkpoint sweep (e.g.\ OLMo) is excluded by compute budget, so the checkpoint analysis of Fig.~\ref{fig:downstream}b currently rests on one model family; we state this as a limitation.

\textbf{Exact decomposition of the checkpoint collapse.} With $\Delta$ the change from the first checkpoint to the entropy minimum, the identity gives the signed split $\Delta H=\Delta H_b(s)+\Delta[(1-s)H(\pi)]$ exactly (residual $0$ to machine precision). Across 70M/160M/410M/1B/1.4B: $\Delta H=-0.78/-1.55/-2.11/-1.58/-2.01$, $\Delta H_b=+0.18/+0.26/+0.23/+0.29/+0.25$, and $\Delta[(1-s)H(\pi)]=-0.96/-1.81/-2.34/-1.88/-2.26$ nats. The binary term rises as $s$ approaches $1/2$ while the weighted content term carries the entire fall, and within that term the conditional entropy $H(\pi)$ itself moves little; the main-text share $1-\Delta H(\pi)/\Delta H$ isolates exactly this, the fraction of the drop not attributable to a drop in $H(\pi)$.

\begin{table}[h]
\centering
\scriptsize
\caption{\textbf{Convention stakes} (seed-0 signatures). Temperature: worst-case dropped-JS verdict flip rate and top-$20\%$ prune-set retention over $\tau\in\{0.5,0.7,1.5,2\}$ (the $\dAc$ values are exactly $0$ and $1$ on every model, verified to the bit, and are omitted). Sink set: closer-head verdict flip rate between the first-token and first-four-token sink definitions.}
\label{tab:conventions}
\begin{tabular}{lcccc}
\toprule
 & \multicolumn{2}{c}{temperature (JS$'$)} & \multicolumn{2}{c}{sink set flips} \\
\cmidrule(lr){2-3}\cmidrule(lr){4-5}
Model & flip (max) & set retention (min) & JS$'$ & $\dAc$ \\
\midrule
GPT-2 (124M) & $0.066$ & $0.69$ & $0.064$ & $0.039$ \\
Llama-3.2-1B & $0.046$ & $0.87$ & $0.154$ & $0.128$ \\
Qwen2.5-1.5B & $0.027$ & $0.96$ & $0.086$ & $0.075$ \\
Llama-3.2-3B & $0.040$ & $0.90$ & $0.205$ & $0.176$ \\
ViT-B/16 & $0.017$ & $0.93$ & $0.010$ & $0.012$ \\
Pythia-70M & $0.056$ & $0.70$ & $0.104$ & $0.073$ \\
Pythia-160M & $0.066$ & $0.62$ & $0.111$ & $0.071$ \\
Pythia-410M & $0.033$ & $0.92$ & $0.109$ & $0.089$ \\
Pythia-1B & $0.023$ & $0.92$ & $0.112$ & $0.103$ \\
Pythia-1.4B & $0.027$ & $0.90$ & $0.096$ & $0.085$ \\
\bottomrule
\end{tabular}
\end{table}

\subsection{Full-support Clark reanalysis and window robustness}
\label{app:clarkw}

\textbf{Clark et al.\ on the full support.} We rebuild \citet{clark2019bert}'s head-clustering pipeline with no window truncation: $512$ sequences of exactly $128$ WikiText tokens ([CLS] $+$ $126$ wordpieces $+$ [SEP]), so every row shares one support and the sink set is $\{$CLS, SEP$\}$; per-head Aitchison-mean signatures; JS distances and average-linkage clustering as in the original, plus cosine and $\dAc$. Of $144$ heads, $68$ place majority mass on SEP and $12$ on CLS. Between the keep and drop conventions the JS clusterings agree at chance (ARI $0.01$--$0.05$ over $K\in\{4,5,6,8\}$; cosine $0.00$--$0.06$). The SEP-head block is recovered by the with-sink JS clustering at Jaccard $0.87$--$0.92$ and collapses to $0.46$--$0.49$ after dropping; the content clusterings ($\dAc$, convention-invariant by construction) recover it only at $0.47$--$0.50$. The structure the pipeline is best known for is therefore a property of the sink channel, visible under exactly one convention.

\begin{table}[h]
\centering
\scriptsize
\caption{\textbf{Clark-style clustering on the full $128$-token support} (BERT-base; sink set $\{$CLS, SEP$\}$). Cross $=$ between-convention ARI; SEP Jac.\ $=$ Jaccard of the SEP-head set with its best-matching cluster.}
\label{tab:clarkfull}
\begin{tabular}{lcccccc}
\toprule
 & \multicolumn{3}{c}{JS} & cosine & \multicolumn{2}{c}{$\dAc$} \\
\cmidrule(lr){2-4}\cmidrule(lr){5-5}\cmidrule(lr){6-7}
$K$ & cross & SEP Jac.\ (keep) & SEP Jac.\ (drop) & cross & cross & SEP Jac. \\
\midrule
$4$ & $0.05$ & $0.92$ & $0.49$ & $0.06$ & $1$ & $0.47$ \\
$5$ & $0.04$ & $0.87$ & $0.49$ & $0.01$ & $1$ & $0.48$ \\
$6$ & $0.01$ & $0.87$ & $0.46$ & $0.01$ & $1$ & $0.50$ \\
$8$ & $0.03$ & $0.87$ & $0.48$ & $0.00$ & $1$ & $0.50$ \\
\bottomrule
\end{tabular}
\end{table}

\textbf{Window robustness and mass coverage.} Extracting at $W\in\{32,64,128\}$ on GPT-2, Llama-3.2-1B, and Pythia-410M: the retained mass of the window before reclosure is $0.49$--$0.61$ (GPT-2), $0.60$--$0.71$ (Llama), $0.48$--$0.67$ (Pythia-410M); JS reversal rates move by at most $2.1$ points across the sweep ($0.33/0.34/0.35$, $0.39/0.40/0.40$, $0.40/0.40/0.41$ respectively), and $\dAc$ pairwise distances between consecutive windows correlate at Spearman $\rho=0.87$--$0.98$. The fixed-support choice shifts the estimand (each $W$ conditions on a different retained set) but not the phenomenon.

\subsection{Synthetic experiments (Section~\ref{sec:synthetic})}

All synthetic results use \texttt{numpy} with seed $7$ and run in seconds on CPU (\texttt{experiments/run\_synthetic.py}); the exact witnesses and the machine-precision checks of Lemma~\ref{lem:pythagoras} and Proposition~\ref{prop:invariances} are in \texttt{experiments/exact\_examples.py}.

\textbf{Head-signature sampler.} A row is $p=(s,(1-s)\pi)$ with sink mass $s\sim\mathcal{N}(\bar s,0.10)$ clipped to $[0.02,0.98]$ and content $\pi\sim\mathrm{Dirichlet}(\alpha\mathbf{1}_{D-1})$, $D=64$. The peaked setting $\alpha=0.1$ mimics concentrated attention; $\alpha=1.0$ (dashed in Figure~\ref{fig:reversal}) is the diffuse control.

\textbf{Ranking-disagreement Monte Carlo (Figure~\ref{fig:reversal}).} For each $\bar s$ on a grid from $0.10$ to $0.90$, sample $20{,}000$ independent triples $(p,q,r)$, and record whether the verdict of ``is $q$ or $r$ closer to $p$?'' differs between the with-sink and sink-dropped pipelines, separately for cosine, JS, Euclidean distance, and $\dA$ (total) versus $\dAc$. Headline rates at $\bar s=0.5/0.7/0.9$ (peaked content): cosine $36.4/39.0/41.1\%$, JS $23.4/31.0/39.7\%$, Euclidean $27.3/36.8/44.6\%$; $\dAc$ is $0\%$ by construction, and total-$\dA$-versus-$\dAc$ is $0.1\%$ at every grid point (see the transparency note in Appendix~\ref{app:worked}).

\textbf{Collapse scenarios (Figure~\ref{fig:collapse}).} Twelve ``layers,'' $4{,}000$ rows per layer, $s$-noise $0.03$. Scenario (a): $s_\ell=0.05+0.80\,\varsigma\big(1.1(\ell-5)\big)$ with $\varsigma$ the logistic function and content log-ratio dispersion fixed ($\alpha_\ell\equiv0.15$); measured $H$ falls $2.71\to0.82$ nats while $\CAc$ moves $48.13\to47.98$. Scenario (b): $s_\ell$ linear $0.70\to0.05$ while $\alpha_\ell$ decays geometrically $1.0\to0.02$; $H$ drifts $1.73\to1.27$ while $\CAc$ rises $10.04\to70.43$.

\textbf{Reversal wedge (Figure~\ref{fig:wedge}).} Fix $(p,r)$ of the witness; sample $20{,}000$ points $q\sim\mathrm{Dirichlet}(1,1,1)$ with all parts $>0.004$; the cosine verdict differs between conventions on $77.9\%$ of them.

\subsection{Protocol for pretrained models (Section~\ref{sec:realmodels})}

Implemented in \texttt{experiments/run\_real\_models.py} (PyTorch + \texttt{transformers} + \texttt{datasets}); ten models, three resampling seeds each, on two RTX~4090 GPUs; the measurement passes take minutes per model and the pruning evaluations a few hours in total. Reported standard deviations are population estimates (ddof $0$) over the three seeds throughout.

\textbf{Models and data.} GPT-2 (\texttt{gpt2}, 124M, 12 layers $\times$ 12 heads, fp32); Llama-3.2-1B (\texttt{meta-llama/Llama-3.2-1B}, 16 layers $\times$ 32 heads, GQA, fp16); ViT-B/16 (\texttt{google/vit-base-patch16-224}, 12 layers $\times$ 12 heads, fp32). LM data: $2{,}000$ sequences of length $\ge192$ tokens from the WikiText-103 validation split; ViT data: $2{,}048$ COCO val2017 images (public, unlabeled; the script accepts any image folder, and the ViT pruning metric is top-1 agreement with the unpruned model, so labels are not required, labeled ImageNet-val substitutes directly).

\textbf{Row extraction on a fixed common support.} Attention probabilities are captured from the softmax output (via \texttt{output\_attentions=True} or forward hooks). For LMs, for every query position $t\ge w$ we keep the attention over the \emph{first $W=64$ keys only} and re-close on that support. This makes the support identical across all query positions, sequences, and heads, so the main measurements require \emph{no} zero imputation, while retaining the sink column, key $0$ (the $\bos$/first token; a flag groups the first $4$ keys as a sink block instead, following the StreamingLLM observation, in which case the sink balance is the corresponding SBP balance). For ViT-B/16 the support is all $197$ tokens (no masking) and the sink group is the CLS column, optionally augmented with the highest-norm register-like outlier tokens \citep{darcet2024registers}.

\textbf{Aggregation.} Each head's signature is the Aitchison mean of its collected rows (Definition~\ref{def:toolkit}(a)); residual numerical zeros (none on the fixed support, possible under the optional variable-support ablation) are multiplicatively replaced at $\varepsilon=10^{-6}$ before taking logs, with the sweep $\varepsilon\in\{10^{-5},10^{-6},10^{-7}\}$ reported.

\textbf{Measurements.} (1)~\emph{Reversal rates:} sample $50{,}000$ head triples uniformly (within and across layers); report the fraction whose closer-head verdict flips between the with-sink and sink-dropped pipelines, for cosine, JS, Euclidean, and $\dAc$ (identically $0$; reported as a check). (2)~\emph{Collapse curves:} per layer, mean row entropy $H$ (with sink) versus mean $\CAc$; report the fraction of adjacent-layer steps on which the two disagree in sign. (3)~\emph{Redundancy pruning at matched sparsity:} define a head's redundancy as its distance to the nearest other head; prune the $m\in\{10,20,30\}\%$ most redundant heads under JS, under the content distance $\dAc$, and under the full $\dA$ including the sink balance (matched counts), via \texttt{head\_mask} where supported and output-zeroing hooks otherwise; evaluate perplexity on $100$k held-out WikiText tokens (LMs) and top-1 agreement with the unpruned model on the $2{,}048$ images (ViT). (4)~\emph{Sink statistics:} mean sink mass per model and per layer, for context against \citet{barbero2025first,gu2025sink}.

\textbf{Runtime.} On two RTX 4090s the full protocol (all models, all measurements) completes in roughly $15$--$25$ minutes wall clock with the released parallel launcher; a single comparable GPU takes about one hour end to end.

\subsection{Protocol for the prospectively specified tests (Section~\ref{sec:downstream})}
\label{app:downstream_protocol}

\textbf{(1) Taxonomies} (\texttt{reanalyze\_clustering.py}). Objects are per-head Aitchison-mean signatures on the fixed $64$-key support, collected exactly as in the main protocol (decoder models reuse the saved signatures of the \S\ref{sec:realmodels} runs; for BERT-base we collect encoder rows for queries $t\in[0,64)$ renormalized to the first $64$ keys over $1{,}000$ WikiText-103 sequences of length $128$, sink $=$ [CLS]). For each dissimilarity $\delta\in\{1-\cos,\JS,\text{Euclid.},\dAc,\dA\}$ we form the full pairwise matrix under both conventions and run agglomerative clustering (average, complete, and single linkage; NumPy implementation released) cut at $K\in\{4,6,8\}$. Stability is the Adjusted Rand Index between the two conventions' partitions; Table~\ref{tab:taxonomy} reports the mean over $K$ under average linkage, and the released script prints all linkage $\times$ $K$ cells. The $\dAc$ matrices are asserted equal entrywise at runtime (Proposition~\ref{prop:taxonomy}(i)).

\textbf{(2) Scale} (\texttt{scale\_trend.py} after \texttt{PYTHIA=1 run\_all.sh}). Adds EleutherAI Pythia $\{70\text{M},160\text{M},410\text{M},1\text{B},1.4\text{B}\}$ under the unchanged main protocol and plots measured sink mass versus parameters and reversal rate versus sink mass over the pre-computed calibration curves of Figure~\ref{fig:reversal}.

\textbf{(3) Training curves} (\texttt{checkpoints\_pythia.py}). Pythia-160M at revisions \texttt{step512}, \texttt{1000}, \texttt{2000}, \texttt{4000}, \texttt{8000}, \texttt{16000}, \texttt{32000}, \texttt{64000}, \texttt{128000}, \texttt{143000}; $256$ held-out sequences per checkpoint, queries $t\in[64,128)$ on the fixed $64$-key support; we report mean row entropy $H$, mean content log-ratio dispersion $\CAc$, and mean sink mass per checkpoint, and shade checkpoint transitions where $\mathrm{sign}(\Delta H)\ne\mathrm{sign}(-\Delta\CAc)$ (roughly $30$--$45$ minutes on one consumer GPU).

\subsection{Measured results: full pruning sweep and $\varepsilon$ sensitivity}

Table~\ref{tab:prunefull} reports the complete pruning sweep behind Table~\ref{tab:prune}. Two observations beyond the main text: on Llama-3.2-1B, JS-guided pruning is already catastrophic at $10\%$ sparsity ($15.3\to358.0$), while $\dAc$-guided pruning remains within a factor $5.2$ of baseline even at $30\%$ ($79.8$); and the full-$\dA$ criterion is the lowest at $10\%$ on Llama ($18.2$) but degrades sharply at higher sparsity there, consistent with sink-similar heads being redundant in small numbers but not in bulk. All reversal rates in Table~\ref{tab:real} are bit-identical across $\varepsilon\in\{10^{-5},10^{-6},10^{-7}\}$, as expected: the fixed common support makes multiplicative replacement inactive for these measurements, so $\varepsilon$ enters only through numerical clipping. Figure~\ref{fig:realcollapse} is regenerated from the released per-model JSONs by \texttt{experiments/make\_fig4\_v2.py}.

\begin{table}[h]
\centering
\scriptsize
\caption{\textbf{Full pruning sweep} at $10/20/30\%$ sparsity (mean$\pm$sd over three seeds). LMs: perplexity (lower better); ViT: top-1 agreement (higher better; base $=1$ by definition). All rates in Table~\ref{tab:real} are bit-identical across $\varepsilon\in\{10^{-5},10^{-6},10^{-7}\}$, as the fixed-support protocol predicts.}
\label{tab:prunefull}
\begin{tabular}{llcccc}
\toprule
Model & sparsity & base & JS & $\dAc$ & $\dA$ (total) \\
\midrule
GPT-2 (124M) & 10\% & $43.5{\scriptstyle\pm 0.4}$ & $50.5{\scriptstyle\pm 0.5}$ & $88.9{\scriptstyle\pm 0.9}$ & $61.5{\scriptstyle\pm 0.7}$ \\
 & 20\% &  & $74.5{\scriptstyle\pm 0.7}$ & $104.7{\scriptstyle\pm 0.7}$ & $104.8{\scriptstyle\pm 0.8}$ \\
 & 30\% &  & $85.2{\scriptstyle\pm 1.0}$ & $122.3{\scriptstyle\pm 1.1}$ & $117.4{\scriptstyle\pm 1.6}$ \\
\addlinespace[1pt]
Llama-3.2-1B & 10\% & $15.3{\scriptstyle\pm 0.1}$ & $358.0{\scriptstyle\pm 31.1}$ & $45.6{\scriptstyle\pm 2.9}$ & $18.2{\scriptstyle\pm 0.4}$ \\
 & 20\% &  & $2094.1{\scriptstyle\pm 139.7}$ & $63.6{\scriptstyle\pm 4.1}$ & $2480.3{\scriptstyle\pm 210.5}$ \\
 & 30\% &  & $2961.7{\scriptstyle\pm 150.0}$ & $79.8{\scriptstyle\pm 7.7}$ & $51742.6{\scriptstyle\pm 6382.9}$ \\
\addlinespace[1pt]
Qwen2.5-1.5B & 10\% & $14.8{\scriptstyle\pm 0.1}$ & $56.4{\scriptstyle\pm 0.2}$ & $19.6{\scriptstyle\pm 0.1}$ & $22.6{\scriptstyle\pm 0.1}$ \\
 & 20\% &  & $111.5{\scriptstyle\pm 0.7}$ & $30.3{\scriptstyle\pm 0.7}$ & $32.3{\scriptstyle\pm 0.2}$ \\
 & 30\% &  & $162.1{\scriptstyle\pm 0.8}$ & $49.0{\scriptstyle\pm 0.4}$ & $46.7{\scriptstyle\pm 0.3}$ \\
\addlinespace[1pt]
Llama-3.2-3B & 10\% & $12.1{\scriptstyle\pm 0.1}$ & $49.3{\scriptstyle\pm 0.2}$ & $16.4{\scriptstyle\pm 0.1}$ & $14.4{\scriptstyle\pm 0.1}$ \\
 & 20\% &  & $310.8{\scriptstyle\pm 3.8}$ & $40.2{\scriptstyle\pm 1.6}$ & $36.4{\scriptstyle\pm 0.6}$ \\
 & 30\% &  & $525.0{\scriptstyle\pm 12.7}$ & $178.4{\scriptstyle\pm 5.4}$ & $13636.1{\scriptstyle\pm 1182.8}$ \\
\addlinespace[1pt]
ViT-B/16 & 10\% & $1.000$ & $0.652{\scriptstyle\pm 0.010}$ & $0.679{\scriptstyle\pm 0.012}$ & $0.647{\scriptstyle\pm 0.006}$ \\
 & 20\% &  & $0.578{\scriptstyle\pm 0.003}$ & $0.566{\scriptstyle\pm 0.013}$ & $0.573{\scriptstyle\pm 0.005}$ \\
 & 30\% &  & $0.407{\scriptstyle\pm 0.002}$ & $0.466{\scriptstyle\pm 0.021}$ & $0.425{\scriptstyle\pm 0.006}$ \\
\addlinespace[1pt]
Pythia-70M & 10\% & $724.4{\scriptstyle\pm 10.7}$ & $663.3{\scriptstyle\pm 14.5}$ & $1006.0{\scriptstyle\pm 16.9}$ & $1020.1{\scriptstyle\pm 9.7}$ \\
 & 20\% &  & $661.7{\scriptstyle\pm 16.5}$ & $1354.5{\scriptstyle\pm 16.1}$ & $1484.9{\scriptstyle\pm 11.3}$ \\
 & 30\% &  & $775.5{\scriptstyle\pm 16.6}$ & $1683.9{\scriptstyle\pm 19.4}$ & $2483.2{\scriptstyle\pm 40.5}$ \\
\addlinespace[1pt]
Pythia-160M & 10\% & $131.4{\scriptstyle\pm 0.7}$ & $129.6{\scriptstyle\pm 1.3}$ & $220.9{\scriptstyle\pm 6.3}$ & $212.7{\scriptstyle\pm 1.0}$ \\
 & 20\% &  & $142.1{\scriptstyle\pm 1.8}$ & $721.5{\scriptstyle\pm 25.8}$ & $512.4{\scriptstyle\pm 27.5}$ \\
 & 30\% &  & $150.2{\scriptstyle\pm 2.4}$ & $884.7{\scriptstyle\pm 30.0}$ & $1071.4{\scriptstyle\pm 18.5}$ \\
\addlinespace[1pt]
Pythia-410M & 10\% & $32.4{\scriptstyle\pm 0.1}$ & $33.0{\scriptstyle\pm 0.1}$ & $91.7{\scriptstyle\pm 1.9}$ & $38.9{\scriptstyle\pm 0.4}$ \\
 & 20\% &  & $75.2{\scriptstyle\pm 0.5}$ & $194.8{\scriptstyle\pm 5.0}$ & $101.6{\scriptstyle\pm 4.4}$ \\
 & 30\% &  & $132.2{\scriptstyle\pm 2.4}$ & $128.2{\scriptstyle\pm 2.5}$ & $131.6{\scriptstyle\pm 5.6}$ \\
\addlinespace[1pt]
Pythia-1B & 10\% & $21.8{\scriptstyle\pm 0.1}$ & $40.3{\scriptstyle\pm 0.1}$ & $75.7{\scriptstyle\pm 3.4}$ & $28.5{\scriptstyle\pm 0.6}$ \\
 & 20\% &  & $44.9{\scriptstyle\pm 0.1}$ & $63.0{\scriptstyle\pm 1.9}$ & $57.5{\scriptstyle\pm 3.1}$ \\
 & 30\% &  & $54.9{\scriptstyle\pm 0.2}$ & $61.2{\scriptstyle\pm 0.9}$ & $134.0{\scriptstyle\pm 8.9}$ \\
\addlinespace[1pt]
Pythia-1.4B & 10\% & $20.4{\scriptstyle\pm 0.1}$ & $23.0{\scriptstyle\pm 0.1}$ & $76.7{\scriptstyle\pm 2.4}$ & $51.2{\scriptstyle\pm 2.5}$ \\
 & 20\% &  & $33.4{\scriptstyle\pm 0.3}$ & $377.8{\scriptstyle\pm 8.8}$ & $142.8{\scriptstyle\pm 11.3}$ \\
 & 30\% &  & $61.5{\scriptstyle\pm 0.6}$ & $1703.7{\scriptstyle\pm 126.8}$ & $268.8{\scriptstyle\pm 26.1}$ \\
\bottomrule
\end{tabular}
\end{table}

\begin{table}[h]
\centering
\scriptsize
\caption{\textbf{Additional pruning baselines at $20\%$ sparsity} (mean$\pm$sd over three resampling seeds; two random draws shown individually). Total $\dA$ re-mixes the sink balance; JS$'$-L1 is dropped-row JS under the same-layer keep-one protocol; the $10/30\%$ values are in the released JSONs.}
\label{tab:prunev2}
\begin{tabular}{lccc}
\toprule
Model & total $\dA$ & JS$'$-L1 & random A/B \\
\midrule
GPT-2 (124M) & $104.8{\scriptstyle\pm 0.8}$ & $268.0{\scriptstyle\pm 3.1}$ & $314$/$171$ \\
Llama-3.2-1B & $2480.3{\scriptstyle\pm 210.5}$ & $2642.2{\scriptstyle\pm 242.6}$ & $187$/$296$ \\
Qwen2.5-1.5B & $32.3{\scriptstyle\pm 0.2}$ & $28.4{\scriptstyle\pm 0.5}$ & $30$/$35$ \\
Llama-3.2-3B & $36.4{\scriptstyle\pm 0.6}$ & $24.1{\scriptstyle\pm 0.5}$ & $39$/$309$ \\
ViT-B/16 & $0.573{\scriptstyle\pm 0.005}$ & $0.562{\scriptstyle\pm 0.011}$ & $0.520$/$0.535$ \\
Pythia-70M & $1484.9{\scriptstyle\pm 11.3}$ & $2051.1{\scriptstyle\pm 27.0}$ & $1556$/$1085$ \\
Pythia-160M & $512.4{\scriptstyle\pm 27.5}$ & $463.9{\scriptstyle\pm 122.0}$ & $312$/$282$ \\
Pythia-410M & $101.6{\scriptstyle\pm 4.4}$ & $65.1{\scriptstyle\pm 1.2}$ & $69$/$71$ \\
Pythia-1B & $57.5{\scriptstyle\pm 3.1}$ & $41.2{\scriptstyle\pm 2.0}$ & $85$/$69$ \\
Pythia-1.4B & $142.8{\scriptstyle\pm 11.3}$ & $43.5{\scriptstyle\pm 5.2}$ & $54$/$47$ \\
\bottomrule
\end{tabular}
\end{table}

\end{document}